%% file: info-relax-sampling-arxiv.tex
\documentclass[11pt]{article}

\usepackage{epsfig, psfrag}
\usepackage{latexsym}
\usepackage{enumerate}
\usepackage{url}
\usepackage{float}
\usepackage[hypertexnames=false,hyperfootnotes=false]{hyperref}
\usepackage{texnansi}
\usepackage{color}
\usepackage{tikz}
\usepackage[margin=10pt,font=small,labelfont=bf]{caption}
\usepackage[subrefformat=parens,labelformat=simple]{subcaption}
\usepackage{afterpage}
\usepackage{enumitem}
\usepackage[boxed]{algorithm}
\usepackage{algpseudocode}
\usepackage[normalem]{ulem}
\usepackage{lmodern}
\usepackage{booktabs}
\usepackage{sectsty}
\usepackage{ifthen}
\usepackage{amsmath}
\usepackage{amsthm}
\usepackage{amssymb}
\usepackage{amsfonts}
\usepackage{thmtools}
\usepackage{thm-restate}
\usepackage{xspace}
\usepackage{titling}
\usepackage{parskip}
\usepackage{natbib}
\usepackage{xfrac}
\usepackage{multirow}
\usepackage{bigdelim}
\usepackage{bm}
\usepackage{makecell}
\usepackage{colortbl}
\usepackage{xcolor}
\PassOptionsToPackage{algo2e,ruled}{algorithm2e}
\RequirePackage{algorithm2e}
\newcommand{\field}[1]{\ensuremath{\mathbb{#1}}}
\newcommand{\N}{\ensuremath{\field{N}}} 
\newcommand{\R}{\ensuremath{\field{R}}} 
\newcommand{\PR}{\ensuremath{\mathbb{P}}} 
\newcommand{\E}{\ensuremath{\mathbb{E}}} 
\newcommand{\defeq}{\ensuremath{\triangleq}}

\newcommand{\Ascr}{\ensuremath{\mathcal A}}

\newcommand{\Escr}{\ensuremath{\mathcal E}}
\newcommand{\Fscr}{\ensuremath{\mathcal F}}
\newcommand{\Gscr}{\ensuremath{\mathcal G}}
\newcommand{\Hscr}{\ensuremath{\mathcal H}}
\newcommand{\Iscr}{\ensuremath{\mathcal I}}

\newcommand{\Lscr}{\ensuremath{\mathcal L}}
\newcommand{\Mscr}{\ensuremath{\mathcal M}}
\newcommand{\Nscr}{\ensuremath{\mathcal N}}

\newcommand{\Pscr}{\ensuremath{\mathcal P}}

\newcommand{\Rscr}{\ensuremath{\mathcal R}}
\newcommand{\Sscr}{\ensuremath{\mathcal S}}

\newcommand{\Uscr}{\ensuremath{\mathcal U}}

\newcommand{\Yscr}{\ensuremath{\mathcal Y}}

\DeclareMathOperator*{\argmin}{\mathrm{argmin}}
\DeclareMathOperator*{\argmax}{\mathrm{argmax}}

\newcommand{\maximize}{\ensuremath{\mathop{\mathrm{maximize}}\limits}}
\declaretheoremstyle[headfont=\sffamily\bfseries,bodyfont=\itshape]{thm-sf}
\declaretheorem[style=thm-sf]{theorem}
\declaretheorem[style=thm-sf]{remark}

\declaretheorem[style=thm-sf]{definition}

\declaretheorem[style=thm-sf]{corollary}
\declaretheorem[style=thm-sf]{lemma}

\declaretheorem[style=thm-sf]{proposition}

\renewcommand{\thmcontinues}[1]{\hyperref[#1]{continued}}
\usetikzlibrary{arrows,patterns,plotmarks,pgfplots.groupplots}
\tikzstyle{every picture} += [>=stealth]
\tikzset{axis/.style={semithick, line join=miter}}
\allsectionsfont{\sffamily}
\makeatletter
\def\@seccntformat#1{\csname the#1\endcsname.\quad}
\makeatother

\floatname{algorithm}{\normalfont\sffamily\bfseries Algorithm}
\floatstyle{ruled}
\newcommand{\emailhref}[1]{\href{mailto:#1}{\tt #1}} 
\provideboolean{fastcompile}
\newcommand{\hidefastcompile}[1]{\ifthenelse{\boolean{fastcompile}}{}{#1}}
\newcommand{\todo}[1]{{\color{red} \noindent {\sffamily\bfseries TODO:} #1}}

\usepackage{pgfplots}
\usepackage{pgfplotstable}
\usetikzlibrary{calc}
\usepackage{mathtools}
\definecolor{orange}{rgb}{0.85,0.33,0.13} 
\definecolor{green}{rgb}{0.13,0.85,0.33}
\definecolor{purple}{rgb}{0.33,0.13,0.85}
\definecolor{lime}{rgb}{0.65,0.85,0.13}
\definecolor{blue}{rgb}{0.13,0.65,0.85}
\pgfplotscreateplotcyclelist{tricolor}{%
  orange,every mark/.append style={fill=orange!80!black},mark=*\\%
  green,every mark/.append style={fill=green!80!black},mark=square*\\%
  purple,every mark/.append style={fill=purple!80!black},mark=otimes*\\%
  black,mark=star\\%
  orange,every mark/.append style={fill=orange!80!black},mark=diamond*\\%
  green,densely dashed,every mark/.append style={solid,fill=green!80!black},mark=*\\%
  purple,densely dashed,every mark/.append style={solid,fill=purple!80!black},mark=square*\\%
  black,densely dashed,every mark/.append style={solid,fill=gray},mark=otimes*\\%
  orange,densely dashed,mark=star,every mark/.append style=solid\\%
  green,densely dashed,every mark/.append style={solid,fill=green!80!black},mark=diamond*\\%
}
\pgfplotsset{colormap={tricolormap}{color=(orange) color=(green) color=(purple)},
  colormap={quadcolormap}{color=(orange) color=(lime) color=(blue) color=(purple)}}
\pgfplotstableset{%
  font=\small,
  every head row/.style={before row=\toprule[1pt], after row=\midrule},
  every last row/.style={after row=\bottomrule[1pt]}}

\renewcommand*{\thead}[1]{\bfseries \makecell{#1}}

\usepackage{setspace}
\usepackage[margin=1in]{geometry}

\setcitestyle{authoryear,round,semicolon,aysep={,},yysep={,},notesep={, }}
\provideboolean{submissionversion}

\provideboolean{hideedits}
\setboolean{hideedits}{true}

\ifthenelse{\boolean{hideedits}}{
  \renewcommand{\todo}[1]{}
  \newcommand{\newedit}[1]{#1}
  \newcommand{\deledit}[1]{}
}{
  \newcommand{\newedit}[1]{{\color{green} #1}}
  \newcommand{\deledit}[1]{{\color{orange} \sout{#1}}}
  
 }

\tikzstyle{rate} += [color=orange,very thick]
\usetikzlibrary{matrix,positioning}
\usepgfplotslibrary{groupplots}
\pgfplotsset{compat=newest}

\ifthenelse{\boolean{submissionversion}}{
  \title{\textsf{\textbf{Thompson Sampling with Information Relaxation Penalties}}}
  \author{}
  \date{}
}{
  \title{\textsf{\textbf{Thompson Sampling with Information Relaxation Penalties\thanks{
          The authors wish to thank Daniel Russo, Martin Haugh, David Brown, Jim Smith, and
          anonymous reviewers for
          helpful comments.}
}}}
\author{ \\
  Seungki Min \\
  Graduate School of Business \\
  Columbia University \\
  \emailhref{smin20@gsb.columbia.edu}
  \and \\
  Costis Maglaras \\
  Graduate School of Business \\
  Columbia University \\
  \emailhref{c.maglaras@gsb.columbia.edu}  \\
  \and \\
  Ciamac C. Moallemi \\
  Graduate School of Business \\
  Columbia University \\
 \emailhref{ciamac@gsb.columbia.edu} \\
}

\date{
 Current Revision: March 2021}
}

\begin{document}

\maketitle
\singlespacing

\begin{abstract}
  We consider a finite-horizon multi-armed bandit (MAB) problem in a Bayesian setting, for which
  we propose an \textit{information relaxation sampling} framework. With this framework, we define
  an intuitive family of control policies that include Thompson sampling (\textsc{TS}) and the
  Bayesian optimal policy as endpoints.  Analogous to \textsc{TS}, which, at each decision epoch
  pulls an arm that is best with respect to the randomly sampled parameters, our algorithms sample
  entire future reward realizations and take the corresponding best action.  However, this is done
  in the presence of ``penalties'' that seek to compensate for the availability of future
  information.

  We develop several novel policies and performance bounds for MAB problems that vary in terms of
  improving performance and increasing computational complexity between the two endpoints.  Our
  policies can be viewed as natural generalizations of \textsc{TS} that simultaneously incorporate
  knowledge of the time horizon and explicitly consider the exploration-exploitation trade-off. We
  prove associated structural results on performance bounds and suboptimality gaps.  Numerical
  experiments suggest that this new class of policies perform well, in particular in settings
  where the finite time horizon introduces significant exploration-exploitation tension into the
  problem.
  Finally, inspired by the finite-horizon Gittins index, we propose an index policy that builds on our framework that particularly outperforms the state-of-the-art algorithms in our numerical experiments.
\end{abstract}

\onehalfspacing

\section{Introduction}
\input{intro}

\section{Finite-horizon \newedit{Bayesian} Multi-armed Bandit} \label{s-problem}
\input{problem}

\section{Information Relaxation Sampling} \label{s-framework}
\input{framework}

\section{Analysis} \label{s-analysis}
\input{analysis}

\section{Numerical Experiments} \label{s-numerical}
\input{numerical}

\section{Extensions} \label{s-extensions}
\input{extensions}

\section{Conclusion} \label{s-conclusion}
\input{conclusion}

\bibliography{info-relax-sampling}

\newpage

\appendix

\newpage
\input{appendix}

\end{document}

%% file: intro.tex
Dating back to the earliest work \citep{Bradt56, Gittins79}, multi-armed bandit (MAB) problems have been considered within a Bayesian framework, in which the unknown parameters are modeled as random variables drawn from a known prior distribution.
In this setting, the problem can be viewed as a Markov decision process (MDP) with a state that is an information state describing the beliefs of unknown parameters that evolve stochastically upon each play of an arm according to Bayes' rule.

Under the objective of expected performance, where the expectation is taken with respect to the prior distribution over unknown parameters, the (Bayesian) optimal policy (\textsc{Opt}) is characterized by Bellman equations immediately following from the MDP formulation.
In the discounted infinite-horizon setting, the celebrated Gittins index \citep{Gittins79} determines an optimal policy, despite the fact that its computation is still challenging.
In the non-discounted finite-horizon setting, which we consider, the problem becomes more difficult \citep{Berry85}, and except for some special cases, the Bellman equations are neither analytically nor numerically tractable, due to the curse of dimensionality.
In this paper, we focus on the Bayesian setting, and attempt to apply ideas from dynamic programming (DP) to develop tractable policies with good performance.

To this end, we apply the idea of \textit{information relaxation} \citep{Brown10}, a technique that provides a systematic way of obtaining the performance bounds on the optimal policy.
In multi-period stochastic DP problems, admissible policies are required to make decisions based only on previously revealed information.
The idea of information relaxation is to consider non-anticipativity as a constraint imposed on the policy space that can be relaxed, while simultaneously introducing a penalty for this relaxation into the objective, as in the usual Lagrangian relaxations of convex duality theory.
Under such a relaxation, the decision maker (DM) is allowed to access future information and is asked to solve an optimization problem so as to maximize her total reward, in the presence of penalties that punish any violation of the non-anticipativity constraint.
When the penalties satisfy a condition (dual feasibility, formally defined in \S \ref{s-framework}), the expected value of the maximal reward adjusted by the penalties provides an upper bound on the expected performance of the (non-anticipating) optimal policy.

The idea of relaxing the non-anticipativity constraint has been studied in different contexts
\citep{Rockafellar91, Karatzas94, Rogers02, Haugh04}, and was later formulated as a formal
framework by \cite{Brown10}, upon which our methodology is developed.  This framework has been
applied to a variety of applications including optimal stopping problems \citep{Moallemi11};
linear-quadratic and linear-convex control \citep{Desai12,Haugh12}; dynamic portfolio execution
\citep{Haugh14}; and more \citep[e.g.,][]{Brown17,haugh2019information}.  Typically, the
application of this method to a specific class of MDPs requires custom analysis. In particular, it
is not always easy to determine penalty functions that (1) yield a relaxation that is tractable to
solve, and (2) provide tight upper bounds on the performance of the optimal policy.  Moreover, the
established information relaxation theory focuses on upper bounds and provides no guidance on the
development of tractable policies.

Our contribution is to apply the information relaxation techniques to the finite-horizon stochastic MAB problem, explicitly exploiting the structure of a Bayesian learning process.
In particular,
\begin{enumerate}[topsep=0pt,itemsep=-1ex,partopsep=1ex,parsep=1ex]
\item we propose a series of information relaxations and penalties of increasing computational complexity;
\item we systematically obtain the upper bounds on the best achievable expected performance that trade off between tightness and computational complexity;
\item and we develop associated (randomized) policies that generalize Thompson sampling (\textsc{TS}) in the finite-horizon setting.
\end{enumerate}

In our framework, which we call \textit{information relaxation sampling}, each of the penalty functions (and information relaxations) determines one policy and one performance bound given a particular problem instance specified by the time horizon and the prior beliefs. 
As a base case for our algorithms, we have \textsc{TS} \citep{Thompson33} and the conventional regret benchmark that has been used for Bayesian regret analysis since \cite{Lai85}.
At the other extreme, the optimal policy \textsc{Opt} and its expected performance follow from the ``ideal'' penalty (which, not surprisingly, is intractable to compute).
By picking increasingly strict information penalties, we can improve the policy and the associated bound between the two extremes of \textsc{TS} and \textsc{Opt}.

As an example, one of our algorithms, \textsc{Irs.FH}, is a very simple modification of
\textsc{TS} that naturally incorporates time horizon $T$.  Recalling that \textsc{TS} makes a
decision based on sampled parameters for each arm from the posterior distribution in each epoch,
observe that knowledge of the parameters is essentially (assuming Bayesian consistency) as
informative as having an infinite number of future reward observations from each arm.  \newedit{By}
contrast, \textsc{Irs.FH} makes a decision based on future Bayesian estimates, updated with only
$T-1$ future reward realizations for each arm, where the rewards are sampled based on the inital
posterior belief.  When $T=1$ (equivalently, at the last decision epoch), such a policy takes
a myopically best action based only on the current estimates, which is indeed an optimal decision,
whereas \textsc{TS} would still explore unnecessarily.  While keeping the recursive structure of
the sequential decision-making process of \textsc{TS}, \textsc{Irs.FH} naturally performs less
exploration than \textsc{TS} as the remaining time horizon diminishes.  This mitigates a common
practical criticism of \textsc{TS}: it explores too much.

Beyond this, we propose other algorithms that more explicitly quantify the benefit of exploration
and more explicitly trade off between exploration and exploitation, at the cost of additional
computational complexity.  As we increase the complexity, we achieve policies that improve
performance, and separately provide tighter tractable computational upper bounds on the expected
performance of any policy for a particular problem instance.  By providing natural generalizations
of \textsc{TS}, our work provides both a deeper understanding of \textsc{TS} and improved policies
that do not require tuning.  Since \textsc{TS} has been shown to be asymptotically regret optimal
in some settings\newedit{, e.g., by the metric of growth-rate \citep{Kaufmann12b} or by the metric of worst-case regret \citep{Agrawal13, Bubeck13}}, our improvements can at best be (asymptotically)
constant factor improvements by that metric. On the other hand, \textsc{TS} is extremely popular
in practice, and we demonstrate in numerical examples that the improvements can be significant and
are likely to be of practical interest.

Moreover, we develop upper bounds on performance that are useful in their own right.
Suppose that a decision maker faces a particular problem instance and is considering any particular MAB policy (be it one we suggest or otherwise).
By simulating the policy, \newedit{we can find} a lower bound on the performance of the optimal policy\deledit{ can be found}.
We introduce a series of upper bounds that can also be evaluated in any problem instance via simulation.
Paired with the lower bound, these provide a computational, simulation-based ``confidence interval'' that can be helpful to the decision maker.
For example, if the upper bound and lower bound are close, the suboptimality gap of the policy under consideration is guaranteed to be small, and it is not worth investing in better policies.


%% file: problem.tex
\subsection{Problem} \label{ss-problem}

\newedit{
We consider a Bayesian MAB problem with $K$ \emph{independent arms} and a \emph{finite time horizon} $T$.
More specifically, we define an MAB instance with a tuple $\big( K, T, \Rscr, \Theta, \Pscr, \Yscr, \mathbf{y} \big)$ as follows.
In each period $t=1,\ldots,T$, the decision maker (DM) selects one among $K$ arms, each of which yields a stochastic reward whenever selected.
We let $\Ascr \defeq \{1, \ldots, K\}$ denote the set of arms, and let $R_{a,n}$ denote the random variable that represents the reward from the $n^\text{th}$ pull\footnote{
	\newedit{One may consider an alternative stochastic model for the reward realization process
          in which the rewards are defined through a time index (e.g., $R_{a,t}$ denotes the reward
          from arm $a$ in period $t$). This would be mathematically equivalent from the
          perspective of the DM. However, once the information set is relaxed, such a model is \emph{not}  equivalent to ours:
	in our model, the DM is not allowed to skip any future reward realizations, and this is crucial for some of the algorithms suggested in this paper.
	See the discussion in \S \ref{ss-irs-vzero}.
}} of arm $a \in \Ascr$.
For each arm $a$, the rewards $\{ R_{a,n} \}_{n \in \N}$ are independent and identically
distributed according to the distribution $\Rscr_a(\theta_a)$, where $\theta_a \in \Theta_a$ is the \emph{parameter} associated with arm $a$:
\begin{equation}
	R_{a,n} \sim \Rscr_a(\theta_a), \quad \forall n \in \N, \quad \forall a \in \Ascr.
\end{equation}
The parameter $\theta_a$ is unknown to the DM, and is modeled as a random variable for which we
have a family of \emph{conjugate priors} $\{ \Pscr_a(y_a) \}_{y_a \in \Yscr_a}$, i.e., a space of
distributions for $\theta_a$ that is closed under a Bayesian update with a reward realization $R_{a,n}$.
Given a \emph{hyperparameter} $y_a \in \Yscr_a$ (also called a \emph{belief}), consider a probability measure $\PR_{y_a}\left[ \cdot \right]$ under which the parameter $\theta_a$ follows the \emph{prior distribution} $\Pscr_a(y_a)$:
\begin{equation}
	\theta_a \sim \Pscr_a(y_a), \quad \forall a \in \Ascr.
\end{equation}
Let $\E_{y_a}\left[ \cdot \right]$ denote the expected value under this probability measure.
For brevity, denote the vector of parameters and hyperparameters across arms by $\bm{\theta} \defeq (\theta_1, \ldots, \theta_K)$ and $\mathbf{y} \defeq (y_1, \ldots, y_K)$, respectively.
Define $\Rscr$, $\Theta$, $\Pscr$, $\Yscr$, $\PR_\mathbf{y}$, and $\E_\mathbf{y}$ analogously.
We will often describe an MAB instance only with a tuple $(T, \mathbf{y})$ when the other components are clear in context.

Throughout the paper, we assume that the rewards are absolutely integrable for each hyperparameter $y_a \in \Yscr_a$:
\begin{equation}
	\E_{y_a}\left[ |R_{a,1}| \right] < \infty, \quad \forall y_a \in \Yscr_a, ~a \in \Ascr,
\end{equation}
where the expectation is taken with respect to the random realization of the parameter $\theta_a$
and also with respect to the random realization of the reward $R_{a,1}$.
}

We further define the \emph{outcome} $\omega \in \Omega$ (also referred to as the future or scenario) as a combination of the parameters and all future reward realizations, i.e.,
\begin{equation} \label{e-outcome}
	\omega \defeq \big( \bm{\theta}, ( R_{a,n} )_{a \in \Ascr, n \in \N} \big)
		~~ \sim ~~ \Iscr(\mathbf{y}),
\end{equation}
that encodes all \newedit{the} uncertainties that the DM encounters \newedit{in the environment and whose distribution is denoted by $\Iscr(\mathbf{y})$.}

\noindent \textbf{Policy.}
Given an outcome $\omega$, the reward at time $t$ can be represented as a function of the DM's action sequence $\mathbf{a}_{1:t} = (a_1, \ldots, a_t) \in \Ascr^t$, i.e.,
\begin{equation} \label{e-count}
	r_t( \mathbf{a}_{1:t}, \omega ) \defeq R_{a_t, n_t( \mathbf{a}_{1:t}, a_t) },
\end{equation}
where $n_t(\mathbf{a}_{1:t}, a) \defeq \sum_{s=1}^t \mathbf{1}\{ a_s = a \}$ counts how many times an arm $a$ has been played up to time $t$ (inclusive).
\newedit{
Consequently, we define the \emph{history} $H_t( \mathbf{a}_{1:t}, \omega )$ as the information revealed to the DM up to time $t$ when taking an action sequence $\mathbf{a}_{1:t}$ given the outcome $\omega$:
\begin{equation}
	H_t( \mathbf{a}_{1:t}, \omega ) \defeq \big( a_1, r_1(a_1, \omega ), a_2, r_2(\mathbf{a}_{1:2}, \omega), \ldots, a_t, r_t( \mathbf{a}_{1:t}, \omega) \big).
\end{equation}
Let $\mathbf{A}_{1:t}^\pi$ be the action sequence taken under the DM's policy $\pi$.
We can define the natural filtration $\mathbb{F} \defeq \big( \Fscr_t \big)_{t =0,1,\ldots,T}$ where $\Fscr_t \defeq \sigma\left( H_t( \mathbf{A}_{1:t}^\pi, \omega ) \right)$ is the $\sigma$-field generated by the history $H_t$.

A policy $\pi$ is called \textit{non-anticipating} if every action $A_t^\pi$ is measurable with respect to $\Fscr_{t-1}$; i.e., each decision is made based only on the information revealed prior to that time.
We denote by $\Pi_\mathbb{F}$ the set of all non-anticipating policies, including randomized ones.
The (Bayesian) \textit{performance} of a policy $\pi$ is measured by the total reward that $\pi$ earns on average, i.e.,
\begin{equation}
	V(\pi, T, \mathbf{y}) \defeq \E_{\mathbf{y}}\left[ \sum_{t=1}^T r_t( \mathbf{A}_{1:t}^\pi, \omega ) \right],
\end{equation}
where $T$ and $\mathbf{y}$ specify, respectively, the length of the time horizon and the prior
hyperparameters of given the MAB instance.
}

\noindent \textbf{Bayesian update.}
\newedit{Whenever the DM observes a reward realization, as a Bayesian learner, she can update her belief associated with the selected arm according to Bayes' rule.
More formally,} we introduce a \textit{Bayesian update function} $\Uscr_a: \Yscr_a \times \R \rightarrow \Yscr_a$ so that after observing a reward $r \in \R$ from an arm $a \in \Ascr$, the \newedit{hyperparameter} associated with arm $a$ is updated from $y_a$ to $\Uscr_a(y_a,r)$ \newedit{(e.g., if $\theta_a \sim \Pscr_a(y_a)$, then $\theta_a | R_{a,1} \sim \Pscr_a(\Uscr_a(y_a,R_{a,1}))$)}.
We will often use $\Uscr:\Yscr \times \Ascr \times \R \rightarrow \Yscr$ to denote the updating of the \newedit{hyperparameter} vector $\mathbf{y}$; i.e., after observing a reward realization $r$ from an arm $a$, the \newedit{hyperparameter} vector is updated from $\mathbf{y}$ to $\Uscr(\mathbf{y},a,r)$, where only the $a^\text{th}$ component is updated.

\newedit{
We further describe the time evolution of the DM's belief throughout the decision making process.
}
Given an outcome $\omega$ and an action sequence $\mathbf{a}_{1:t}$, the posterior hyperparameter vector at time $t$ can be recursively expressed as
\begin{equation} \label{e-belief-trajectory}
	\mathbf{y}_t( \mathbf{a}_{1:t}, \omega; \mathbf{y} ) \defeq \Uscr\left( \mathbf{y}_{t-1}(\mathbf{a}_{1:t-1}, \omega; \mathbf{y}), a_t, r_t(\mathbf{a}_{1:t}, \omega) \right),
	\quad \forall t \geq 1,
\end{equation}
\newedit{
with $\mathbf{y}_0 \defeq \mathbf{y}$.
We often write $[ \mathbf{y}_t( \mathbf{a}_{1:t}, \omega; \mathbf{y}) ]_a$ to denote the $a^\text{th}$ component of $\mathbf{y}_t( \mathbf{a}_{1:t}, \omega; \mathbf{y})$.
This hyperparameter vector $\mathbf{y}_t( \mathbf{a}_{1:t}, \omega; \mathbf{y} )$ sufficiently describes the DM's belief given the history $H_t( \mathbf{a}_{1:t}, \omega)$.}

\noindent \textbf{Mean reward.}
\newedit{
We introduce several notions of mean reward that play a crucial role throughout the paper.
For each arm $a \in \Ascr$, we let $\mu_a(\theta_a)$ denote the \emph{conditional mean reward} given the parameter $\theta_a$, and let $\bar{\mu}_a(y_a)$ be the \emph{predictive mean reward} given the hyperparameter $y_a$:
\begin{equation}
	\mu_a(\theta_a) \defeq \E\left[ R_{a,n} \big| \theta_a \right], \quad
	\bar{\mu}_a(y_a) \defeq \E_{y_a}\left[ \mu_a(\theta_a) \right].
\end{equation}
We further define the \emph{posterior predictive mean reward process} $\left\{ \hat{\mu}_{a,n} \right\}_{n \geq 0}$ by
\begin{equation}
	\hat{\mu}_{a,n}(\omega; y_a) \defeq \E_{y_a}\left[ \mu_a(\theta_a) | R_{a,1},\ldots,R_{a,n} \right],
\end{equation}
which represents the predictive mean reward \newedit{(i.e., the finite-sample Bayesian estimate of $\mu_a(\theta_a)$)} after observing first $n$ rewards associated with the arm $a$.

\begin{remark} \label{rem:mean-reward-martingale}
	Fix an arm $a \in \Ascr$.
	The posterior predictive mean reward process $\{ \hat{\mu}_{a,n} \}_{n \geq 0}$ is a martingale adapted to the filtration generated by the sequence of rewards $(R_{a,1}, R_{a,2}, R_{a,3}, \ldots)$.
	Furthermore, it starts at the value of the prior predictive mean reward $\bar{\mu}_a(y_a)$
        and converges to the conditional mean reward $\mu_a(\theta_a)$; i.e.,
        $\hat{\mu}_{a,0}(\omega; y_a) = \bar{\mu}_a(y_a)$ and $\lim_{n \rightarrow \infty}
        \hat{\mu}_{a,n}(\omega; y_a) = \mu_a(\theta_a)$ almost surely (see
        Proposition~\ref{prop-convergence} in the Appendix).
\end{remark}
}


\newedit{
\subsection{Natural Exponential Family} \label{ss-exponential-family}
We will often consider the case where the reward distribution $\Rscr_a(\theta_a)$ belongs to the \emph{natural exponential family}.
In this case, the closed-form expressions are available for the aforementioned notation.
For any given $\theta_a \in \Theta_a \subseteq \R$, the probability measure for a random reward $R_{a,n}$ is determined by
\begin{equation}
	\PR\left[ R_{a,n} \in dr \, | \, \theta_a \right] = h_a(dr) \exp\left( \theta_a r - A_a(\theta_a) \right),
\end{equation}
where $h_a(dr)$ is the \emph{reference measure} and $A_a(\cdot)$ is the \emph{log-partition function} that is a logarithm of the normalization factor.
We then have a family of conjugate priors $\{ \Pscr_a(y_a) \}_{y_a \in \Yscr_a}$ where $\Yscr_a \defeq \{ y_a = (\xi_a, \nu_a) | \xi_a \in \R, \nu > 0 \}$, so, for any given hyperparameter $y_a \in \Yscr_a$, the corresponding prior $\Pscr_a(y_a)$ is also an exponential family distribution and can be described as
\begin{equation}
	\PR_{(\xi_a, \nu_a)}\left[ \theta_a \in d\theta \right] = f_a( \xi_a, \nu_a ) \exp\left( \xi_a \theta - \nu_a A_a(\theta) \right) d\theta,
\end{equation}
where $f_a(\xi_a, \nu_a)$ is the normalization factor and $\nu_a$ represents the effective number of observations.
Within this family of conjugate priors, it is well known that the posterior distribution can be expressed as
\begin{equation}
	\PR_{(\xi_a, \nu_a)}\left[ \theta_a \in d\theta \, \left| \, R_{a,1}, \ldots, R_{a,n} \right. \right] 
		= \PR_{\left( \xi_a + \sum_{i=1}^n R_{a,i}, ~ \nu_a + n \right)}\left[ \theta_a \in d\theta \right].
\end{equation}
This property can also be expressed via the Bayesian update function as $\Uscr_a( (\xi_a, \nu_a), r ) = (\xi_a + r, \nu_a + 1)$.
We also have the following identities for the mean reward metrics:
\begin{equation}
	\mu_a(\theta_a) = A_a'(\theta_a) 
	, \quad 
	\bar{\mu}_a(\xi_a, \nu_a) = \frac{\xi_a}{\nu_a}
	, \quad
	\hat{\mu}_{a,n}(\omega; \xi_a, \nu_a) = \frac{ \xi_a + \sum_{i=1}^n R_{a,n} }{ \nu_a + n },
\end{equation}
where $A_a' \defeq dA_a/d\theta_a$.
We refer the reader to \citet{Pena12} for further details.
}

\noindent \textbf{Bernoulli and Gaussian MABs.}
\newedit{We briefly illustrate the Bernoulli MAB and Gaussian MAB as representative examples of
  the problem instance described by a natural exponential family.}
In the Bernoulli MAB, the rewards of an arm are Bernoulli random variables \newedit{whose success probability is drawn from a Beta distribution}.
In the Gaussian MAB, the rewards of an arm are normally distributed with an unknown mean and a known noise variance where \newedit{the mean is also normally distributed}.
Table \ref{tbl-example} summarizes the previously defined notation.

\begin{table}[H]
  \centering
  \small
\begin{tabular}{ c  c  c  }
 \toprule
 & \thead{Bernoulli MAB} & \thead{Gaussian MAB} \\
 \midrule
Prior distribution	& $\mu_a \sim \text{Beta}(\alpha_a, \beta_a)$ 	& $\mu_a \sim \Nscr(m_a, v_a^2)$ \\
Reward distribution 	& $R_{a,n} \sim \text{Bernoulli}\big( \mu_a \big)$ 		& $R_{a,n} \sim \Nscr(\mu_a, \sigma_a^2)$ \\
 \midrule
 Parameter $\theta_a$		& $\theta_a = \log \frac{\mu_a}{1-\mu_a}$		& $\theta_a = \frac{\mu_a}{\sigma_a^2}$ \\
 Hyperparameters $\xi_a, \nu_a$	& $\xi_a = \alpha_a$, $\nu_a = \alpha_a+\beta_a$	& $\xi_a = \frac{m_a \sigma_a^2}{v_a^2}$, $\nu_a = \frac{\sigma_a^2}{v_a^2}$ \\
 Reference measure $h_a$ 	& $h_a(dr) = \delta_0(dr) + \delta_1(dr)$ 		& $h_a(dr) = \frac{1}{\sqrt{2\pi \sigma_a^2}} \exp\left( - \frac{r^2}{\sigma_a^2} \right)dr$ \\
 Log-partition function $A_a$	& $A_a(\theta_a) = \log\left( 1 + e^{\theta_a} \right)$	& $A_a(\theta_a) = \frac{\sigma_a^2 \theta_a^2}{2}$ \\
 \midrule
Mean reward $\mu_a$ 	& $\mu_a(\theta_a) = \frac{e^{\theta_a}}{1+e^{\theta_a}}$ 	& $\mu_a(\theta_a) = \sigma_a^2 \theta_a$ \\
Predictive mean $\bar{\mu}_a$ 	& $\bar{\mu}_a(\alpha_a, \beta_a) = \frac{\alpha_a}{\alpha_a+\beta_a}$ 	& $\bar{\mu}_a(m_a, v_a^2) = m_a$ \\
 \bottomrule
\end{tabular}
\caption{Description of a Bernoulli MAB and a Gaussian MAB.
	Here, $\delta_x(dr)$ denotes a Dirac measure that has a single atom at $x$.
} \label{tbl-example}
\end{table}

\deledit{
In these MAB problems, the belief state of each arm is represented by a two-dimensional vector, $y_a = (\alpha_a, \beta_a)$ and $y_a = (m_a, \nu_a^2)$, respectively, which are the sufficient statistics for Beta distribution and Gaussian distribution.
More generally, when the reward distribution $\Rscr_a$ is a member of an exponential family, its conjugate prior $\Pscr_a$ can be represented by a low-dimensional vector $y_a$ (the sufficient statistics for $\Pscr_a$), and the Bayesian update function $\Uscr_a$ admits a simple closed form.
In the other cases, when the reward distributions do not belong to an exponential family, the belief $y_a$ may be an infinite-dimensional vector that represents the entire probability density of prior/posterior distribution, and there may not be a closed-form expression for $\Uscr_a$.
We note that the theoretical foundations of our framework do not rely on a parsimonious representation of the belief state nor a closed-form expression for the belief update function, which may concern in practice.
}

\subsection{Bayesian Optimal Policy}

In a Bayesian framework, the MAB problem \newedit{can be viewed as a Markov decision process (MDP) in which a state corresponds to an information state (or belief state) of the DM.
It has the following recursive structure that we will exploit throughout the paper.}
Given an MAB instance with time horizon $T$ and prior belief $\mathbf{y}$, suppose that the DM has just earned $r$ by pulling an arm $a$ at time $t=1$.
Then the remaining problem for the DM is equivalent to an MAB instance with time horizon $T-1$ and prior belief $\Uscr(\mathbf{y}, a, r)$.
Based on this Markovian structure, we obtain the following Bellman equations for the MAB problem: \newedit{for all $T \in \N$ and $\mathbf{y} \in \Yscr$,
\begin{align} \label{e-bellman}
	Q^*(T, \mathbf{y}, a) &\defeq \E_\mathbf{y}\left[  R_{a,1} + V^*( T-1, \Uscr( \mathbf{y}, a, R_{a,1} ) ) \right],
	\\
	V^*(T, \mathbf{y}) &\defeq \max_{a \in \Ascr} Q^*(T-1, \mathbf{y}, a),
\end{align}
with $V^*(0, \mathbf{y}) \defeq 0$ for all $\mathbf{y} \in \Yscr$.
The value function $V^*(T, \mathbf{y})$ represents the best possible performance that a non-anticipating policy can achieve in the MAB problem specified by the time horizon $T$ and the prior belief $\mathbf{y}$, or equivalently, \newedit{the maximum expected future reward that one can earn during $T$ remaining periods\footnote{\newedit{
We intentionally refrain from indexing the value function $V^*$ by time $t$, since such a representation conceals the Markovian structure of the Bayesian MAB problem and leads to complicated expressions for the variables that exploit this Markovian structure.
To avoid confusion, the horizon $T$ will be written as an argument to functions whereas the time index $t$ will be written as a subscript, throughout the paper.
}} when the current belief is $\mathbf{y}$}.
}

While Bellman equations are, in general, intractable to solve and directly apply, they offer a characterization of the \textit{Bayesian optimal policy}
(\textsc{Opt})\deledit{ and the best achievable performance $V^*$}.
At a certain moment, when the remaining time horizon is $T$ and the belief is $\mathbf{y}$, \textsc{Opt} takes an action with the largest state-action value (Q-value), i.e., pulls the arm $\newedit{A^*}=\argmax_a Q^*(T, \mathbf{y}, a)$, and this action selection procedure is repeated while updating $T$ and $\mathbf{y}$ according to Bayes' rule as described in Algorithm \ref{alg-opt}.
Such a policy achieves the best possible performance among all non-anticipating policies:
\begin{equation}
	V^*(T, \mathbf{y}) = \sup_{\pi \in \Pi_\mathbb{F}} V( \pi, T, \mathbf{y} ) = V( \textsc{Opt}, T, \mathbf{y} )
	, \quad 
	\newedit{\forall T \in \mathbb{N}, \mathbf{y} \in \Yscr.}
\end{equation}

\begin{algorithm2e}[H] \label{alg-opt}
  \SetAlgoLined\DontPrintSemicolon
  \SetKwFunction{algo}{OPT}
  \SetKwProg{myalg}{Function}{}{}
  \myalg{\algo{$T, \mathbf{y}$}}{
  \newedit{\tcp{$T$:remaining time horizon, $\mathbf{y}$:current belief}}
  \nl \KwRet $\argmax_a Q^*(T, \mathbf{y}, a)$ \;
  }{}
  \vspace{0.3cm}
  \setcounter{AlgoLine}{0}
  \SetKwFunction{proc}{OPT-Outer}
  \SetKwProg{myproc}{Procedure}{}{}
  \myproc{\proc{$T, \mathbf{y}$}}{
  	\newedit{\tcp{$T$:time horizon, $\mathbf{y}$:prior belief}}
  	\nl $\mathbf{y}_0 \gets \mathbf{y}$ \;
	\nl \For{$t = 1, 2, \ldots, T$}{
	\nl	Select $A_t \gets $ \algo{$T-t+1, \mathbf{y}_{t-1}$}\;
	\nl	Earn and observe a reward $r_t$ and update belief $\mathbf{y}_t \gets \Uscr( \mathbf{y}_{t-1}, A_t, r_t )$
	}
  }
  \caption{Bayesian optimal policy (\textsc{Opt})}
\end{algorithm2e}

\subsection{Thompson Sampling} \label{ss-ts}

Thompson sampling (\textsc{TS}) is a simple heuristic that makes decisions based on random sampling.
When the remaining time is $T$ and the current belief is $\mathbf{y}$, it samples the parameters $\tilde{\bm{\theta}}$ from the prior\footnote{
	Conventionally, the term ``posterior distribution'' is used to describe \newedit{the distribution that} \textsc{TS} samples the parameters from.
	We explicitly use ``prior distribution'' instead: for example, at time $t=1$, the parameters are apparently sampled from the prior, not the posterior, distribution.
	After observing a reward realization, we will have a posterior but it \newedit{will become} a prior at the next decision epoch.
} distribution at that moment, $\Pscr(\mathbf{y})$, and pulls the arm that is believed to be best given the sampled parameters $\tilde{\bm{\theta}}$\newedit{, i.e., takes action} $\newedit{A^\textsc{TS}} = \argmax_a \mu_a(\tilde{\theta}_a)$.
Like \textsc{Opt}, it repeats this procedure at every decision epoch while updating the belief $\mathbf{y}$ whenever a reward realization is observed.

\begin{algorithm2e}[H] \label{alg-ts}
  \SetAlgoLined\DontPrintSemicolon
  \SetKwFunction{algo}{TS}
  \SetKwProg{myalg}{Function}{}{}
  \myalg{\algo{T, $\mathbf{y}$}}{
  \newedit{\tcp{$T$:remaining time horizon, $\mathbf{y}$:current belief}}
  \nl Sample parameters $\tilde{\bm{\theta}} \sim \Pscr(\mathbf{y})$ \;
  \nl \KwRet $\argmax_a\{ \mu_a(\tilde{\theta}_a) \}$ \;
  }{}
  \caption{\newedit{Arm selection rule of Thompson sampling when remaining time is $T$ and current belief is $\mathbf{y}$}}
\end{algorithm2e}

Note that \textsc{TS} does not take into account \newedit{the time information} when making a decision.
It applies the identical sampling and selection rule, irrespective of the remaining time periods.
This often leads to the unnecessary explorations near the end of the horizon, which motivates our framework.


%% file: framework.tex
\newedit{
We apply the information relaxation framework \citep{Brown10} to the Bayesian MAB problem and propose a general framework which we call \textit{information relaxation sampling} (IRS).
The main idea behind the information relaxation is to relax the information constraint so that the decision maker (DM) is allowed to exploit some future information that is supposed to be unknown.
As in the usual Lagrangian relaxation, an upper bound on the best possible performance can be obtained by solving the relaxed problem.

To motivate in detail, let us consider a situation under which the parameters $\bm{\theta}$ are revealed to the DM when the remaining period is $T$ and the current belief is $\mathbf{y}$.
The optimal action for this DM is to keep playing the arm with the highest mean reward, i.e.,
$\argmax_a \mu_a(\theta_a)$, and by doing so will earn $\E_\mathbf{y}[ T \times \max_a \mu_a(\theta_a) ]$ on average, which is indeed an upper bound on the performance of the optimal policy, $V^*(T, \mathbf{y})$.

Let us now postulate a situation under which the same kind of DM is informed with sampled parameters $\tilde{\bm{\theta}}$ that are drawn from the distribution $\Pscr(\mathbf{y})$.
For this (falsely informed) DM, the optimal action is again to play the arm with the highest mean reward but now with respect to the sampled parameters, i.e., $\argmax_a \mu_a(\tilde{\theta}_a)$.
This procedure effectively describes the arm selection rule of Thompson sampling in the situation specified by the remaining horizon $T$ and the current belief $\mathbf{y}$.

Above, we motivated a performance bound, $\E_\mathbf{y}[ T \times \max_a \mu_a(\theta_a) ]$, and a non-anticipating policy, \textsc{TS}, from the relaxation of the parameter information.
Analogously, we can produce another performance bound and another policy by considering a different set of future information to relax: the performance bound is obtained by computing how much the clairvoyant DM can earn with this additional information; and the policy is obtained by speculating which action the same kind of DM will take if the additional information is replaced with sampled (simulated) instance.

We will particularly consider the relaxations of information that are less effective than the full
parameter information for the DM to maximize her future payoff. This will result in tighter
relaxations, in the sense of a better (tighter) performance upper bound as well as a better
performing policy.

In what follows, we formalize this idea utilizing the notion of \emph{information relaxation penalties} that allows us to describe and control the benefit from having additional information explicitly.
We will first describe the general framework and then propose a specific family of penalties that are particularly suitable for Bayesian MAB problems.
}

\deledit{
We propose a general framework, which we refer to as \textit{information relaxation sampling} (IRS), that takes as an input a ``penalty function'' $z_t(\cdot)$, and produces as outputs a policy $\pi^z$ and an associated performance bound $W^z$.
}

\noindent \textbf{Information relaxation penalties and the inner problem.}
Applying the information relaxation framework developed by \citet{Brown10}, we relax the non-anticipativity constraint imposed on policy space $\Pi_\mathbb{F}$ (i.e., \newedit{$A_t^\pi$} is $\Fscr_{t-1}$-measurable).
\newedit{Without loss of generality,\footnote{\newedit{
        Any partial information relaxation can be equivalently described within the perfect information relaxation by adding additional terms into the penalty function.
        See the discussion after Theorem \ref{thm-strong-duality}.
}} we consider the perfect information relaxation under which the DM is} allowed to first observe all future outcomes in advance, and then pick an action (i.e., \newedit{$A_t^\pi$} is $\sigma(\omega)$-measurable).
\newedit{As in any other Lagrangian relaxation,} we impose penalties on the DM for violating the non-anticipativity constraint.

We introduce a \textit{penalty function} $z_t(\mathbf{a}_{1:t}, \omega; T, \mathbf{y})$ to denote the penalty that the DM incurs at time $t$, when taking an action sequence $\mathbf{a}_{1:t}$ given \newedit{an outcome $\omega$ for an MAB problem with time horizon $T$ and prior belief $\mathbf{y}$}.
The clairvoyant DM can find the best action sequence that is optimal for \newedit{this} particular outcome $\omega$ in the presence of penalties $z_t$, by solving the following (deterministic) optimization problem, referred to as the \textit{inner problem}:
\begin{equation} \label{e-inner-problem}
        \maximize_{\mathbf{a}_{1:T} \in \Ascr^T} \quad \sum_{t=1}^T r_t( \mathbf{a}_{1:t}, \omega ) - z_t( \mathbf{a}_{1:t}, \omega; T, \mathbf{y} ).
        \tag{$*$}
\end{equation}

\begin{definition}[Dual feasibility] \label{defn-penalty-condition}
Given $T$ and $\mathbf{y}$, a penalty function $z_t$ is dual feasible if it is \newedit{a zero mean for any non-anticipating policy $\pi \in \Pi_\mathbb{F}$, i.e.,
\begin{equation} \label{e-penalty-condition}
         \E_\mathbf{y}\left[ \sum_{t=1}^T z_t( \mathbf{A}_{1:t}^\pi, \omega; T, \mathbf{y} ) \right] = 0
               , \quad \forall \pi \in \Pi_\mathbb{F}.
\end{equation}
}
\end{definition}
We remark that the mapping $\mathbf{a}_{1:t} \mapsto z_t(\mathbf{a}_{1:t}, \omega)$ is a stochastic function of the action sequence $\mathbf{a}_{1:t}$ since the outcome $\omega$ is random.
This dual feasibility condition requires that the DM who makes decisions on the natural filtration will receive zero penalties in expectation.

\newedit{
The complexity of the inner problem depends very much on the penalty function.
Assuming that the penalty function can be evaluated in $O(1)$ computation, an enumerative brute-force optimization of the inner problem may require $O(K^T)$ computations.
In what follows, we will illustrate that for suitably designed penalty functions, the inner problem exhibits a recursive structure and thus can be solved effectively using dynamic programming techniques.}
\\

\noindent \textbf{IRS performance bound.}
We let $W^z(T, \mathbf{y})$ be the expected maximal value of the inner problem \eqref{e-inner-problem}, when the outcome $\omega$ is randomly drawn from its prior distribution $\Iscr(\mathbf{y})$, i.e., the expected total payoff that a clairvoyant DM can achieve in the presence of penalties.:
\begin{equation} \label{e-bound}
        W^z( T, \mathbf{y} ) \defeq \newedit{\E_{\mathbf{y}}}\left[ \max_{\mathbf{a}_{1:T} \in \Ascr^T} \left\{  \sum_{t=1}^T r_t( \mathbf{a}_{1:t}, \omega ) - z_t( \mathbf{a}_{1:t}, \omega; T, \mathbf{y} ) \right\} \right].
\end{equation}
\newedit{Once we have an algorithm to solve the inner problem, this value can be computed numerically via simulation:}
let $\omega_1, \omega_2, \ldots, \omega_S$ be the samples independently drawn from $\Iscr(\mathbf{y})$, and $W_s$ be the the maximal value of the inner problem with respect to $\omega_s$ for each $s=1,\ldots,S$ separately.
The bound $W^z$ can be computed by taking the average of these maximal values, i.e., $\frac{1}{S} \sum_{s=1}^S W_s$.
The following theorem shows that $W^z$ is indeed a valid performance bound of the stochastic MAB problem.

\begin{theorem}[Weak duality and strong duality] \label{thm-weak-duality} \label{thm-strong-duality}
If the penalty function $z_t$ is dual feasible, $W^z$ is an upper bound on the optimal value $V^*$:
\begin{equation} \label{e-weak-duality}
        \text{(Weak duality)} \quad \quad
        W^z(T, \mathbf{y}) \geq V^*(T, \mathbf{y}).
        \quad \quad \quad \quad
\end{equation}
There exists a dual feasible penalty function denoted by $z_t^\textup{ideal}$, such that
\begin{equation} \label{e-strong-duality}
        \text{(Strong duality)} \quad \quad
        W^\textup{ideal}(T, \mathbf{y}) = V^*(T, \mathbf{y}).
        \quad \quad \quad \quad
\end{equation}
The ideal penalty function $z_t^\textup{ideal}$ has the following functional form:
\begin{align}\label{e-ideal-penalty}
        z_t^\textup{ideal}( \mathbf{a}_{1:t}, \omega; T, \mathbf{y} ) &\defeq r_t( \mathbf{a}_{1:t}, \omega ) -  \E_\mathbf{y}\left[ r_t( \mathbf{a}_{1:t}, \omega ) \left| \newedit{H_{t-1}( \mathbf{a}_{1:t-1}, \omega )} \right. \right]
                \\ &  \quad + V^*\left( T-t,  \mathbf{y}_t( \mathbf{a}_{1:t}, \omega; \mathbf{y} ) \right) - \E_\mathbf{y}\left[ \left. V^*\left( T-t, \mathbf{y}_t( \mathbf{a}_{1:t}, \omega; \mathbf{y} ) \right)   \right| \newedit{H_{t-1}( \mathbf{a}_{1:t-1}, \omega )} \right]. \nonumber
\end{align}
\end{theorem}
Recall that a dual feasible penalty function does not penalize (in expectation) non-anticipating policies, which include \textsc{Opt}.
Even when the future information is available, the DM can earn $V^*$ under the penalties by implementing \textsc{Opt} without taking advantage of future information.
When the DM makes use of future information, she can always outperform \textsc{Opt}, which leads to the weak duality result.
The ideal penalty $z_t^\text{ideal}$ precisely penalizes for the additional profit extracted from using the future information, thereby removing any incentive to deviate from \textsc{Opt} and resulting in the strong duality.

The ideal penalty is, of course, intractable, but its structure highlights what a good penalty may look like.
It implies that there are two sources of additional profit: in DP terminology, one from knowing future immediate rewards and one from knowing future state transitions, each of which will be taken into account later in this paper.

As another implication, it also shows that relaxing more the available information can always be compensated by adding associated terms to the penalty function.
That is, a partial information relaxation (e.g., \newedit{$A_t^\pi$} is measurable w.r.t. $\Gscr_{t-1}$ such that $\newedit{\sigma(H_{t-1})} \subseteq \Gscr_{t-1} \subseteq \sigma(\omega)$) with some penalty function $z_t^\mathbb{G}$ is equivalent to the perfect information relaxation (i.e., \newedit{$A_t^\pi$} is measurable w.r.t. $\sigma(\omega)$) with a penalty function $z_t^\mathbb{G} + z_t^{\sigma(\omega) \setminus \mathbb{G}}$ if the additional term $z_t^{\sigma(\omega) \setminus \mathbb{G}}$ exactly penalizes the relative benefit from having more information $\sigma(\omega)$ than $\Gscr_{t-1}$.
Hence, it is sufficient to consider the perfect information relaxation, as we do in this paper, and the actual amount of information available for the DM can be equivalently controlled by adjusting the penalty function.

Before proceeding, we remark that the above results are already well established in \citet{Brown10} (see Lemma 2.1 and Theorem 2.3 therein) for a general class of MDP problems, except for a subtle difference regarding the assumption on the predictability of reward realizations.
In MDP problems, the reward at each state is typically assumed to be deterministic (otherwise, it is replaced with its expected value), since the stochastic evolution of the state is of a major concern.
By contrast, in MAB problems it is essential to consider the randomness of rewards since learning from the noisy reward realizations is of a major concern, and therefore, we do not assume that $r_t$ is measurable with respect to $\newedit{\sigma(H_{t-1})}$.
As a consequence, our ideal penalty function \eqref{e-ideal-penalty} has a slightly different functional form than the one formulated in \citet{Brown10}.\footnote{
        \citet{Brown10} show that $z_t^\text{ideal} = V^*\left( T-t, \mathbf{y}_t \right) - \E\left[ \left. V^*\left( T-t, \mathbf{y}_t \right)   \right| H_{t-1} \right]$, when $r_t$ is assumed to be measurable with respect to $\sigma(H_{t-1})$ and so $r_t -  \E\left[ r_t \left| H_{t-1} \right. \right] = 0$.
}
We further exploit this fact when designing a variety of penalty functions.
\\

\noindent \textbf{IRS policy.}
\newedit{
Since the true outcome $\omega$ is not available in reality, it cannot be used in online decision making.
We derive a non-anticipating policy by leveraging the idea of ``posterior sampling,'' which utilizes the sampled outcome $\tilde{\omega}$ instead of the true outcome $\omega$.
}

Given a penalty function $z_t$, we characterize a randomized and non-anticipating IRS policy $\pi^z$ as follows.
\newedit{Exploiting the recursive structure of a Bayesian MAB problem,} the policy $\pi^z$ specifies ``which arm to pull when the remaining time is $T$ and the current belief is $\mathbf{y}$,'' \newedit{i.e., the very first action that it would take in an MAB instance with horizon $T$ and prior belief $\mathbf{y}$}.
Given $T$ and $\mathbf{y}$, it (i) first \newedit{randomly generates} an outcome $\tilde{\omega}$ \newedit{(i.e., sampling from $\Iscr(\mathbf{y})$)}, (ii) solves the inner problem to find a best action sequence $\tilde{\mathbf{a}}_{1:T}^*$ with respect to this randomly generated outcome $\tilde{\omega}$ in the presence of penalties $z_t$, and (iii) takes the first action $\tilde{a}_1^*$ that the clairvoyant optimal solution $\tilde{\mathbf{a}}_{1:T}^*$ suggests.
Analogous to \textsc{TS} and \textsc{Opt}, \textbf{it repeats steps (i)--(iii) at every decision epoch}, while updating the remaining time $T$ and belief $\mathbf{y}$ upon each decision making and reward realization.

\begin{algorithm2e}[H] \label{alg-irs}
  \SetAlgoLined\DontPrintSemicolon
  \SetKwFunction{algo}{IRS}
  \SetKwProg{myalg}{Function}{}{}
  \myalg{\algo{$T, \mathbf{y}; z$}}{
  \newedit{\tcp{$T$:remaining time horizon, $\mathbf{y}$:current belief}}
  \nl Sample an outcome $\tilde{\omega} \sim \Iscr(\mathbf{y})$: Equivalently, for each $a \in \Ascr$,
  $$ \tilde{\theta}_a \sim \Pscr_a(y_a), \quad \tilde{R}_{a,n} \sim \Rscr_a(\tilde{\theta}_a), \quad \forall n \in \{1, \ldots, T\}. $$
  \nl Find the best action sequence with respect to the sampled outcome $\tilde{\omega}$ under penalties $z_t$:
             $$ \tilde{\mathbf{a}}_{1:T}^* \gets \argmax_{\mathbf{a}_{1:T} \in \Ascr^T} \left\{ \sum_{s=1}^T r_s( \mathbf{a}_{1:t}, \tilde{\omega}) - z_s( \mathbf{a}_{1:s}, \tilde{\omega}; T, \mathbf{y} ) \right\}. $$
  \nl \KwRet $\tilde{a}_1^*$ \;
  }{}
  \vspace{0.3cm}
  \setcounter{AlgoLine}{0}
  \SetKwFunction{proc}{IRS-Outer}
  \SetKwProg{myproc}{Procedure}{}{}
  \myproc{\proc{$T, \mathbf{y}; z$}}{
        \newedit{ \tcp{$T$:time horizon, $\mathbf{y}$:prior belief}}
        \nl $\mathbf{y}_0 \gets \mathbf{y}$ \;
        \nl \For{$t = 1, 2, \ldots, T$}{
        \nl     Pull $A_t \gets $ \algo{$T-t+1, \mathbf{y}_{t-1}; z$}\;
        \nl     Earn and observe a reward $r_t$ and update belief $\mathbf{y}_t \gets \Uscr(\mathbf{y}_{t-1}, A_t, r_t )$
        }
  }
  \caption{Information relaxation sampling (IRS) policy}
\end{algorithm2e}

\newedit{
In step (i), the random generation of the outcome $\tilde{\omega}$ given the belief $\mathbf{y}$ is equivalent to, for each arm $a \in \Ascr$, sampling the parameter from its posterior, $\tilde{\theta}_a \sim \Pscr_a(y_a)$, and then sampling the future reward realizations, $\tilde{R}_{a,n} \sim \Rscr_a( \tilde{\theta}_a )$ for $n = 1, \ldots, T$.
In other words,} the IRS policy $\pi^z$ randomly generates (simulates) a plausible future scenario \newedit{within its own probability space specified by $T$ and $\mathbf{y}$.}

\newedit{The optimization problem in the step (ii) is identical to the inner problem
  \eqref{e-inner-problem} except that the true outcome $\omega$ is replaced with the sampled one
  $\tilde{\omega}$.  Therefore, the dynamic programming algorithm that solves the inner problem
  can also be utilized for this online decision-making process, not only for the computation of
  performance bound $W^z$. Note that there can be multiple solutions to this
  optimization problem and the tie-breaking rule may affect the performance of the policy.  We do
  not observe that the choice of tie-breaking rule is significance in our numerical
  experiments. In some instances that follow, however, we will adopt a specific tie-breaking rule
  for the purpose of theoretical analysis.}

Also note that \newedit{in step (iii)} only the first action $\tilde{a}_1^*$ of the optimal solution $\tilde{\mathbf{a}}_{1:T}^*$ is utilized, and at the following decision epoch a new outcome is sampled based on the updated belief.
If we consider an MAB instance with time horizon $T$, the policy $\pi^z$ solves $T$ different instances of the inner problem throughout the entire decision\newedit{-making} process, \newedit{with a decreasing length of time horizon}, from $T$ to 1, \newedit{and with a stochastically evolving belief state}.
\newedit{See the \textsc{Irs-Outer} procedure in Algorithm \ref{alg-irs}, which is in fact
  identical to that employed in \textsc{Opt} and \textsc{TS}.}

\begin{remark} \label{rem-ideal-penalty-optimality}
The ideal penalty yields the Bayesian optimal policy, i.e., $\pi^\textup{ideal} = \textsc{Opt}$.
\end{remark}

\newedit{Recall that the ideal penalty \eqref{e-ideal-penalty} yields the performance bound $W^\textup{ideal}$ that is equal to the best achievable performance $V^*$, because the DM under the ideal penalty has no incentive to utilize any future information.
For the same reason, the corresponding IRS policy $\pi^\textup{ideal}$ does not utilize the (randomly generated) future information in its decision making, and tries to make the best decision based only on the information revealed so far.
Therefore, its decision should always coincide with the Bayesian optimal policy's decision.}
\\

\noindent \textbf{Choice of penalty functions.}
\newedit{
We have so far described the general framework that takes a penalty function $z_t$ as input, and yields a performance bound $W^z$ and a policy $\pi^z$ as outputs.
While any dual feasible penalty functions can be utilized in general, we propose the following set of penalty functions that are particularly suitable for the MAB problems:}

\begin{align}
        \label{e-penalty-ts}
        z_t^\textsc{TS}( \mathbf{a}_{1:t}, \omega ) &\defeq r_t( \mathbf{a}_{1:t}, \omega ) - \E\left[ r_t( \mathbf{a}_{1:t}, \omega ) \left| \bm{\theta} \right. \right],
        \\
        \label{e-penalty-irs-fh}
        z_t^\textsc{Irs.FH}( \mathbf{a}_{1:t}, \omega ) &\defeq r_t( \mathbf{a}_{1:t}, \omega ) - \newedit{\E_\mathbf{y}}\left[ r_t( \mathbf{a}_{1:t}, \omega ) \left| \newedit{\hat{\bm{\mu}}_{T-1}(\omega)} \right. \right],
        \\
        \label{e-penalty-irs-vzero}
        z_t^\textsc{Irs.V-Zero}( \mathbf{a}_{1:t}, \omega ) &\defeq r_t( \mathbf{a}_{1:t}, \omega) - \newedit{\E_\mathbf{y}}\left[ r_t( \mathbf{a}_{1:t}, \omega ) \left| \newedit{H_{t-1}( \mathbf{a}_{1:t-1}, \omega )} \right. \right],
        \\
        \label{e-penalty-irs-vemax}
        z_t^\textsc{Irs.V-EMax}( \mathbf{a}_{1:t}, \omega ) &\defeq r_t( \mathbf{a}_{1:t}, \omega )  -  \newedit{\E_\mathbf{y}}\left[ r_t( \mathbf{a}_{1:t}, \omega ) \left| \newedit{H_{t-1}( \mathbf{a}_{1:t-1}, \omega )} \right. \right]
                \\ &  \quad + W^\textsc{TS}\left( T-t, \mathbf{y}_t( \mathbf{a}_{1:t}, \omega ) \right) - \newedit{\E_\mathbf{y}}\left[ \left. W^\textsc{TS}\left( T-t, \mathbf{y}_t( \mathbf{a}_{1:t}, \omega ) \right) \right| \newedit{H_{t-1}( \mathbf{a}_{1:t-1}, \omega )} \right], \nonumber
\end{align}
\newedit{where $\hat{\bm{\mu}}_{T-1}(\omega; \mathbf{y}) \defeq \big( \hat{\mu}_{a,T-1}(\omega; y_a)  \big)_{a \in \Ascr}$ and the dependency of some expressions on $T$ and $\mathbf{y}$ is suppressed for clarity.
Also recall that the ideal penalty is given by
\begin{align} \label{e-penalty-ideal}
        z_t^\textup{ideal}( \mathbf{a}_{1:t}, \omega ) &\defeq r_t( \mathbf{a}_{1:t}, \omega ) -  \E_\mathbf{y}\left[ r_t( \mathbf{a}_{1:t}, \omega ) \left| H_{t-1}( \mathbf{a}_{1:t-1}, \omega ) \right. \right]
                \\ &  \quad + V^*\left( T-t,  \mathbf{y}_t( \mathbf{a}_{1:t}, \omega ) \right) - \E_\mathbf{y}\left[ \left. V^*\left( T-t, \mathbf{y}_t( \mathbf{a}_{1:t}, \omega ) \right)   \right| H_{t-1}( \mathbf{a}_{1:t-1}, \omega ) \right]. \nonumber
\end{align}
We can show that these penalty functions satisfy the dual feasibility condition (Definition \ref{defn-penalty-condition}); see Appendix~\ref{prf-dual-feasible} for a formal proof.
}
\begin{remark} \label{rem-dual-feasibility}
All penalty functions \eqref{e-penalty-ts}--\eqref{e-penalty-ideal} are dual feasible.
\end{remark}

\newedit{
This set of penalty functions results in a set of policies that ranges from Thompson sampling (\textsc{TS}) to the Bayesian optimal policy (\textsc{Opt}) and a set of performance bounds that ranges from the conventional regret benchmark $W^\textsc{TS}$ ($=\E[ T \times \max_a \mu_a(\theta_a) ]$) to the optimal value function $W^\textup{ideal}$ ($=V^*$).
More specifically, at one extreme, the simplest penalty function $z_t^\textsc{TS}$ yields \textsc{TS} and $W^\textsc{TS}$ as outputs, and at the other extreme, the ideal penalty function $z_t^\textup{ideal}$ yields \textsc{Opt} and $V^*$ which would be optimal.
The other three penalty functions ($z_t^\textsc{Irs.FH}$, $z_t^\textsc{Irs.V-Zero}$, and
$z_t^\textsc{Irs.V-EMax}$) connect the two extremes and are sequentially ``better'', where we informally say that a penalty function is better than another if it is closer to the ideal penalty function and thus yields a better performing policy and a tighter performance bound.
Deferring detailed explanations to \S \ref{ss-ts-revisited}--\S \ref{ss-irs-vemax}, we briefly illustrate general principles to design ``good'' penalty functions and motivate these penalty functions.
}

\newedit{
In design of information relaxation penalties, we first need to determine to which information set we relax the non-anticipativity constraint, i.e., what kind of additional information will be revealed to the DM in the relaxation.
Although we have described our framework based on the perfect information relaxation (i.e., the relaxation in which the DM perfectly knows the entire future outcomes $\omega$), any imperfect information relaxation can be equivalently described within the perfect information relaxation using a properly constructed penalty function.\footnote{\newedit{
	In fact, this is the main idea underlying the existence of the ideal penalty function; see the discussion after Theorem \ref{thm-strong-duality}.}}
Among the suggested penalty functions,\footnote{\newedit{
	We can motivate one more penalty function that corresponds to the perfect information
        relaxation. Such a penalty function is simply given by $z_t \equiv 0$, which is
        illustrated in Appendix~\ref{s-example}. However, we do not suggest it's use since it
        is even worse than $z_t^\textsc{TS}$.}}
$z_t^\textsc{TS}$ is the one that corresponds to the information relaxation to the parameter information $\bm{\theta}$, $z_t^\textsc{Irs.FH}$ corresponds to the information relaxation to the posterior predictive mean rewards $\hat{\bm{\mu}}_{T-1}$ (i.e., the finite-sample mean-reward estimates),
and $z_t^\textup{ideal}$ corresponds to no information relaxation. 
}

\newedit{
One principle to motivate a better penalty function is to choose a smaller set of future information for the relaxation.
When less additional information is revealed to the DM in the relaxation, the additional profit that the DM can extract from this information becomes smaller, and hence the DM has to make more realistic decisions that rely more on the currently available information rather than the future information that is supposed to be unknown.
Comparing $z_t^\textsc{Irs.FH}$ with $z_t^\textsc{TS}$, one may notice that the finite-sample mean-reward estimates $\hat{\bm{\mu}}_{T-1}$ are less informative than the parameters $\bm{\theta}$ for the DM to exploit in her profit maximization because, in terms of mean-reward estimation, the parameters are informative as much as an infinite number of observations (i.e., $\E[ \mu_a(\theta_a) | \bm{\theta} ] = \lim_{T \rightarrow \infty} \E[ \mu_a(\theta_a) | R_{a,1}, \ldots, R_{a,T-1} ] = \lim_{T \rightarrow \infty} \hat{\mu}_{a,T-1}$).
In this sense, $z_t^\textsc{Irs.FH}$ is better than $z_t^\textsc{TS}$, and resulting policy $\pi^\textsc{Irs.FH}$ and performance bound $W^\textsc{Irs.FH}$ improve upon \textsc{TS} and $W^\textsc{TS}$ toward \textsc{Opt} and $V^*$.
}

\newedit{
Another principle to motivate a better penalty function is to adopt a more precise approximation of the ideal penalty function $z_t^\textup{ideal}$, particularly regarding the terms containing the optimal value function $V^*$.
In the presence of penalties that reflect the value of the additional information more accurately, the DM has less incentive to exploit this additional information in the relaxed decision making problem, and similarly to the above argument, this leads to more realistic decisions.
Among our suggestions, $z_t^\textsc{Irs.V-Zero}$ approximates the term $V^*$ with zero, and
$z_t^\textsc{Irs.V-EMax}$ approximates the term $V^*$  with a tractable upper bound $W^\textsc{TS}$.
By doing so, $z_t^\textsc{Irs.V-EMax}$ takes into account the continuation value of each action explicitly and improves upon  $z_t^\textsc{Irs.V-Zero}$.
}

\newedit{
Consider the inner problem associated with each choice of penalty function \eqref{e-penalty-ts}--\eqref{e-penalty-ideal}.
Recall that each inner problem is a deterministic multi-period decision making problem that has a form of $\max_{\mathbf{a}_{1:T} \in \Ascr^T} \sum_{t=1}^T r_t(\mathbf{a}_{1:t}) - z_t(\mathbf{a}_{1:t})$.
A penalty function $z_t$ effectively redefines what the DM earns at each time, i.e., $r_t(\mathbf{a}_{1:t})$ is replaced with $r_t(\mathbf{a}_{1:t}) - z_t(\mathbf{a}_{1:t})$.
More specifically, the penalty function $z_t^\textsc{TS}$ effectively replaces the realized rewards associated with each arm with their expected value given parameters $\bm{\theta}$;
as does $z_t^\textsc{Irs.FH}$ (with their expected value given the finite-sample mean-reward estimates $\hat{\bm{\mu}}_{T-1}$);
as does $z_t^\textsc{Irs.V-Zero}$ (with their expected value conditional on how many times the arm has previously been selected up to each point in time).
}

\begin{table}[H]
  \centering
  \small
\begin{tabular}{ c c c l c }
\toprule
\thead{Penalty\\function} & \thead{Policy} & \thead{Performance\\bound} & \thead{Inner problem} &
                                                                                                  \thead{Run time} \\
  \midrule
$z_t^\textsc{TS}$               &  \textsc{TS}          & $W^\textsc{TS}$                       & Find a best arm given parameters.             & $O(K)$        \\
$z_t^\textsc{Irs.FH}$           &  $\pi^\textsc{Irs.FH}$                & $W^\textsc{Irs.FH}$           & Find a best arm given finite observations.    & \makecell{ $O(K)$ or\\$O(KT)$ } \\
$z_t^\textsc{Irs.V-Zero}$       &  $\pi^\textsc{Irs.V-Zero}$    & $W^\textsc{Irs.V-Zero}$       & Find an optimal allocation of $T$ pulls.              & $O(KT^2)$                     \\
$z_t^\textsc{Irs.V-EMax}$       & $\pi^\textsc{Irs.V-EMax}$     & $W^\textsc{Irs.V-EMax}$       & Find an optimal action sequence.                      & $O(KT^K)$      \\
$z_t^\text{ideal}$              &  \textsc{Opt}                 &  $V^*$                                        & Solve Bellman equations.                              & --             \\
 \bottomrule
\end{tabular}
\caption{
 List of algorithms following from penalty functions \eqref{e-penalty-ts}--\eqref{e-penalty-ideal}.
  \textsc{TS} refers to Thompson sampling and \textsc{Opt} refers to the Bayesian optimal policy.
  Run time represents the computational complexity of solving one instance of the inner problem \eqref{e-inner-problem}, that is, the time required to obtain one sample \newedit{in a computation of} performance bound $W^z$ or to \newedit{decide which arm to select in each period in a run of} policy $\pi^z$.
} \label{tbl-summary}
\end{table}

Table \ref{tbl-summary} summarizes \newedit{these inner problems}.
As we sequentially increase the computational complexity of a penalty function, from $z_t^\textsc{TS}$ to $z_t^\textup{ideal}$, the penalty function more accurately penalizes the benefit from knowing future outcomes, i.e., more explicitly prevents the DM from exploiting future information.
As a result, the inner problem becomes closer to the original stochastic optimization problem, which results in a better performing policy and a tighter performance bound.
Using this approach, we achieve a family of algorithms that are intuitive and tractable, exhibiting a trade-off between quality and computational efficiency.
See Appendix \S \ref{s-example} for an illustrative example. 

The run time in Table \ref{tbl-summary} represents the computational complexity of solving one instance of the inner problem, i.e., the time it takes to obtain one sample \newedit{in a computation of} performance bound $W^z$ or to \newedit{decide which arm to select in each period in a run of} policy $\pi^z$.
In this run-time analysis, performing the Bayesian belief updating and the sampling of a random variable is counted as a single operation.

\subsection{Thompson Sampling Revisited} \label{ss-ts-revisited}

With the penalty function $z_t^\textsc{TS}( \mathbf{a}_{1:t}, \omega ) \defeq r_t( \mathbf{a}_{1:t}, \omega ) - \mu_{a_t}( \theta_{a_t} )$, the inner problem \eqref{e-inner-problem} reduces to
\begin{equation}
        \max_{\mathbf{a}_{1:T} \in \Ascr^T} \left\{ \sum_{t=1}^T r_t( \mathbf{a}_{1:t}, \omega ) - z_t( \mathbf{a}_{1:t}, \omega ) \right\}
                = \max_{\mathbf{a}_{1:T} \in \Ascr^T} \left\{ \sum_{t=1}^T \mu_{a_t}( \theta_{a_t} ) \right\}
                = T \times \max_{a \in \Ascr} \mu_a(\theta_a).
\end{equation}
Given an outcome $\omega$, and in the presence of penalties, a hindsight optimal action sequence is to keep pulling the true best arm, i.e., $\argmax_a \mu_a(\theta_a)$, for $T$ times in a row.
The resulting performance bound $W^\textsc{TS}$ reduces to the conventional performance benchmark,
\begin{equation} \label{e-full-info}
       W^\textsc{TS}( T, \mathbf{y} ) = \newedit{\E_\mathbf{y}}\left[ T \times \max_{a \in \Ascr} \mu_a(\theta_a) \right],
\end{equation}
which measures how much the DM could have achieved if the parameters had been revealed in advance.

\begin{remark}
        The performance bound $W^\textsc{TS}$ is the conventional benchmark that has been widely used in the Bayesian regret analysis \citep{Lai85, Russo14, Russo17}.
        The Bayesian regret of a policy $\pi$ is defined as
        \begin{equation}
                \text{BayesRegret}( \pi, T, \mathbf{y} )
                        \defeq \newedit{\E_\mathbf{y}}\left[ \sum_{t=1}^T \max_a \mu_a(\theta_a) - \mu_{A_t^\pi}(\theta_{A_t^\pi}) \right]
                        = W^\textsc{TS}( T, \mathbf{y} ) - V( \pi, T, \mathbf{y} ),
        \end{equation}
        which quantifies the suboptimality of the policy $\pi$.
\end{remark}

It is trivial to see that the corresponding policy $\pi^\textsc{TS}$ is equivalent to Thompson sampling.
The policy $\pi^\textsc{TS}$ utilizes a sampled outcome $\tilde{\omega}$ instead of the true outcome $\omega$;
accordingly, it selects an arm $\newedit{A^\textsc{TS}} = \argmax_a \mu_a(\tilde{\theta}_a)$, where $\tilde{\bm{\theta}} \sim \Pscr(\mathbf{y})$, which is identical to the procedure described in Algorithm \ref{alg-ts}.
In order for the policy $\pi^\textsc{TS}$ to make a decision at a certain time, note that it does not need to sample future rewards, and thus it requires $O(K)$ computations only.

\subsection{IRS.FH} \label{ss-irs-fh}

Recall that $\hat{\mu}_{a,T-1}(\omega; y_a)$ is the \newedit{posterior predictive mean reward} of an arm $a$ that the DM \newedit{will have after observing} $T-1$ reward realizations $R_{a,1}, \ldots, R_{a,T-1}$ given the initial belief $y_a$:
\begin{equation}
        \hat{\mu}_{a,T-1}(\omega; y_a) \defeq \newedit{\E_{y_a}}\left[ \mu_a(\theta_a) \left| R_{a,1}, \ldots, R_{a,T-1} \right. \right].
\end{equation}
Given \eqref{e-penalty-irs-fh}, the optimal solution to the inner problem \eqref{e-inner-problem} is to always pull the arm \newedit{with the highest posterior predictive mean reward}, i.e., $\argmax_a \hat{\mu}_{a,T-1}(\omega;y_a)$:
\begin{equation}
        \max_{\mathbf{a}_{1:T} \in \Ascr^T} \left\{ \sum_{t=1}^T r_t( \mathbf{a}_{1:t}, \omega ) - z_t^\textsc{Irs.FH}( \mathbf{a}_{1:t}, \omega ) \right\}
                = \max_{\mathbf{a}_{1:T} \in \Ascr^T} \left\{ \sum_{t=1}^T \hat{\mu}_{a_t,T-1}( \omega ) \right\}
                = T \times \max_{a \in \Ascr} \hat{\mu}_{a,T-1}(\omega).
\end{equation}
This inner problem yields the performance bound $W^\textsc{Irs.FH}$, such that
\begin{equation}
        W^\textsc{Irs.FH}(T, \mathbf{y}) = \newedit{\E_\mathbf{y}}\left[ T \times \max_{a \in \Ascr} \hat{\mu}_{a,T-1}(\omega; y_a) \right],
\end{equation}
and the policy $\pi^\textsc{Irs.FH}$ that is implemented in Algorithm \ref{alg-irs-fh}.

\begin{algorithm2e}[H] \label{alg-irs-fh}
  \SetAlgoLined\DontPrintSemicolon
  \SetKwFunction{algo}{IRS.FH}
  \SetKwProg{myalg}{Function}{}{}
  \myalg{\algo{$T, \mathbf{y}$}}{
  \newedit{\tcp{$T$:remaining time horizon, $\mathbf{y}$:current belief}}
  \nl Sample parameters $\tilde{\bm{\theta}} \sim \Pscr(\mathbf{y})$ and rewards $\tilde{R}_{a,n} \sim \Rscr_a(\tilde{\theta}_a)$, $\forall n \in \{1, \ldots, T\}$, $\forall a \in \Ascr$.\;
  \nl \KwRet $\argmax_a \left\{ \newedit{\E_{y_a}}\left[ \mu_a(\theta_a) \left| R_{a,1}=\tilde{R}_{a,1}, \ldots, R_{a,T-1} = \tilde{R}_{a,T-1} \right. \right] \right\}$ \;
  }{}
  \caption{Arm selection rule of $\pi^\textsc{Irs.FH}$ when remaining time is $T$ and current belief is $\mathbf{y}$}
\end{algorithm2e}

\textsc{Irs.FH} (\textsc{FH} stands for finite horizon) is almost identical to \textsc{TS} except that \newedit{the conditional mean reward} $\mu_a(\theta_a)$ is replaced with \newedit{the posterior predictive mean reward} $\hat{\mu}_{a,T-1}(\omega)$.
\newedit{As a finite-sample Bayesian estimate of the conditional mean reward, $\hat{\mu}_{a,T-1}(\omega)$ is less informative than $\mu_a(\theta_a)$ from the DM's perspective}.
In terms of mean reward estimation, \newedit{the DM will never be able to identify $\mu_a(\theta_a)$ perfectly within a finite horizon, i.e.,} knowing the parameters is equivalent to having an infinite number of observations.
The inner problem of \textsc{TS} requires the DM to ``identify the best arm based on an infinite number of samples,'' whereas that of \textsc{Irs.FH} requires \newedit{the DM} to ``identify the best arm based on a finite number of samples'' and takes into account the length of the time horizon explicitly.
By restricting the DM's access to fewer information, \newedit{\textsc{Irs.FH}} requires the DM to be more realistic, that is, to consider the uncertainties more precisely.

\begin{figure}[!h]
        \centering
        \includegraphics[width=1.0\linewidth]{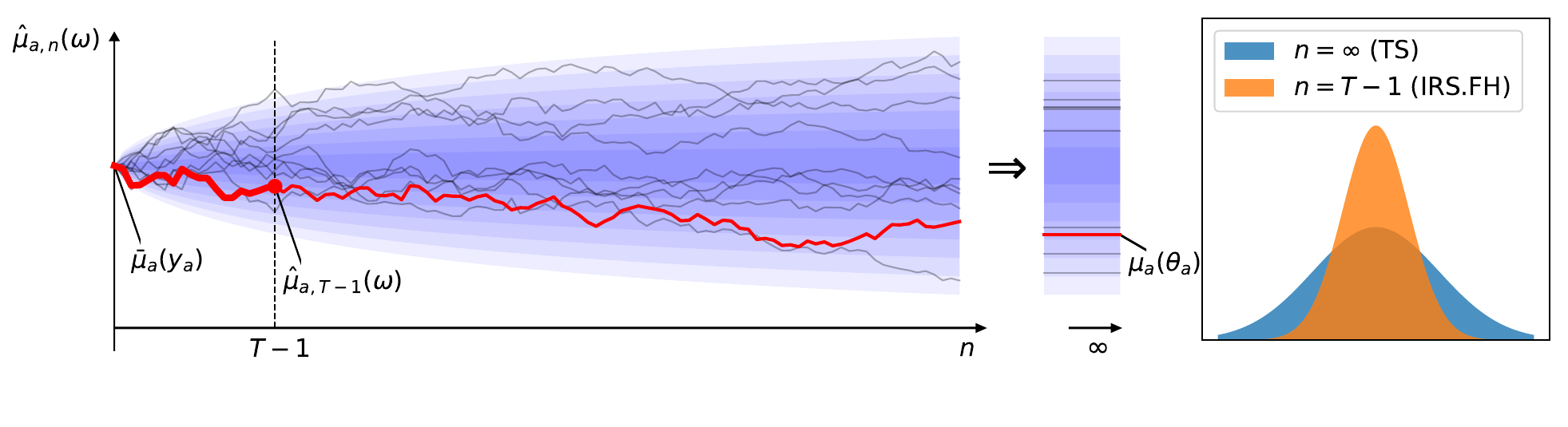}
        \caption{ (Left) Sample paths of \newedit{posterior predictive mean reward process} of an arm $a$, $\big\{ \hat{\mu}_{a,n}(\omega) \big\}_{n \geq 0}$. \newedit{This process is a martingale that starts at (prior) predictive mean $\bar{\mu}_a$ and converges to conditional mean $\mu_a$ (Remark \ref{rem:mean-reward-martingale}).}
                (Right) The distributions of $\hat{\mu}_{a,T-1}$ and $\mu_a$: $\hat{\mu}_{a,T-1}$ is more concentrated than $\mu_a$, while all have the same mean $\bar{\mu}_a(y_a)$.
        }
        \label{fig-estimate-trajectory}
\end{figure}

To sharpen our comparison between \textsc{Irs.FH} and \textsc{TS}, let us compare the variance of $\hat{\mu}_{a,T-1}(\omega)$ and $\mu_a(\theta_a)$ \newedit{induced by the randomness of outcome $\omega$.
As depicted in Figure \ref{fig-estimate-trajectory}, $\mu_a$ is more widely distributed than $\hat{\mu}_{a,T-1}$ because a larger (infinite vs. $T-1$) number of samples makes it easier for the posterior to deviate from the initial prior (see also Remark \ref{rem:mean-reward-martingale}).}
By Jensen's inequality, we further have $W^\textsc{Irs.FH} = \E[ T \times \max_a \hat{\mu}_{a,T-1}(\omega) ] \leq \E[ T \times \max_a \mu_a(\theta_a) ] = W^\textsc{TS}$ for any problem instance, meaning that \textsc{Irs.FH} yields a performance bound that is tighter than the conventional benchmark.
Also note that the same argument holds for the comparison between $\hat{\mu}_{a,T-1}(\tilde{\omega})$ and $\mu_a(\tilde{\theta}_a)$ since the synthesized outcome $\tilde{\omega}$ is identically distributed with the (true) outcome $\omega$.
The variability of $\hat{\mu}_{a,T-1}(\tilde{\omega})$ (respectively, $\mu_a(\tilde{\theta}_a)$) governs the randomness of the action taken by policy $\pi^\textsc{Irs.FH}$ (resp., $\pi^\textsc{TS}$), i.e., $\newedit{A^\textsc{Irs.FH}} = \argmax_a \hat{\mu}_{a,T-1}(\tilde{\omega})$ (resp., $\newedit{A^\textsc{TS}} = \argmax_a \mu_a(\tilde{\theta}_a)$).
Given $T$ and $\mathbf{y}$, the policy $\pi^\textsc{Irs.FH}$ performs fewer random explorations than \textsc{TS}, as it is less likely to deviate from the myopic decision to play an arm with the largest current estimate $\bar{\mu}_a(y_a)$.
More desirably, the degree of exploration of $\pi^\textsc{Irs.FH}$ is controlled by the remaining time horizon as the variance of $\hat{\mu}_{a,T-1}(\omega)$ depends on $T$.
At the last decision epoch ($T=1$), $\pi^\textsc{Irs.FH}$ takes a myopic action that is indeed optimal.

\noindent \textbf{\newedit{Efficiently sampling $\hat{\mu}_{a,T-1}(\tilde{\omega})$ for natural exponential families}.}
In order to obtain $\hat{\mu}_{a,T-1}(\tilde{\omega})$ for each arm $a$ for a synthesized outcome $\tilde{\omega}$, one may apply Bayes' rule sequentially for each reward realization, which will take \newedit{$O(T)$} computations \newedit{per arm}.

\newedit{
As discussed in \S \ref{ss-exponential-family}, in an MAB were the reward distribution is a
natural exponential family, the posterior predictive mean reward is given by
        \begin{equation}
                \hat{\mu}_{a,T-1}( \tilde{\omega}; \xi_a, \nu_a ) = \frac{ \xi_a + \sum_{n=1}^{T-1} \tilde{R}_{a,n} }{ \nu_a + T-1 }.
        \end{equation}
Therefore, it is sufficient to sample the sum of $T-1$ future rewards, $\sum_{n=1}^{T-1}
\tilde{R}_{a,n}$, in otder to sample the posterior predictive mean reward.
Observe that the conditional distribution of the sum given $\tilde{\theta}_a$ also belongs to the natural exponential family, induced by a log-partition function $(T-1) A_a(\tilde{\theta}_a)$.
This distribution may be tractable to compute: for example, its distribution is $\text{Binomial}( T-1, \mu_a(\tilde{\theta}_a))$ for the Beta-Bernoulli case, and $\Nscr\big( (T-1) \cdot \mu_a(\tilde{\theta}_a), (T-1) \cdot \sigma_a^2 \big)$ for the Gaussian case.
In these settings, we can sample the sum $\sum_{n=1}^{T-1} \tilde{R}_{a,n}$ directly from the tractable distribution (after sampling $\tilde{\theta}_a$) using $O(1)$ computation, and then use it to compute $\hat{\mu}_{a,T-1}(\tilde{\omega})$ without sequentially updating the belief.
In such cases, a single decision of $\pi^\textsc{Irs.FH}$ can be made within $O(K)$ operations, independent of $T$, similar in computational complexity to \textsc{TS}.
}

\subsection{IRS.V-Zero} \label{ss-irs-vzero}

\textsc{IRS.V-Zero} introduces a further complication \newedit{in} that its inner problem requires the DM to consider her causal process in the course of solving the inner problem.
Under the penalty $z_t^\textsc{Irs.V-Zero}$ given in \eqref{e-penalty-irs-vzero}, the DM at time $t$ earns $\E\left[ r_t( \mathbf{a}_{1:t}, \omega ) \left| \newedit{H_{t-1}( \mathbf{a}_{1:t-1}, \omega )} \right. \right]$, the expected mean reward that she can infer from observations prior to time $t$.
As we defined $R_{a,n}$ to be a reward from the $n^\text{th}$ pull on arm $a$ (not the pull at time $n$), the posterior belief associated with each arm is determined only by the number of past pulls performed on that arm.
Recall that $\hat{\mu}_{a,n}(\omega)$ is the expected mean reward of arm $a$ that the DM can infer from the first $n$ reward realizations:
\begin{equation}
        \hat{\mu}_{a,n}(\omega; y_a) \defeq \newedit{\E_{y_a}}\left[ \mu_a( \theta_a ) \left| R_{a,1}, \ldots, R_{a,n} \right. \right].
\end{equation}
Therefore, the DM earns $\hat{\mu}_{a,n-1}(\omega)$ from the $n^\text{th}$ pull on arm $a$, irrespective of the detailed sequence of the past actions.
More formally, the DM's earning at time $t$ is
\begin{equation}
        r_t( \mathbf{a}_{1:t}, \omega ) - z_t^\textsc{Irs.V-Zero}( \mathbf{a}_{1:t}, \omega )
                = \newedit{\E_\mathbf{y}}\left[ \mu_{a_t}( \theta_{a_t} ) \left| \newedit{H_{t-1}( \mathbf{a}_{1:t-1}, \omega )} \right. \right]
                = \hat{\mu}_{a_t,n_{t-1}(\mathbf{a}_{1:t-1},a_t)}(\omega),
\end{equation}
where $n_{t-1}(\mathbf{a}_{1:t-1}, a)$, defined in \eqref{e-count}, denotes the number of pulls conducted on a particular arm $a$ prior to time $t$.

Let $S_{a,n}(\omega) \defeq \sum_{i=1}^n \hat{\mu}_{a,i-1}(\omega)$ be the cumulative payoff from the first $n$ pulls of an arm $a$.
Given an outcome $\omega$, we observe that the total payoff is determined only by the total number of pulls \newedit{on} each arm, and not the sequence in which the arms have been pulled.
Therefore, solving the inner problem \eqref{e-inner-problem} is equivalent to ``finding the optimal allocation $(n_1^*, n_2^*, \ldots, n_K^*)$ among $T$ remaining opportunities'':  more formally,
\begin{equation} \label{e-inner-problem-irs-vzero}
        \max_{\mathbf{a}_{1:T} \in \Ascr^T}\left\{ \sum_{t=1}^T \hat{\mu}_{a_t,n_{t-1}(\mathbf{a}_{1:t-1},a_t)} \right\}
                = \max_{\mathbf{a}_{1:T} \in \Ascr^T}\left\{ \sum_{a=1}^K \sum_{n=1}^{n_T(\mathbf{a}_{1:T},a)} \hat{\mu}_{a,n-1} \right\}
                = \max_{\mathbf{n}_{1:K} \in N_T}\left\{ \sum_{a=1}^K S_{a,n_a} \right\},
\end{equation}
where $N_T \defeq \{ (n_1,\ldots,n_K) \in \N_0^K: \sum_{a=1}^K n_a = T \}$ is the set of all feasible allocations.
Once the $S_{a,n}$'s are computed, we can solve this inner problem within $O(KT^2)$ operations by sequentially applying \newedit{the} sup convolution $K$ times.
The detailed implementation is provided in \S \ref{ss-irs-vzero-impl}.

\newedit{The policy $\pi^\textsc{Irs.V-Zero}$ further needs to decide which arm to pull given the optimal allocation $(\tilde{n}_1^*, \tilde{n}_2^*, \ldots, \tilde{n}_K^*)$ that is obtained for the sampled outcome $\tilde{\omega}$.}
In principle, any arm $a$ that was included in the solution of the inner problem, $\tilde{n}_a^* > 0$, would suffice, but we suggest a selection rule by which the arm that needs the most pulls is chosen, i.e., $\newedit{A^\textsc{Irs.V-Zero}} = \argmax_a \tilde{n}_a^*$.
\newedit{This} guarantees that $\pi^\textsc{Irs.V-Zero}$ behaves like \textsc{TS} when $T$ is large, as formally stated in Proposition \ref{prop-asymptotic-behavior}.
\\

\noindent \textbf{Comparison with \textnormal{\textsc{TS}} and \textnormal{\textsc{Irs.FH}}}.
Recall that in the inner problems of \textsc{TS} and \textsc{Irs.FH}, the DM at time $t$ earns $\E[ r_t | \bm{\theta} ]$ and $\E[ r_t | \hat{\bm{\mu}}_{T-1} ]$, respectively, which are the mean reward estimates that rely on the information not available at the moment; e.g., $\hat{\mu}_{a,T-1}$ is revealed only after playing the arm $a$ for $T-1$ times.
\textsc{Irs.V-Zero} is more restrictive for the DM in the sense that she at time $t$ earns $\E[ r_t | \newedit{H_{t-1}} ]$, which does not include any information that does not belong to \newedit{$H_{t-1}$}.
\textsc{Irs.V-Zero} reflects the fact that the $n^\text{th}$ reward of an arm will not be revealed unless the arm is pulled $n$ times, and its inner problem requires the DM to allocate a pull in order to incorporate the next reward realization into her information set; thus learning about an arm comes at the cost of sacrificing an opportunity to learn about the other arms.

More specifically, let us focus on the total payoff of a particular allocation $(n_1,\ldots,n_K)$ under each penalty function $z_t^\textsc{Irs.V-Zero}$ and $z_t^\textsc{Irs.FH}$.
The allocation yields $\sum_{a=1}^K S_{a,n_a}(\omega)$ in the inner problem of \textsc{Irs.V-Zero} whereas the same allocation yields $\sum_{a=1}^K n_a \times \hat{\mu}_{a,T-1}(\omega)$ in the inner problem of \textsc{Irs.FH}.
In terms of variability originating from the randomness of $\omega$, we observe that each summand $S_{a,n_a}(\omega) = \sum_{i=1}^{n_a} \hat{\mu}_{a,i-1}(\omega)$ is less \newedit{noisy} than its counterpart $n_a \times \hat{\mu}_{a,T-1}(\omega)$ since \newedit{a larger number of observations makes it easier for the posterior to deviate from the initial prior and hence} the variance of individual terms $\hat{\mu}_{a,0}(\omega), \ldots, \hat{\mu}_{a,n_a-1}(\omega)$ is smaller than the variance of $\hat{\mu}_{a,T-1}(\omega)$ and, therefore, $\sum_{a=1}^K S_{a,n_a}(\omega)$ is smaller than $\sum_{a=1}^K n_a \times \hat{\mu}_{a,T-1}(\omega)$.
Analogous to the comparison between \textsc{Irs.FH} and \textsc{TS}, we have that \textsc{Irs.V-Zero} yields a performance bound $W^\textsc{Irs.V-Zero}$ that is tighter than $W^\textsc{Irs.FH}$ (formally stated in Theorem \ref{thm-monotonicity}) and a policy $\pi^\textsc{Irs.V-Zero}$ that performs fewer random explorations than $\pi^\textsc{Irs.FH}$.

\subsection{IRS.V-EMax} \label{ss-irs-vemax}
Under perfect information relaxation, the DM perfectly knows not only (i) what she will earn at future times but also (ii) how her belief will evolve as a result of her action sequence.
The previous algorithms focus on the former component by making the DM adjust the future rewards by conditioning (e.g.,  $\E[ r_t(a_t) | \bm{\theta} ]$, $\E[ r_t(a_t) | \hat{\bm{\mu}}_{T-1} ]$ and $\E[ r_t(a_t) | \newedit{H_{t-1}} ]$).
\textsc{Irs.V-EMax} also focuses on the \newedit{latter} component as well by charging \newedit{the DM} an additional cost for using the information on her future belief transitions.

To motivate this in detail, recall that the ideal penalty $z_t^\text{ideal}$ \eqref{e-ideal-penalty} is
\begin{align}
        z_t^\textup{ideal}( \mathbf{a}_{1:t}, \omega ) &\defeq r_t( \mathbf{a}_{1:t}, \omega ) -  \E\left[ r_t( \mathbf{a}_{1:t}, \omega )\left| H_{t-1}( \mathbf{a}_{1:t-1}, \omega ) \right. \right]
                \\ &  \quad + V^*\left( T-t, \mathbf{y}_t( \mathbf{a}_{1:t}, \omega ) \right) - \E\left[ \left. V^*\left( T-t, \mathbf{y}_t( \mathbf{a}_{1:t}, \omega ) \right)   \right| H_{t-1}( \mathbf{a}_{1:t-1}, \omega ) \right], \nonumber
\end{align}
where $V^*\left( T-t, \mathbf{y}_t \right)$ measures the value of having a belief $\mathbf{y}_t$ at a future time $t+1$.
Note that, at the moment the DM takes an action $a_t$, the next belief state $\mathbf{y}_t = \Uscr( \mathbf{y}_{t-1}, a_t, r_t )$ is not measurable with respect to the natural filtration \newedit{$\sigma(H_{t-1})$} since the next observation $r_t$ is unknown.
In DP terms, the conditional expectation $\E\left[ \left. V^*\left( T-t, \mathbf{y}_t \right) \right| \newedit{H_{t-1}} \right]$ captures the expected value of \newedit{a} (random) next state given the current state.
Accordingly, the gap between its realized value and its expected value, $V^*\left( T-t, \mathbf{y}_t \right) - \E\left[ \left. V^*\left( T-t, \mathbf{y}_t \right) \right| \newedit{H_{t-1}} \right]$, measures the additional gain from knowing the next belief state $\mathbf{y}_t$.
In addition to the term $r_t - \E\left[ r_t \left| \newedit{H_{t-1}} \right. \right] $ ($=z_t^\textsc{Irs.V-Zero}$), which measures the benefit from knowing which action will \newedit{yield} a large immediate reward, the ideal penalty also penalizes the long-term benefit from knowing which action will lead to a favorable belief state.

The penalty function $z_t^\textsc{Irs.V-EMax}$ is obtained from $z_t^\textup{ideal}$ by replacing $V^*(T, \mathbf{y})$ with $W^\textsc{TS}(T, \mathbf{y})$, which is rather tractable.
The use of $W^\textsc{TS}(T, \mathbf{y}) \defeq \E_{\mathbf{y}}\left[ T \times \max_a \mu_a(\theta_a) \right]$, introduced in \eqref{e-full-info}, leads to a simple expression for its conditional expectation: since $\bm{\theta}|H_{t-1}$ is distributed with $\Pscr(\mathbf{y}_{t-1})$, we have
\begin{align} \label{e-expected-full-info}
        \E_{\mathbf{y}}\left[ \left. W^\textsc{TS}\left( T-t, \mathbf{y}_t \right) \right| H_{t-1} \right]
                &= (T-t) \times  \E_{\mathbf{y}}\left[ \left. \max_a \mu_a(\theta_a) \right| H_{t-1} \right]
                \\&= (T-t) \times  \E_{\mathbf{y}_{t-1}}\left[ \max_a \mu_a(\theta_a)  \right]
                \\&= W^\textsc{TS}\left( T-t, \mathbf{y}_{t-1} \right).
\end{align}
In the associated inner problem, the payoff that the DM earns at time $t$ is
\begin{align}
        & r_t(\mathbf{a}_{1:t}, \omega) - z_t^\textsc{Irs.V-EMax}(\mathbf{a}_{1:t}, \omega)
        \\&= \hat{\mu}_{a_t,n_{t-1}(\mathbf{a}_{1:t-1},a_t)}(\omega) - W^\textsc{TS}\left( T-t, \mathbf{y}_t(\mathbf{a}_{1:t}, \omega) \right) + W^\textsc{TS}\left( T-t, \mathbf{y}_{t-1}(\mathbf{a}_{1:t-1}, \omega) \right)
        \\&= \bar{\mu}_{a_t}( [\mathbf{y}_{t-1}( \mathbf{a}_{1:t-1}, \omega )]_{a_t} ) - W^\textsc{TS}\left( T-t, \mathbf{y}_t(\mathbf{a}_{1:t}, \omega) \right) + W^\textsc{TS}\left( T-t, \mathbf{y}_{t-1}(\mathbf{a}_{1:t-1}, \omega) \right), \label{e-payoff-irs-vemax}
\end{align}
which is completely determined by the prior belief $\mathbf{y}_{t-1}$ and the posterior belief $\mathbf{y}_t$.

We further observe that, given $\omega$, the future belief $\mathbf{y}_t(\mathbf{a}_{1:t}, \omega)$ depends only on how many times each arm has been pulled, irrespective of the sequence of the pulls.
For example, consider two action sequences $\mathbf{a}_{1:t}^A = (1,1,2,1,2)$ and $\mathbf{a}_{1:t}^B = (2,1,1,2,1)$.
Even though the order of observations would differ, in both cases the agent would observe $(R_{1,1}, R_{1,2}, R_{1,3})$ from arm 1 and $(R_{2,1}, R_{2,2})$ from arm 2 and end up with the same belief $\mathbf{y}_t(\mathbf{a}_{1:t}^A, \omega) = \mathbf{y}_t(\mathbf{a}_{1:t}^B, \omega)$.
We may conclude from this observation that a belief state can be sufficiently parameterized with the pull counts $\mathbf{n}_{1:K}=(n_1,\ldots,n_K)$ instead of action sequence $\mathbf{a}_{1:t}$ \deledit{that is, with $\mathbf{y}_t(\mathbf{n}_{1:K})$ instead of $\mathbf{y}_t(\mathbf{a}_{1:t})$}.

\newedit{As a result, the total number of possible future beliefs is $O(T^K)$, not $O(K^T)$, and we can come up with a dynamic programming algorithm that solves the inner problem within $O(c_W T^K + KT^K)$ computations where $c_W$ is the cost of numerically calculating $W^\textsc{TS}(T, \mathbf{y})$.
We refer the interested reader to \S \ref{ss-irs-vemax-impl}.}

\deledit{Given the above observations, we can solve the inner problem within $O(KT^K)$ computations by dynamic programming.
While deferring the detailed description of procedure to \S \ref{ss-irs-vemax-impl}, we here briefly highlight the main idea of dynamic programming.}

\deledit{Let us consider a subproblem of \eqref{e-inner-problem} given a pull allocation $\mathbf{n}_{1:K}$, in which we are constrained to play each arm $n_1,\ldots,n_K$ times and we are looking for the best sequence of pulls $\mathbf{a}_{1:t}$ that maximizes the total payoff with $t = \sum_a n_a$.
The maximal value of this subproblem can be computed from the result of other $K$ subproblems parameterized with $\mathbf{n}_{1:K} - \mathbf{e}_1, \mathbf{n}_{1:K} - \mathbf{e}_2, \ldots, \mathbf{n}_{1:K} - \mathbf{e}_K$, where $\mathbf{e}_a$ is a basis vector whose $a^\text{th}$ component is one;
i.e., having decided to play an arm $a$ at time $t$, the previous belief state should be $\mathbf{y}_{t-1}(\mathbf{n}_{1:K}-\mathbf{e}_a)$ and we can earn at most the maximal value of the subproblem with $\mathbf{n}_{1:K}-\mathbf{e}_a$ plus the payoff of transition from $\mathbf{y}_{t-1}(\mathbf{n}_{1:K}-\mathbf{e}_a)$ to $\mathbf{y}_t(\mathbf{n}_{1:K})$, as represented in \eqref{e-payoff-irs-vemax}.}

\deledit{Each subproblem can be solved in $O(K)$ computations if the previous subproblems and the payoffs are pre-calculated.
Note that the total number of possible future beliefs is $O(T^K)$, not $O(K^T)$.
Therefore, the inner problem \eqref{e-inner-problem} can be solved by sequentially solving $O(T^K)$ subproblems, which will require $O(c_W T^K + KT^K)$ operations, where $c_W$ is the cost of numerically calculating $W^\textsc{TS}(T, \mathbf{y})$.}

\subsection{IRS.Index Policy} \label{ss-irs-index}

Finally, we propose \newedit{the} \textsc{Irs.Index} policy, which does not strictly belong to the IRS framework, and does not produce a performance bound, but does exhibit strong empirical performance.

Roughly speaking, \newedit{the} \textsc{Irs.Index} is \newedit{a single-sample approximation} of the finite-horizon Gittins index \citep{Kaufmann12}, where \newedit{the approximation is motivated by \textsc{Irs.V-EMax} algorithm}.
It first solves the single-armed bandit problem for each arm in isolation, and makes a decision based on the results of these subproblems.
\\

\noindent \textbf{Single-armed bandit problem.}
Consider a special case of an MAB instance in which there is a single arm $a$ that yields stochastic rewards $R_{a,n} \sim \Rscr_a(\theta_a)$ with an outside option that yields a deterministic reward $\lambda$.
We have a prior distribution $\Pscr_a(y_a)$ over unknown parameter $\theta_a$ whereas the deterministic reward $\lambda$ is known a priori.

Given an outcome $\omega_a = \left( \theta_a, (R_{a,n})_{n \in \N} \right)$, we can simulate the future belief trajectory $( y_{a,n} )_{n \in \{0,\ldots,T\}}$, where $y_{a,n}$ is the belief after $n$ reward realizations are observed:
\begin{equation}
        y_{a,0} \defeq y_a, \quad y_{a,n} \defeq \Uscr_a( y_{a,n-1}, R_{a,n} ), \quad \forall n = 1, \ldots, T.
\end{equation}
\newedit{Let $V^*(T, y_a, \lambda)$ be the optimal value function associated with this single-armed bandit problem.}
We consider the penalty function $z_t^\textsc{Irs.V-EMax}$ in which the value function $V^*(T, y_a, \lambda)$ is approximated by $W^\textsc{TS}(T, y_a, \lambda) = \E_{y_a}\left[ T \times \max( \mu_a(\theta_a), \lambda ) \right]$.
We define $\Ascr \defeq \{0,1\}$ such that $a_t = 1$ if the stochastic arm at time $t$ is selected, and $a_t=0$ if the outside option is selected.
The associated inner problem is
\begin{align} \label{e-irs-index-1}
        \text{maximize} \quad & \sum_{t=1}^T \hat{\mu}_{a,n_t-1}(\omega_a) \cdot \mathbf{1}\{ a_t = 1 \} + \lambda \cdot \mathbf{1}\{a_t = 0 \} - (T-t) \times \left(\Gamma^\lambda_{n_t}(\omega_a) - \Gamma^\lambda_{n_{t-1}}(\omega_a) \right)
        \\
        \text{subject to} \quad & n_t = \sum_{s=1}^t \mathbf{1}\{a_t=1\}, \quad a_t \in \{0,1\}, \quad \forall t=1,\ldots,T,
\end{align}
where $\hat{\mu}_{a,n}(\omega_a) \defeq \E_{y_a}[ \mu_a(\theta_a) | R_{a,1},\ldots,R_{a,n} ] = \bar{\mu}_a(y_{a,n})$ and
\begin{equation}
        \Gamma^\lambda_n(\omega_a) \defeq \E_{y_{a,n}}\left[ \max( \mu_a(\theta_a), \lambda ) \right].
\end{equation}
With some algebra (Proposition \ref{prop-irs-index-reformulation} in \S \ref{ss-irs-index-impl}), we can reformulate the optimization problem as
\begin{equation} \label{e-irs-index-2}
        \max_{0 \leq n \leq T}\left\{ T \times \Gamma^\lambda_0(\omega_a) + (T-n) \times \left( \lambda - \min_{0 \leq i \leq n} \Gamma^\lambda_i(\omega_a) \right)  + \sum_{i=1}^n \left( \hat{\mu}_{a,i-1}(\omega_a) - \Gamma^\lambda_{i-1}(\omega_a) \right) \right\},
\end{equation}
where the decision variable $n$ is the total number of pulls on the stochastic arm.

Let $\varphi_a(\lambda, \omega_a)$ be the (maximal) relative benefit from pulling the stochastic arm against not pulling at all:
\begin{equation} \label{e-irs-index-worth-trying}
        \varphi_a(\lambda, \omega_a) \defeq \max_{1 \leq n \leq T} \left\{ T \times \Gamma^\lambda_0 + (T-n) \times \left( \lambda - \min_{0 \leq i \leq n} \Gamma^\lambda_i \right) + \sum_{i=1}^n \left( \hat{\mu}_{a,i-1} - \Gamma^\lambda_{i-1} \right) \right\} - T \times \lambda.
\end{equation}
Note that $\max\{ \cdot \}$ was taken over $n \geq 1$.
We interpret the meaning of the sign of $\varphi_a(\lambda, \omega_a)$ as follows: given an outcome $\omega_a$, the stochastic arm is worth trying against the deterministic outside option $\lambda$ if $\varphi_a(\lambda, \omega_a) \geq 0$, and not worth trying if $\varphi_a(\lambda, \omega_a) < 0$.

Given $\omega_a$ and $\lambda$, the value of $\varphi_a(\lambda, \omega_a)$ can be computed in $O(T)$ operations by precalculating $\sum_{i=1}^n \hat{\mu}_{a,i-1}(\omega_a)$, $\min_{0 \leq i \leq n} \Gamma^\lambda_i(\omega_a)$, and $\sum_{i=1}^n \Gamma^\lambda_{i-1}(\omega_a)$ over $n=1,\ldots,T$ sequentially.
The single-armed bandit problem has an additional advantage of computational efficiency: in contrast to the implementation of \textsc{Irs.V-EMax} in the multi-arm setting, the approximate value function (captured by $\Gamma^\lambda_n$) often admits a closed-form expression in the single-armed setting.
In the cases of the Beta-Bernoulli MAB and the Gaussian MAB, for example, we have
\begin{align}
        \E_{\mu \sim \text{Beta}(\alpha, \beta)}\left[ \max\left( \mu, \lambda \right) \right]
                &= \lambda \times F^{\text{beta}}_{\alpha, \beta}\left( \lambda \right) + \frac{\alpha}{\alpha+\beta} \times \left( 1 - F^{\text{beta}}_{\alpha+1, \beta}\left( \lambda \right) \right),
        \\
        \E_{\mu \sim \Nscr(m, \nu^2)}\left[ \max\left( \mu, \lambda \right) \right]
                &= m + (\lambda - m) \times \Phi\left( \nu^{-1} (\lambda - m) \right) + \nu \times \phi\left( \nu^{-1}(\lambda - m) \right),
\end{align}
where $F^\text{beta}_{\alpha,\beta}( \cdot )$ represents the c.d.f. of $\text{Beta}(\alpha, \beta)$ distribution, and $\Phi(\cdot)$ and $\phi(\cdot)$ represent the c.d.f. and the p.d.f. of the standard normal distribution, respectively.
With these expressions, $\Gamma_n^\lambda(\omega_a)$'s can be computed very efficiently without using numerical integration or Monte Carlo sampling.
\\

\noindent \textbf{Index policy.}
We now return to the original MAB problem with $K$ arms.
Recall that the single-armed bandit algorithm tells us whether an arm (given an outcome $\omega_a$) is worth trying against the deterministic reward $\lambda$.
We use this algorithm as a module to compute the index of each arm.

More specifically, consider a certain decision epoch when the remaining time is $T$ and the belief is $\mathbf{y}$.
For each arm $a=1,\ldots,K$ separately, the policy $\pi^\textsc{Irs.Index}$ samples the future outcome $\tilde{\omega}_a$ (i.e., draws $\tilde{\theta}_a \sim \Pscr_a(y_a)$ and $\tilde{R}_{a,n} \sim \Rscr_a(\tilde{\theta}_a)$ for $n = 1, \ldots, T$), and finds a threshold value on the deterministic outside option that makes the arm barely worth trying:
\begin{equation}
        \lambda_a^*(\tilde{\omega}_a) \defeq \sup\left\{ \lambda \in \R ~; ~ \varphi_a(\lambda, \tilde{\omega}_a) \geq 0 \right\}.
\end{equation}
By the definition of $\varphi_a(\lambda, \omega_a)$, the threshold value $\lambda_a^*(\tilde{\omega}_a)$ measures the value of arm $a$ as an opportunity cost of not pulling arm $a$ at all, given a particular outcome $\tilde{\omega}_a$.
We use the value $\lambda_a^*(\tilde{\omega}_a)$ as an index of arm $a$ so that the index policy plays the arm with the largest index, \newedit{i.e.,} $\newedit{A^\textsc{Irs.Index}} = \argmax_a \lambda_a^*(\tilde{\omega}_a)$.

Although the monotonicity of the mapping $\lambda \mapsto \varphi_a(\lambda, \tilde{\omega}_a)$ is not theoretically proven, we observe that the bisection search works sufficiently well in our numerical experiments.
Since each instance of single-armed bandit problems requires $O(T)$ computations to solve, the entire procedure for \newedit{arm selection} requires a run time of $O( c_b \times KT )$, where $c_b$ represents the number of iterations in a bisection search.
See \S \ref{ss-irs-index-impl} for the implementation details.

In addition to the \textsc{Irs.Index} policy described above, some numerical experiments include a heuristic variation of it, \newedit{called} \textsc{Irs.Index*}, that is obtained by using
\begin{equation}
        \varphi_a(\lambda, \omega_a) \defeq \max_{1 \leq n \leq T} \left\{ \sum_{i=1}^n \left( \hat{\mu}_{a,i-1}(\omega_a) - \lambda - \left( \Gamma_i^\lambda(\omega_a) - \Gamma_0^\lambda(\omega_a) \right) \right) \right\},
\end{equation}
instead of \eqref{e-irs-index-worth-trying}.
This alternative formulation yields indices that are relatively stable across the different samples of outcome $\tilde{\omega}_a$.

We note that our index, $\lambda_a^*(\tilde{\omega}_a)$, is a random approximation of the finite-horizon Gittins (FH-Gittins) index studied in \citet{Kaufmann12}, \citet{NinoMora11}, and \citet{Lattimore16}.
The original FH-Gittins algorithm precisely solves the single-armed bandit problem, which is shown to be an optimal stopping problem in which one must decide when to stop pulling the stochastic arm as one's belief state evolves stochastically.
Applying the information relaxation framework to the single-armed bandit problem, we solve, instead, a simple deterministic problem in which one must find a deterministic schedule optimized to a particular belief trajectory associated with a randomly generated outcome $\tilde{\omega}$.
As in the previous algorithms, the penalties help us to obtain a solution close to the optimal stopping policy of the original single-armed bandit problem.


%% file: analysis.tex
In this section, we provide theoretical analyses that characterize IRS policies and performance bounds in particular for \textsc{TS}, \textsc{Irs.FH}, and \textsc{Irs.V-Zero}.

\begin{remark}[Single-period optimality] \label{prop-last-decision} When $T=1$, all of the policies
  $\pi^\textsc{Irs.FH}$, $\pi^\textsc{Irs.V-Zero}$, $\pi^\textsc{Irs.V-Emax}$, and
  $\pi^\textsc{Irs.Index}$ take the optimal action; i.e., they pull the myopically best arm
  $\newedit{A^*}=\argmax_a \bar{\mu}_a(y_a)$.
\end{remark}

\begin{proposition}[Asymptotic behavior] \label{prop-asymptotic-behavior}
Assume that $\mu_i(\theta_i) \ne \mu_j(\theta_j)$ almost surely for any two distinct arms $i \ne j$. As
$T \nearrow \infty$, the distribution of the $\pi^\textsc{Irs.FH}$'s action converges to that of Thompson sampling:
\begin{equation}
	\lim_{T \rightarrow \infty} \PR\left[ \newedit{A^\textsc{Irs.FH}(T, \mathbf{y})} = a \right] = \PR\left[ \newedit{A^\textsc{TS}(\mathbf{y})} = a \right]
	, \quad \forall a \in \Ascr.
\end{equation}
Similarly, so does the distribution of the $\pi^\textsc{Irs.V-Zero}$'s action:\footnote{We assume a particular selection rule such that $\tilde{a}^\textsc{Irs.V-Zero}=\argmax_a \tilde{n}_a^*$, as discussed in \S \ref{ss-irs-vzero}.}
\begin{equation}
	\lim_{T \rightarrow \infty} \PR\left[ \newedit{A^\textsc{Irs.V-Zero}(T, \mathbf{y})} = a \right] = \PR\left[ \newedit{A^\textsc{TS}(\mathbf{y})} = a \right]
	, \quad \forall a \in \Ascr.
\end{equation}
\end{proposition}
\newedit{$A^\textsc{TS}(\mathbf{y})$, $A^\textsc{Irs.FH}(T, \mathbf{y})$ and $A^\textsc{Irs.V-Zero}(T, \mathbf{y})$} denote the action taken by policies $\pi^\textsc{TS}$, $\pi^\textsc{Irs.FH}$, and $\pi^\textsc{Irs.V-Zero}$, respectively, when the remaining time is $T$ and the current belief is $\mathbf{y}$.
These actions are random variables, since each of these policies uses a randomly sampled outcome $\tilde{\omega}$ of its own.
Remark \ref{prop-last-decision} can be easily verified by observing that, when $T=1$, $r_1(a, \omega) - z_1(a, \omega; T, \mathbf{y}) = \bar{\mu}_a(y_a)$ for any $a \in \Ascr$ for each of the penalty functions.
\newedit{The results in Proposition \ref{prop-asymptotic-behavior} follow from Remark \ref{rem:mean-reward-martingale} stating that the posterior predictive mean reward process converges to the conditional mean reward, i.e., $\lim_{n \rightarrow \infty} \hat{\mu}_{a,n}(\tilde{\omega}) = \mu_a(\tilde{\theta}_a)$}.
The assumption $\mu_i(\theta_i) \ne \mu_j(\theta_j)$ is made to avoid the ambiguity of the tie-breaking rule that is used in \textsc{TS}.

Remark \ref{prop-last-decision} and Proposition \ref{prop-asymptotic-behavior} illustrate that $\pi^\textsc{Irs.FH}$ and $\pi^\textsc{Irs.V-Zero}$ behave like TS during the initial decision epochs, gradually shift toward the myopic scheme, and end up with the optimal decision; by contrast, \textsc{TS} continues to explore.
The transition from exploration to exploitation under these IRS policies occurs smoothly, without relying on an auxiliary control parameter.
While maintaining their recursive structure, IRS policies take into account the time horizon $T$, and naturally balance exploitation and exploration.

\begin{theorem}[Monotonicity of performance bounds] \label{thm-monotonicity}
\textsc{Irs.FH} and \textsc{Irs.V-Zero} monotonically improve the performance bound
\begin{equation} \label{e-monotonicity-1}
	W^\textsc{TS}( T, \mathbf{y} ) \geq W^\textsc{Irs.FH}( T, \mathbf{y} ) \geq W^\textsc{Irs.V-Zero}(T, \mathbf{y}),
\end{equation}
and also
\begin{equation} \label{e-monotonicity-2}
	W^\textsc{TS}(T, \mathbf{y}) \geq W^\textsc{Irs.V-EMax}(T, \mathbf{y}).
\end{equation}
Recall that $W^\textsc{TS}(T, \mathbf{y}) = \mathbb{E}_{\mathbf{y}}\left[ T \times \max_a \mu_a(\theta_a) \right]$ is the conventional regret benchmark.
\end{theorem}
Empirically (\S \ref{s-numerical}), we observe that $W^\textsc{Irs.V-Zero} \geq W^\textsc{Irs.V-EMax}$.
In addition, we have $W^\textsc{Irs.V-EMax} \geq W^\textup{ideal}$ since $W^\textup{ideal}$ is the lowest attainable upper bound (Theorem \ref{thm-strong-duality}).
\newedit{The second inequality \eqref{e-monotonicity-2} holds in a stronger sense: for every outcome $\omega$, the maximal value of the inner problem associated with $W^\textsc{TS}$ is greater than that of the inner problem associated with $W^\textsc{Irs.V-EMax}$.}

While the entire proof is provided in \S \ref{prf-monotonicity}, we highlight here the main ideas.
The first result \eqref{e-monotonicity-1} follows from the monotonicity of the information structure incorporated in each penalty function:
\textsc{TS}, \textsc{Irs.FH}, and \textsc{Irs.V-Zero} \newedit{replace the realized rewards with} $\E(r_t|\bm{\theta})$, $\E(r_t | \hat{\bm{\mu}}_{T-1})$, and $\E(r_t|\newedit{H_{t-1}})$\deledit{ at time $t$}, respectively, \newedit{where} $\bm{\theta}$ is more informative than $\hat{\bm{\mu}}_{T-1}$, and $\hat{\bm{\mu}}_{T-1}$ is more informative than \newedit{$H_{t-1}$} \newedit{for the DM} to infer the value of future reward $r_t$.
Based on this observation, we use a variant of Jensen's inequality to prove the results.\footnote{\label{foot-monotonicity}
	We remark that $W^\textsc{Irs.FH} \geq W^\textsc{Irs.V-Zero}$ is not an immediate consequence of the fact that $\sigma(\hat{\bm{\mu}}_{T-1})$ is a stronger filtration than $\sigma(H_{t-1})$.
	It \newedit{further relies on a particular structure} of MAB problems: \newedit{the rewards of an arm are independent and identically distributed conditionally on the parameter}.
	See \S \ref{prf-monotonicity2} for a further discussion.
}
The second result \eqref{e-monotonicity-2} is proven based on Theorem 4 of \cite{Desai12}, which says that if an approximate value function $\widehat{V}$ is a supersolution (Definition \ref{def-supersolution}) to the Bellman equation and a penalty function $\hat{z}$ approximates the ideal penalty with $\widehat{V}$ in place of $V^*$, the resulting performance bound $W^{\hat{z}}$ is smaller than $\widehat{V}$.
By showing that $W^\textsc{TS}$ is a supersolution to \eqref{e-bellman}, we prove that $W^\textsc{Irs.V-EMax} \leq W^\textsc{TS}$ since $z_t^\textsc{Irs.V-EMax}$ is constructed upon $W^\textsc{TS}$.


Although Theorem \ref{thm-monotonicity} compares the performance bound among IRS algorithms, we interpret that its tightness, $W^z-V^*$, reflects the degree of optimism that its corresponding policy $\pi^z$ possesses.
Recall that $W^z$ is the expected value of the best possible payoff when the DM is informed of some future outcomes in advance.
The weak duality $W^z \geq V^*$ implies that IRS policies are basically optimistic: an IRS policy takes an action as if it can earn more than the optimal policy in the belief that the sampled outcome is the ground truth.
In this sense, the gap $W^z - V^*$ captures how optimistically the policy $\pi^z$ interprets the sampled outcome.
When $W^z - V^*$ is relatively small for a certain penalty function $z_t$, we may conclude that the penalty function $z_t$ makes the DM less optimistic and induces a policy $\pi^z$ that performs fewer random explorations.

We further compare the performance of IRS policies using an alternative suboptimality measure.
We define the ``suboptimality gap'' of an IRS policy $\pi^z$ to be $W^z(T, \mathbf{y}) - V(\pi^z, T, \mathbf{y})$, and analyze it instead of the conventional (Bayesian) regret, $W^\textsc{TS}(T, \mathbf{y}) - V(\pi^z, T, \mathbf{y})$.
While its non-negativity is guaranteed by weak duality (Theorem \ref{thm-weak-duality}), more desirably, the optimal policy yields a zero suboptimality gap (Theorem \ref{thm-strong-duality} and Remark \ref{rem-ideal-penalty-optimality}).
This measure coincides with the conventional regret measure only for TS.

\newedit{
  \begin{theorem}[Suboptimality gap for natural exponential families] \label{thm-suboptimality}
    Consider an MAB instance with a reward distribution that is a natural exponential family
    distribution, as described in \S \ref{ss-exponential-family}, in which each arm $a \in \Ascr$ is
    described with a log-partition function $A_a(\theta_a)$ and a hyperparameter
    $y_a = (\xi_a, \nu_a)$.  Suppose that all the log-partition functions are $L$-smooth, i.e.,
\begin{equation}
	\frac{d^2}{d\theta_a^2} A_a(\theta_a) \leq L, \quad \forall \theta_a \in \Theta_a, ~ a \in \Ascr.
\end{equation}
Further assume that $\nu_a = \nu$ for all $a \in \Ascr$.
Then, for any $T \geq 2$, we have
\begin{align}
	W^\textsc{TS}(T, \mathbf{y}) - V( \pi^\textsc{TS}, T, \mathbf{y} ) &\leq 2 \sqrt{L}\left[ \frac{1}{\sqrt{\nu}} + \sqrt{ 2 \log T } \times \left( \frac{K}{\sqrt{\nu}} + 2 \sqrt{KT} \right) \right],
	\label{e-suboptimality-ts}
	\\
	W^\textsc{Irs.FH}(T, \mathbf{y}) - V( \pi^\textsc{Irs.FH}, T, \mathbf{y} ) &\leq 2 \sqrt{L}\left[ \frac{1}{\sqrt{\nu}} + \sqrt{ 2 \log T } \times \left( \frac{K}{\sqrt{\nu}} + 2 \sqrt{KT} - \frac{1}{3} \sqrt{ \frac{T}{K} } \right) \right],
	\label{e-suboptimality-irs-fh}
	\\
	W^\textsc{Irs.V-Zero}(T, \mathbf{y}) - V( \pi^\textsc{Irs.V-Zero}, T, \mathbf{y} ) &\leq \sqrt{L}\left[ \frac{1}{\sqrt{\nu}} + \sqrt{ 2 \log T } \times \left( \frac{K}{\sqrt{\nu}} + 2 \sqrt{KT} - \frac{1}{3} \sqrt{ \frac{T}{K} } \right) \right].
	\label{e-suboptimality-irs-vzero}
\end{align}
\end{theorem}
\begin{remark} \label{rem:bernoulli-mab}
	For a Bernoulli MAB with symmetric arms, each of which has a prior $\text{Beta}(\alpha, \beta)$ for its mean reward, we have $L = \frac{1}{2}$ and $\sqrt{\nu} = \sqrt{\alpha+\beta}$.
\end{remark}
\begin{remark} \label{rem:gaussian-mab}
	For a Gaussian MAB with symmetric arms, each of which has a prior $\Nscr(m, v^2)$ for its mean reward and a noise variance $\sigma^2$, we have $L = \sigma$ and $\sqrt{\nu} = v/\sigma$.
\end{remark}
}
Theorem \ref{thm-suboptimality} indirectly shows to the improvements to the suboptimality gaps:
although all the bounds have the same asymptotic order of $O(\sqrt{KT \log T})$, the IRS policies improve the leading coefficient or the additional term.\footnote{\newedit{
	Recall that $W^\textsc{TS} - V(\pi^\textsc{TS})$ represents the Bayesian regret of \textsc{TS}.
	It will be worth mentioning some known results that may be comparable to the bound \eqref{e-suboptimality-ts} established in Theorem \ref{thm-suboptimality}.
	
	For the cases where the reward distributions have a bounded support in $[0,1]$,
        \citet{Bubeck13} have shown that the Bayesian regret of \textsc{TS} is bounded from above
        by $14 \sqrt{KT}$; and further shown that its asymptotic order is unimprovable in the
        sense that for any policy there exists a prior distribution such that the policy
        experiences Bayesian regret no smaller than $\frac{1}{20} \sqrt{KT}$. However, this does not imply that the regret of the Bayesian optimal policy is bounded from below by $\frac{1}{20}\sqrt{KT}$ in the context of Theorem~\ref{thm-suboptimality}, since we consider a specific prior and the policy optimized to that prior.
	
	
	For Gaussian MAB in the non-Bayesian setting, \citet{Agrawal13} have shown that the regret of \textsc{TS} is $O( \sqrt{KT \log T} )$; and further shown that its asymptotic order is unimprovable in the sense that for any policy there exists an instance (i.e., the set of true mean values) such that the policy's regret is at least $\Omega(\sqrt{KT \log K})$.
	
	While there is no result in the literature that is comparable to the other bounds \eqref{e-suboptimality-irs-fh} and \eqref{e-suboptimality-irs-vzero}, we conjecture that they will be tight just as the bound for \textsc{TS} \eqref{e-suboptimality-ts} is, given the fact all three policies exhibit the identical asymptotic behavior for large $T$ (Proposition~\ref{prop-asymptotic-behavior}).
}}
\newedit{These results hold for a wide range of MAB problems including the Bernoulli MAB and the Gaussian MAB as stated in Remarks \ref{rem:bernoulli-mab} and \ref{rem:gaussian-mab}.
}

The proof of Theorem \ref{thm-suboptimality}, provided in \S \ref{prf-suboptimality}, relies on an essential property of IRS policies that generalizes the ``probability matching'' property of \textsc{TS}, i.e., a matching between nature's randomness and \newedit{the} decision maker's randomness.
It is well known that \textsc{TS} is randomized in a way that, conditional on past observations, the probability that an action $a$ is chosen equals the probability that the action $a$ is chosen by someone who knows the parameters.
Analogously, the IRS policy $\pi^z$ is randomized in a way that, conditional on past observations, the probability that an action $a$ is chosen equals the probability that the action $a$ is chosen by someone who knows the entire future but is penalized (Proposition \ref{prop-generalized-posterior-sampling}).
Recall that the penalties are designed to penalize the benefit from having additional future information.
A better choice of penalty function would prevent the policy $\pi^z$ from picking an action that is overly optimized for a randomly sampled future realization, which in turn would improve the quality of the decision making.

Given the above observation, our proof \newedit{utilizes} the approach taken by \citet{Russo14} \newedit{that exploits} the probability matching property of \textsc{TS} to bound its Bayesian regret.
More specifically, for each penalty function, we carefully construct a sequence of confidence intervals on the mean reward such that the corresponding policy's instantaneous suboptimality at each time (loss against the hindsight solution) is bounded by the width of the confidence interval \deledit{with high probability}\newedit{approximately}.
\newedit{For} a better penalty function, the confidence intervals can be made tighter so that the total suboptimality can also be bounded more effectively.
\newedit{In our analysis, the natural exponential family is assumed in order to analyze the concentration of posterior distribution in a closed form, and the smoothness condition on the log-partition function is assumed in order to guarantee that the reward distribution is sub-Gaussian, whereas \citet{Russo14} consider an arbitrary reward distribution with a bounded support.}


%% file: numerical.tex
\subsection{Experimental Setup}

We conduct numerical simulations to evaluate the effectiveness of our framework in comparison to alternative algorithms.
In addition to the IRS algorithms discussed so far, we consider other recently developed algorithms that are particularly suitable for a Bayesian setting: the Bayesian upper confidence bound \citep{Kaufmann12} (\textsc{Bayes-UCB}, with a quantile of $1 - \frac{1}{t}$), information-directed sampling \citep{Russo17} (\textsc{IDS}), the optimistic Gittins index \citep{Farias18} (\textsc{OGI}, one-step look-ahead approximation with a discount factor of $\gamma_t = 1 - \frac{1}{t}$), \newedit{and the Lagrangian index policies suggested in \citet{Brown20} (\textsc{Lagr-RT} and \textsc{Lagr-OT}, with a random and an optimal tie-breaking rule, respectively)}.

Our numerical experiments include Beta-Bernoulli MABs and Gaussian MABs.
Given an MAB problem instance specified by the prior distribution $\Pscr( \mathbf{y} )$ and the reward distribution $\Rscr$, we simulate the policies and calculate the IRS bounds with respect to the different values of time horizon $T$.

Let $S$ be the number of simulations we perform.
For each $s \in \{1, \ldots, S\}$, we first sample the parameters $\theta_a^{(s)} \sim \Pscr_a(y_a)$ and the rewards $R_{a,n}^{(s)} \sim \Rscr_a(\theta_a^{(s)})$ for all $n \in \{1, \ldots, T_\text{max}\}$ and $a \in \Ascr$, which is equivalent to sampling an outcome $\omega^{(s)} \sim \Iscr(\mathbf{y})$.
Given the $s^\text{th}$ sampled outcome $\omega^{(s)}$, for each time horizon $T \in \{ 5,10,15,\ldots,T_{\max}\}$, we simulate each policy $\pi$ (that may utilize the time horizon $T$);
i.e., at each time $t=1,\ldots,T$, the policy makes a decision\footnote{\newedit{
	Recall that IRS policies are randomized policies that perform their own simulations at each time along the sample path. 
	This posterior sampling procedure is independent of the random generation of true outcomes.
}} \newedit{on} which arm to pull, $A_t^\pi$, and then the associated reward, $r_t( \mathbf{A}_{1:t}^\pi, \omega^{(s)} ) = R_{A_t^\pi, n_t(\mathbf{A}_{1:t}^\pi, A_t^\pi)}^{(s)}$, is revealed accordingly.
After simulating one sample path, $\sum_{t=1}^T \mu_{A_t^\pi}( \theta_{A_t^\pi}^{(s)} )$ is recorded as the performance of $\pi$ for the $s^\text{th}$ sample, and the expected performance $V(\pi, T, \mathbf{y})$ is measured by its sample average across $S$ samples for each $T$.

In order to compute IRS bounds, we use the same set of samples $\omega^{(1)}, \ldots, \omega^{(S)}$.
For each penalty function $z$ and for each $T \in \{ 5,10,\ldots,T_{\max}\}$, we solve the associated inner problems with respect to $\omega^{(1)}, \ldots, \omega^{(S)}$, and the IRS bound $W^z(T, \mathbf{y})$ is evaluated by taking the average of the maximal values over $S$ instances.

More explicitly, we use the following sample averages to calculate $V(\pi, T, \mathbf{y})$ and $W^z(T, \mathbf{y})$:
\begin{equation}
        V(\pi, T, \mathbf{y}) \approx \frac{1}{S} \sum_{s=1}^S \left( \sum_{t=1}^T \mu_{A_t^\pi}( \theta_{A_t^\pi}^{(s)} ) \right)
        , \quad
        W^z(T, \mathbf{y}) \approx \frac{1}{S} \sum_{s=1}^S \max_{\mathbf{a}_{1:T} \in \Ascr^T} \left\{ \sum_{t=1}^T r_t( \mathbf{a}_{1:t}, \omega^{(s)} ) - z_t( \mathbf{a}_{1:t}, \omega^{(s)} ) \right\}.
\end{equation}
Note again that the same outcome $\omega^{(s)}$ is used across the different values of time horizon $T$ and across different algorithms.
Sharing the randomness enhances the consistency of the estimates.
In what follows, we use 20,000 samples (i.e., $S=20,000$).

Based on $V(\pi, T, \mathbf{y})$ and $W^\textsc{TS}(T, \mathbf{y})$ measured with the sample averages, we calculate the Bayesian regret of a policy $\pi$:
\begin{align}
                \text{BayesRegret}( \pi, T, \mathbf{y} )
                       & \defeq \E\left[ \sum_{t=1}^T \max_a \mu_a(\theta_a) - \mu_{A_t^\pi}(\theta_{A_t^\pi}) \right]
                        = W^\textsc{TS}(T, \mathbf{y}) - V(\pi, T, \mathbf{y}),
\end{align}
which is a conventional measure in performance analysis of Bayesian algorithms as discussed in \S \ref{ss-ts-revisited}.
We further calculate the regret (lower) bound obtained from a IRS penalty function $z_t$:
\begin{align}
                \text{RegretBound}( z, T, \mathbf{y} )
                		&\defeq W^\textsc{TS}(T, \mathbf{y}) - W^z(T, \mathbf{y}).
\end{align}
By the weak duality (Theorem \ref{thm-weak-duality}), we have $\text{BayesRegret}( \pi, T, \mathbf{y} ) \geq \text{RegretBound}( z, T, \mathbf{y} )$ for any $\pi \in \Pi_\mathbb{F}$.
By its definition, the regret bound produced by \textsc{TS} is zero.

\subsection{Results}

\noindent \textbf{Bernoulli MAB with two arms ($K=2$).}
We first provide the results for a Bernoulli MAB in which
\begin{equation}
        \mu_a \sim \text{Beta}(1,1), \quad R_{a,n} \sim \text{Bernoulli}(\mu_a), \quad \forall a \in \{1, 2\}.
\end{equation}
We consider relatively short time horizons ($\leq T_{\max}=200$) since we are focusing on a finite-horizon regime rather than an asymptotic regime.
In this particular case, since the state (belief) space is discrete and small in size, $O(T^4)$, we are able to solve the Bellman equations \eqref{e-bellman} numerically, and thus we can implement the Bayesian optimal policy, which is labeled as \textsc{Opt} in what follows.

Figure \ref{fig-bern-2arms} shows the regrets (solid lines) of all the policies discussed above
and the regret bounds (dashed lines) produced by the IRS algorithms.\footnote{\newedit{
There also exists a performance bound induced by the Lagrangian index policies.
We omit it from Figure \ref{fig-bern-2arms}, however, since that bound is not so tight and thus not informative to be displayed in the same plot; e.g., when $T=200$, the associated regret bound is $-12.54$, which is far below the current x-axis.
}}  Table~\ref{tbl-bern-2arms} provides further details including the percentage improvement in regret
  over \textsc{TS}, i.e.,
  \[
    \text{RegretImprovement}(\pi, T, \mathbf{y}) \defeq
    1-\frac{\text{BayesRegret}( \pi, T, \mathbf{y} )}
    {\text{BayesRegret}( \textsc{TS}, T, \mathbf{y})},
  \]
  and the improvement in regret bound over \text{TS} benchmarked to the regret
  of the best performing algorithm, i.e.,
  \[
    \text{BoundImprovement}(\pi, T, \mathbf{y}) \defeq
    \frac{ \text{RegretBound}( z, T, \mathbf{y} ) -\text{RegretBound}( z^\textsc{TS}, T, \mathbf{y}
      ) }{\min_{\pi'} \text{BayesRegret}( \pi', T, \mathbf{y}) }.
  \]
In Figure~\ref{fig-bern-2arms}, note that lower regret curves are better, and higher bound
curves are better.

Comparing the IRS algorithms (\textsc{TS}, \textsc{Irs.FH}, \textsc{Irs.V-Zero}, \textsc{Irs.V-Emax}, and \textsc{Opt}), we first observe a clear improvement in both the performance of policies and the tightness of bounds, as we adopt a more complicated penalty function, \newedit{albeit one that} requires a longer run time:
as visualized in Figure \ref{fig-bern-2arms}, the regret curve approaches the \textsc{Opt} curve from above and the bound curve approaches it from below, where the \textsc{Opt} curve represents the lowest attainable regret that is the highest attainable regret bound at the same time.
The suboptimality gap (the gap between a regret curve and its corresponding bound curve) becomes smaller, which is consistent with the implication of Theorem \ref{thm-suboptimality}.

Finally, we note that the \textsc{Irs.Index} policy is outperforming all the other policies; i.e., the regret curve of \textsc{Irs.Index} is surprisingly close to the \textsc{Opt} curve.
Although it is developed based on \textsc{Irs.V-EMax}, it performs better than \textsc{Irs.V-EMax}, and the reasons for that still need to be researched.

\begin{figure}[H]
        \centering
        \includegraphics[width=0.9\linewidth]{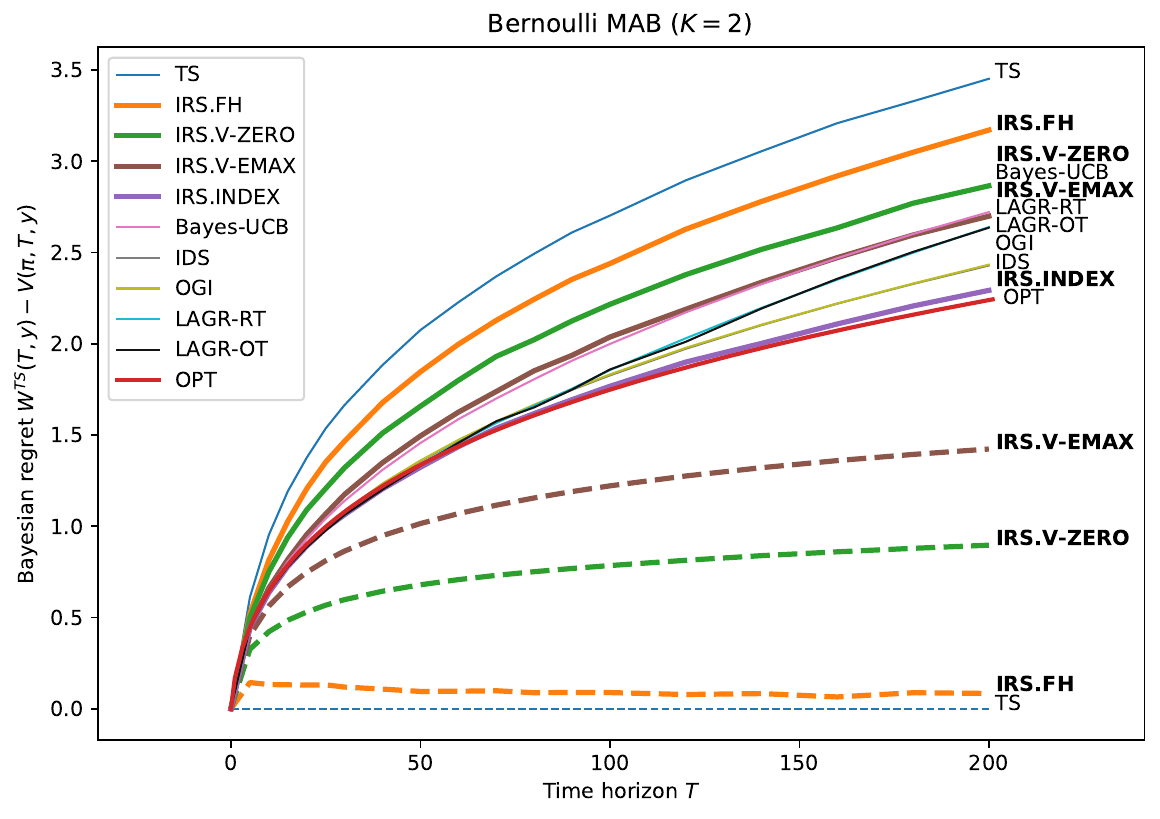}
        \caption{Regret plot for a Bernoulli MAB with two arms.
        		The solid lines represent the (Bayesian) regret of algorithms, $W^\textsc{TS}(T, \mathbf{y}) - V(\pi, T, \mathbf{y})$, and the dashed lines represent the regret bounds that IRS algorithms produce, $W^\textsc{TS}(T, \mathbf{y}) - W^z(T, \mathbf{y})$.
		\newedit{Each data point reports the average across 20,000 simulations.}}
        \label{fig-bern-2arms}
\end{figure}

\begin{table}[H]
\centering
{\small
\input{tabular_bern_2arms_all}
}
\caption{ \newedit{Simulation results for} a Bernoulli MAB with two arms when $T=200$.  The best
  results are emphasized with bold letters.  The third and fifth columns show the percentage
  improvements over \textsc{TS} in regret and in bound respectively; e.g., \textsc{Irs.V-EMax}
  achieves a regret that is 21.8\% better than that of \textsc{TS}, and yields a regret bound that
  accounts for 63.7\% of the lowest regret observed
  empirically. \\
  The last column shows the average time required \newedit{for a policy to make decisions along one sample path including the time required posterior sampling for the case of IRS policies.
  ${}^*$\textsc{Lagr-RT} and \textsc{Lagr-OT} require substantial
    offline computation prior to simulation. This takes around 20 hours in the setting of this
    simulation.}}
\label{tbl-bern-2arms}
\end{table}

\noindent \textbf{Bernoulli MAB with ten arms ($K=10$).}  We next consider a Bernoulli MAB with
ten arms and $T_{\max} = 500$.  \newedit{\textsc{Irs.V-EMax} and \textsc{Opt} are omitted from
  this simulation due to their computational cost, and so are \textsc{Lagr-RT} and \textsc{Lagr-OT}
  for long horizons\footnote{\label{foot:lagrangian-index}\newedit{ \textsc{Lagr-RT} and
      \textsc{Lagr-OT} require substantial offline pre-computation. This involves a convex
      optimization problem with $T$ decision variables, where a single evaluation of
      the objective function requires $\Theta(T^3)$ operations.  As recommended by \citet{Brown20}, we
      have implemented a cutting-plane method using a commercial optimization software (Gurobi),
      but it takes over a week to complete the pre-computation when $T=350$.  }} ($T > 350$).}
Figure \ref{fig-bern-10arms} and Table \ref{tbl-bern-10arms} show the simulation results.  We
again observe a monotonic improvement in the performance of policies and the tightness of bounds
among IRS algorithms, and the \textsc{Irs.Index} policy still performs best.

\begin{figure}[H]
        \centering
        \includegraphics[width=0.9\linewidth]{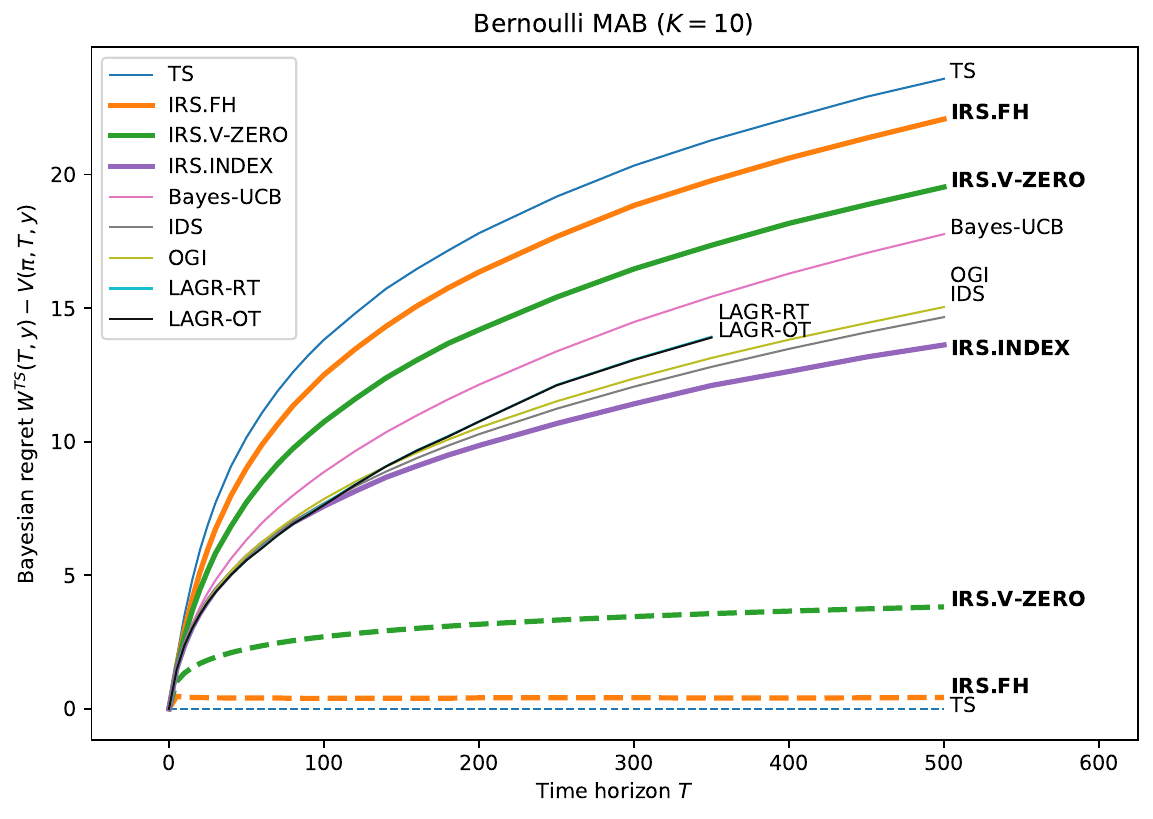}
        \caption{Regret plot for a Bernoulli MAB with ten arms.
        		\newedit{\textsc{Lagr-RT} and \textsc{Lagr-OT} are simulated only for $T \leq 350$ due to the computational cost (see Footnote \ref{foot:lagrangian-index}).}}
        \label{fig-bern-10arms}
\end{figure}

\begin{table}[H]
\centering
  {%
        \small
        \input{tabular_bern_10arms_all}
  }
  \caption{\newedit{Simulation results} for a Bernoulli MAB with ten arms when $T=500$.}
  \label{tbl-bern-10arms}
\end{table}

\newpage
\noindent \textbf{Gaussian MABs ($K=2$ or $10$).}
We next consider Gaussian MABs in which
\begin{equation}
	\mu_a \sim \Nscr(0,1^2), \quad R_{a,n} \sim \Nscr(\mu_a, 1^2), \quad \forall a \in \{1,\ldots,K\}.
\end{equation}
Figure \ref{fig-gauss-2arms} and Table \ref{tbl-gauss-2arms} show the case of two arms ($K=2$), and Figure \ref{fig-gauss-10arms} and Table \ref{tbl-gauss-10arms} show the case of ten arms ($K=10$).
\newedit{The algorithms \textsc{Lagr-RT} and \textsc{Lagr-OT} are not implemented for Gaussian
  MABs since they require either discrete belief states or some form of state discretization.}
The results are similar to those of Bernoulli MABs.

\begin{figure}[H]
        \centering
        \includegraphics[width=0.9\linewidth]{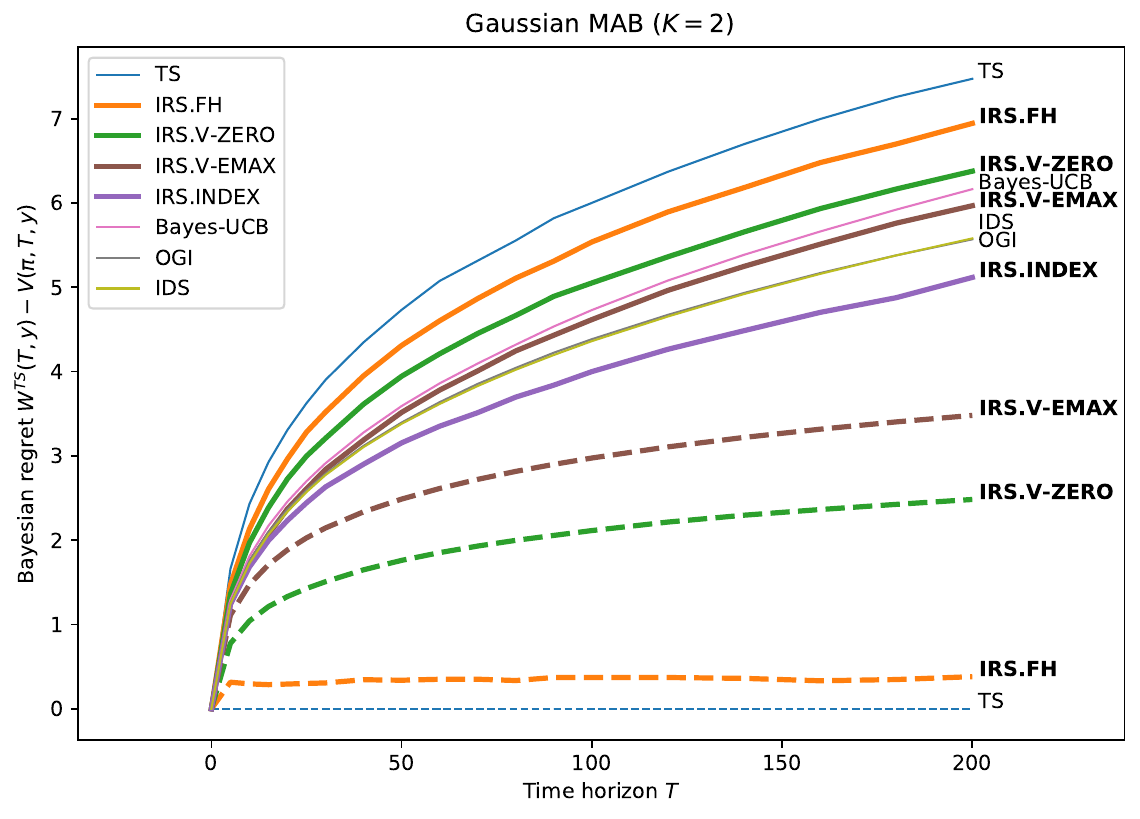}
        \caption{Regret plot for a Gaussian MAB with two arms.}
        \label{fig-gauss-2arms}
\end{figure}

\begin{table}[H]
\centering
  {%
        \small
        \input{tabular_gauss_2arms_all}
  }
  \caption{Simulation results for a Gaussian MAB with two arms when $T=200$.}
  \label{tbl-gauss-2arms}
\end{table}

\begin{figure}[!h]
        \centering
        \includegraphics[width=0.9\linewidth]{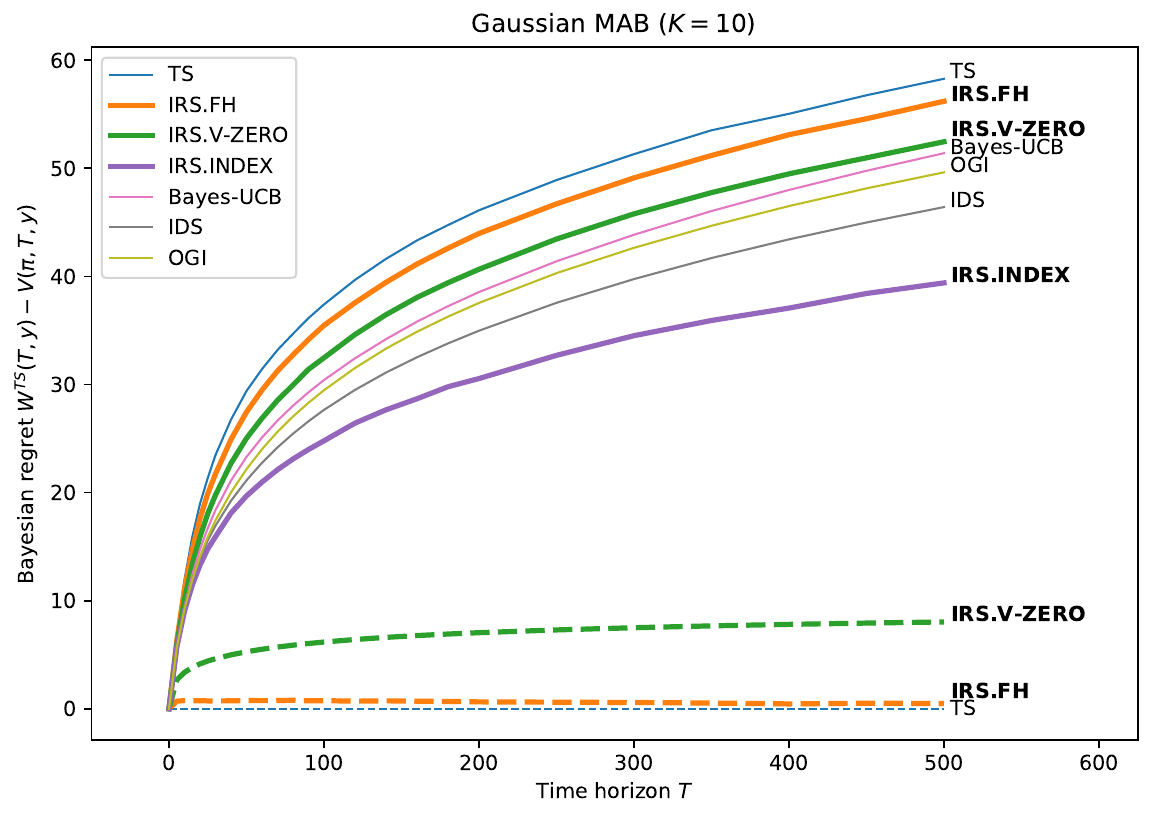}
        \caption{Regret plot for a Gaussian MAB with ten arms.}
        \label{fig-gauss-10arms}
\end{figure}

\begin{table}[!htbp]
\centering
  {%
        \small
        \input{tabular_gauss_10arms_all}
  }
  \caption{Simultion results for a Gaussian MAB with ten arms when $T=500$.}
  \label{tbl-gauss-10arms}
\end{table}

\noindent \textbf{Gaussian MAB with different noise variances ($K=5$).}
We next consider a problem where
\begin{equation}
        \mu_a \sim \Nscr(0,1^2), \quad R_{a,n} \sim \Nscr(\mu_a, \sigma_a^2), \quad \forall a \in \{1,\ldots,5\}
\end{equation}
and $(\sigma_1, \sigma_2, \sigma_3, \sigma_4, \sigma_5) = (0.1,~ 0.4,~ 1,~ 4,~ 10)$.
In this MAB instance, it is particularly crucial for the algorithms to consider how much the DM can learn about each of the arms during the remaining time periods, since the difficulty of estimating the mean reward of an arm $a$ heavily depends on the noise level $\sigma_a$ that varies across the arms.\footnote{
        In order for the posterior distribution to be concentrated so as to have a standard deviation of $0.1$, for example, one observation is enough for arm 1 whereas 100 and 10,000 observations are required for arm 3 and arm 5, respectively.
}

As shown in Figure \ref{fig-gauss-5arms-asym}, \textsc{Bayes-UCB} shows a particularly poor performance, as it keeps pulling arm 5 without considering the fact that arm 5 is too noisy to be learnt within such a short period of time (i.e., $T \leq 500$).
By contrast, we observe that our IRS policies and \textsc{IDS} algorithm outperform \newedit{the} \textsc{Bayes-UCB}, \textsc{OGI}, and \textsc{TS} algorithms, since they explicitly take into account the value of exploration by quantifying the informativeness of a new observation for each arm (more specifically, by considering how the belief will change as a new reward realization is revealed).
Notably, the \textsc{Irs.FH} policy, which is a very simple modification of \textsc{TS}, significantly improves \newedit{the performance of} \textsc{TS} without degrading its computational efficiency.

The example also illustrates the significance of having a tighter performance bound.
If the benchmark is set to $W^\textsc{Irs.V-Zero}$, when $T=500$, the \textsc{Irs.Index*} policy\footnote{
        The \textsc{Irs.Index*} policy is a heuristic modification of the \textsc{Irs.Index} policy. See \S \ref{ss-irs-index-impl}.
} achieves 94\% $\left(= \frac{V(\pi^\textsc{Irs.Index*},T,\mathbf{y})}{ W^\textsc{Irs.V-Zero}(T, \mathbf{y}) }\right)$ of the benchmark.
If the benchmark is set to $W^\textsc{TS}$ instead, as in a conventional regret analysis, we might have concluded that the \textsc{Irs.Index*} policy achieves only 88\% $\left(= \frac{V(\pi^\textsc{Irs.Index*},T,\mathbf{y})}{ W^\textsc{TS}(T, \mathbf{y}) }\right)$ of that (looser) bound, which would suggest a larger margin of possible improvement.

\begin{figure}[H]
        \centering
        \includegraphics[width=0.9\linewidth]{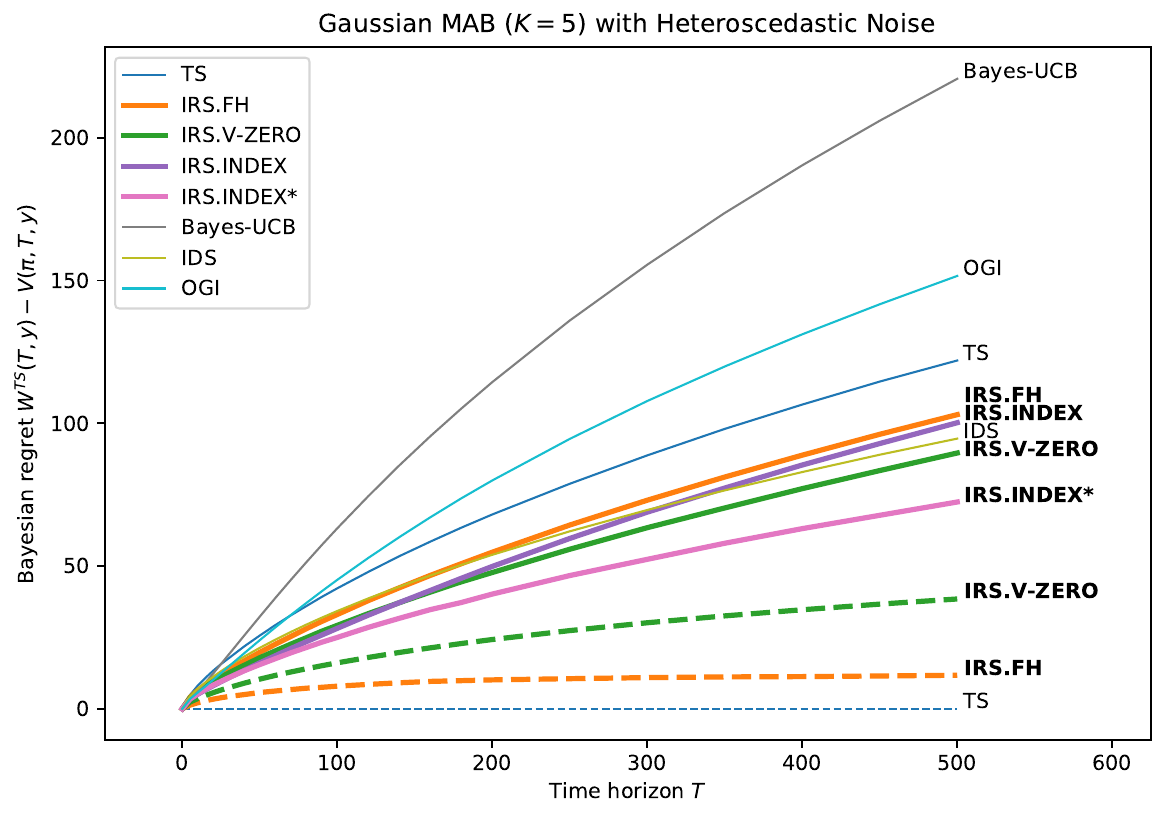}
        \caption{Regret plot for a Gaussian MAB with five arms with different noise variances.}
        \label{fig-gauss-5arms-asym}
\end{figure}

\begin{table}[H]
\centering
{\small
\input{tabular_gauss_5arms_asym}
}
\caption{Simulation results for a Gaussian MAB with five arms with different noise variances when $T=500$.}
\label{tbl-gauss-5arms-asym}
\end{table}


%% file: tabular_bern_2arms_all.tex
\begin{tabular}{ c c c c c c }
 \toprule
\thead{Algorithm} & \thead{Bayesian regret \\ (s.e.)} & \thead{Regret \\ improvement} & \thead{Regret lower \\ bound (s.e.)} & \thead{Bound \\ improvement} & \thead{Policy\\run time} \\
 \midrule
 \textsc{TS} & 3.45 (0.021) & 0.0\% & 0.00 (--) & 0.0\% &  17 ms  \\ 
 \textsc{Irs.FH} & 3.17 (0.020) & 8.1\% & 0.08 (0.040) & 3.8\% &  37 ms  \\ 
 \textsc{Irs.V-Zero} & 2.87 (0.021) & 17.0\% & 0.90 (0.055) & 40.0\% &  527 ms  \\ 
 \textsc{Irs.V-EMax} & 2.70 (0.020) & 21.8\% & 1.42 (0.326) & 63.6\% &  29.5 sec  \\ 
 \textsc{Irs.Index} & 2.29 (0.023) & 33.6\% & -- & -- &  3.6 sec  \\ 
 \textsc{Bayes-UCB} & 2.72 (0.020) & 21.2\% & -- & -- &  44 ms  \\ 
 \textsc{IDS} & 2.43 (0.028) & 29.6\% & -- & -- &  3.7 sec  \\ 
 \textsc{OGI} & 2.43 (0.028) & 29.5\% & -- & -- &  262 ms  \\ 
 \newedit{\textsc{Lagr-RT}} & \newedit{2.64 (0.046)} & \newedit{23.5\%} & \newedit{-12.54} & \newedit{--} &  \newedit{19 ms*}  \\ 
 \newedit{\textsc{Lagr-OT}} & \newedit{2.64 (0.046)} & \newedit{23.6\%} & \newedit{-12.54} & \newedit{--} &  \newedit{14 ms*}  \\ 
 \textsc{Opt} & \textbf{2.24} (--) & 35.1\% & \textbf{2.24} (--) & 100.0\% &  --  \\ 

 \bottomrule
\end{tabular}

%% file: tabular_bern_10arms_all.tex
\begin{tabular}{ c c c c c c }
 \toprule
\thead{Algorithm} & \thead{Bayesian regret \\ (s.e.)} & \thead{Regret \\ improvement} & \thead{Regret lower \\ bound (s.e.)} & \thead{Bound \\ improvement} & \thead{Policy\\run time} \\
 \midrule
 \textsc{TS} & 23.59 (0.078) & 0.0\% & 0.00 (--) & 0.0\% &  50 ms  \\ 
 \textsc{Irs.FH} & 22.08 (0.076) & 6.4\% & 0.43 (0.042) & 4.0\% &  300 ms  \\ 
 \textsc{Irs.V-Zero} & 19.54 (0.074) & 17.2\% & \textbf{3.82} (0.058) & 35.6\% &  17.0 sec  \\ 
 \textsc{Irs.Index} & \textbf{13.62} (0.080) & 42.2\% & -- & -- &  56.2 sec  \\ 
 \textsc{Bayes-UCB} & 17.77 (0.077) & 24.7\% & -- & -- &  140 ms  \\ 
 \textsc{IDS} & 14.67 (0.093) & 37.8\% & -- & -- &  16.4 sec  \\ 
 \textsc{OGI} & 15.04 (0.092) & 36.2\% & -- & -- &  2.6 sec  \\ 

 \bottomrule
\end{tabular}

%% file: tabular_gauss_2arms_all.tex
\begin{tabular}{ c c c c c c }
 \toprule
\thead{Algorithm} & \thead{Bayesian regret \\ (s.e.)} & \thead{Regret \\ improvement} & \thead{Regret lower \\ bound (s.e.)} & \thead{Bound \\ improvement} & \thead{Policy\\run time} \\
 \midrule
 \textsc{TS} & 7.47 (0.047) & 0.0\% & 0.00 (--) & 0.0\% &  17 ms  \\ 
 \textsc{Irs.FH} & 6.94 (0.045) & 7.1\% & 0.38 (0.100) & 7.4\% &  37 ms  \\ 
 \textsc{Irs.V-Zero} & 6.38 (0.048) & 14.7\% & 2.48 (0.133) & 48.5\% &  625 ms  \\ 
 \textsc{Irs.V-EMax} & 5.97 (0.044) & 20.2\% & \textbf{3.48} (1.154) & 68.0\% &  13.3 sec  \\ 
 \textsc{Irs.Index} & \textbf{5.12} (0.054) & 31.5\% & -- & -- &  2.2 sec  \\ 
 \textsc{Bayes-UCB} & 6.16 (0.045) & 17.5\% & -- & -- &  38 ms  \\ 
 \textsc{IDS} & 5.58 (0.068) & 25.3\% & -- & -- &  679 ms  \\ 
 \textsc{OGI} & 5.57 (0.067) & 25.5\% & -- & -- &  196 ms  \\ 

 \bottomrule
\end{tabular}

%% file: tabular_gauss_10arms_all.tex
\begin{tabular}{ c c c c c c }
 \toprule
\thead{Algorithm} & \thead{Bayesian regret \\ (s.e.)} & \thead{Regret \\ improvement} & \thead{Regret lower \\ bound (s.e.)} & \thead{Bound \\ improvement} & \thead{Policy\\run time} \\
 \midrule
 \textsc{TS} & 58.28 (0.180) & 0.0\% & 0.00 (--) & 0.0\% &  35 ms  \\ 
 \textsc{Irs.FH} & 56.20 (0.180) & 3.6\% & 0.48 (0.156) & 1.2\% &  215 ms  \\ 
 \textsc{Irs.V-Zero} & 52.46 (0.188) & 10.0\% & \textbf{8.04} (0.216) & 20.4\% &  13.7 sec  \\ 
 \textsc{Irs.Index} & \textbf{39.40} (0.244) & 32.4\% & -- & -- &  30.4 sec  \\ 
 \textsc{Bayes-UCB} & 51.40 (0.178) & 11.8\% & -- & -- &  77 ms  \\ 
 \textsc{IDS} & 46.41 (0.324) & 20.4\% & -- & -- &  4.0 sec  \\ 
 \textsc{OGI} & 49.63 (0.335) & 14.8\% & -- & -- &  1.6 sec  \\ 

 \bottomrule
\end{tabular}

%% file: tabular_gauss_5arms_asym.tex
\begin{tabular}{ c c c c c c }
 \toprule
\thead{Algorithm} & \thead{Bayesian regret \\ (s.e.)} & \thead{Regret \\ improvement} & \thead{Regret lower \\ bound (s.e.)} & \thead{Bound \\ improvement} & \thead{Policy\\run time} \\
 \midrule
 \textsc{TS} & 121.99 (0.615) & 0.0\% & 0.00 (--) & 0.0\% &  34 ms  \\ 
 \textsc{Irs.FH} & 103.03 (0.628) & 15.5\% & 11.75 (0.656) & 16.2\% &  128 ms  \\ 
 \textsc{Irs.V-Zero} & 89.59 (0.690) & 26.6\% & \textbf{38.47} (0.827) & 53.1\% &  7.4 sec  \\ 
 \textsc{Irs.Index} & 100.20 (0.657) & 17.9\% & -- & -- &  12.8 sec  \\ 
 \textsc{Irs.Index*} & \textbf{72.43} (0.866) & 40.6\% & -- & -- &  12.3 sec  \\ 
 \textsc{Bayes-UCB} & 220.66 (1.285) & -80.9\% & -- & -- &  88 ms  \\ 
 \textsc{IDS} & 94.63 (0.817) & 22.4\% & -- & -- &  2.9 sec  \\ 
 \textsc{OGI} & 151.61 (1.030) & -24.3\% & -- & -- &  829 ms  \\ 

 \bottomrule
\end{tabular}

%% file: extensions.tex
Below, we describe several natural generalizations of the methods developed in this paper beyond
the setting of Section~\ref{s-problem}:


\noindent \textbf{MAB with unknown time horizon.}
This paper studies finite-time horizon MABs for which we suggest algorithms that exploit the knowledge of the time horizon $T$ and we focus on a relatively small $T$ such that the time horizon becomes an important ingredient in optimally balancing exploration and exploitation.
We briefly illustrate how to relax our framework's dependency on $T$, i.e., \newedit{how to extend to} the setting with an unknown horizon and the setting with an indefinitely long horizon.

First, our framework (penalties, policies, and upper bounds) can naturally incorporate the unknown $T$ within the \textit{Bayesian setting}; i.e., the horizon $T$ is also a random variable whose prior distribution is known.
As a simple case, if $T$ is independent of the DM's actions, we can reformulate the objective function of the inner problem as  $\sum_{t=1}^\infty \gamma_t \left( r_t(\mathbf{a}_{1:t}, \omega) - z_t(\mathbf{a}_{1:t}, \omega) \right)$ where the discount factor $\gamma_t \defeq \PR[ T \geq t ]$ is the survivor probability, and $r_t(\cdot)$ and $z_t(\cdot)$ are the reward and penalty terms used in the paper.
Alternatively, we can treat the random variable $T$ like the random reward realizations \newedit{by sampling} $T$ from its prior distribution while a penalty function (additionally) penalizes for the gain from knowing $T$ (one can imagine that the outcome $\omega$ now includes the realization of $T$ \newedit{and} not only the future reward realizations).
Structural results such as weak duality and strong duality will continue to hold.

Second, we can consider \newedit{a} practical modification of IRS policies when $T$ is large or infinite.
We can construct a dual feasible penalty function that mixes \textsc{Irs.FH} and \textsc{Irs.V-Zero},\footnote{
        In its inner problem, \textsc{Irs.V-Zero}-like penalties are applied for the initial $\lfloor T_0 / K \rfloor$ pulls and then \textsc{Irs.FH}-like penalties are applied for the later pulls.
} which induces an algorithm whose complexity is $O\left( K \min\{ T, T_0 \}^2 \right)$ for some predefined constant $T_0$.
Alternatively, we can convert \newedit{the} \textsc{Irs.V-EMax} or \textsc{Irs.Index} policy into an anytime policy by setting the inner problem's horizon large enough, despite that the performance bound will \deledit{be }no longer \newedit{be} obtainable.

\noindent \textbf{MAB in more complicated settings.}
Even though this paper develops a framework for the stochastic MAB with independent arms, which would be the simplest and oldest problem in \newedit{the} MAB literature, we believe that our framework applies \newedit{to} more complicated settings.
Consider the following examples:
\begin{itemize}
        \item A finite-horizon MAB with correlated arms (e.g., $R_{a,n} \sim \mathcal{N}(\mathbf{x}_a^\top \bm{\theta}, \sigma_a^2)$ where $\bm{\theta} \in \mathbb{R}^d$ is shared across the arms, and $\mathbf{x}_a \in \mathbb{R}^d$ is an arm's feature vector): \textsc{Irs.V-Zero} can be immediately implemented by adopting the DP algorithm discussed in \S \ref{ss-irs-vemax-impl}.

        \item MAB with the delayed reward realization: \textsc{Irs.FH} can be immediately implemented by simulating the DM's learning process in the presence of delay.

        \item MAB with a budget constraint (in which each arm consumes a certain amount of budget and the DM wants to maximize the total reward within a limited budget. See \citet{Ding13}): all IRS algorithms can be implemented by solving a budget-constrained optimization problem instead of a horizon-constrained optimization problem.
\end{itemize}
In these extensions, we obtain not only the online decision\newedit{-}making policies but also their performance bounds as in this paper.
Generally speaking, our framework provides a \newedit{systematic} way of improving \textsc{TS} by taking into account the exploitation-exploration trade-off more carefully, particularly in the presence of some constraint that \newedit{induces} incomplete learning; the main challenge would be to design a suitable penalty function that is tractable yet captures the problem-specific exploration-exploitation trade-off precisely.


%% file: conclusion.tex

\noindent \textbf{Contribution to MAB literature.}
We first highlight that our IRS framework generalizes Thompson sampling to the finite-horizon MAB setting.
As pointed out in \citet{Russo17b}, \textsc{TS} may perform poorly in time-sensitive learning problems in which exploitation is rather more encouraged than exploration. 
Interpreted as a special case of IRS policies, it is clear that \textsc{TS} is implicitly assuming an infinite time horizon in the sense that its associated inner problem solves a best-arm identification problem with an infinite number of observations.
As summarized in Table \ref{tbl-summary}, IRS algorithms consider more complicated inner problems in which the benefit from exploration is limited by the time-horizon constraint.
While maintaining the Bayesian recursive structure of its sequential decision-making process, we improve \textsc{TS} within a unified framework that also includes the Bayesian optimal policy as another special case.

Furthermore, the IRS framework provides a set of (Bayesian) performance bounds that are tighter than the conventional benchmark that has been widely used since \cite{Lai85}.
We believe that these benchmarks would be useful, in a Bayesian setting, in measuring the optimality of an algorithm or in assessing the intrinsic difficulty of an MAB problem instance.

\noindent \textbf{Contribution to information relaxation literature.}
The information relaxation framework is certainly a powerful tool to obtain performance bounds in a general class of decision-making problems.
Although there have been several studies \citep{Moallemi11} that elicit a decision-making policy based on this framework, they are limited to using a performance bound as a proxy for the value function.
Instead of approximating the value function explicitly, the IRS framework considers simulation-based randomized policies that make each decision that is optimized to a single instance of \newedit{a} simulated environment, and our results show that this scheme is very powerful in online learning problems where random exploration is required.

In applying the information relaxation framework to a particular application, the most challenging task is to find a suitable penalty function that is tractable yet yields a tight performance bound.
In this paper, by exploiting the recursive structures embedded in the Bayesian learning process, we derive a series of penalty functions so that users themselves can find a balance between the quality of policies/bounds and the computational cost.
We also provide theoretical analyses of the tightness of performance bounds and the suboptimality of associated policies by leveraging the existing analysis developed in the MAB literature.
These analytic results would be rare in the information relaxation literature due to the complex nature of the performance bound produced by the information relaxation framework. 


%% file: appendix.tex
\section{An Illustrative Example} \label{s-example}

Let us consider a Bernoulli MAB with \newedit{eight periods ($T=8$)} and three arms ($K=3)$ with the following priors:
\begin{equation}
	\mu_1 \sim \text{Beta}(3,1)
	, \quad
	\mu_2 \sim \text{Beta}(1,1)
	, \quad
	\mu_3 \sim \text{Beta}(1,3),
\end{equation}
where $R_{a,n} \sim \text{Bernoulli}(\mu_a)$ for each $a \in \{1,2,3\}$ and $n \in \{1,2,\cdots,8\}$.
Given this prior belief, the \newedit{predictive} mean reward of each arm is $\bar{\mu}_1 = \E_{\mu_1 \sim \text{Beta}(3,1)}[ \mu_1 ] = \frac{3}{4}$, $\bar{\mu}_2 = \frac{1}{2}$, and $\bar{\mu}_3 = \frac{1}{4}$, respectively.
As an illustrative example, we examine a particular instance where the true outcome $\omega$ is given as follows:
\begin{table}[htbp]
\centering
\begin{tabular}{c | c |  c  c  c  c  c  c  c  c }
\toprule
 & \multirow{2}{*}{ \thead{ True means $\mu_a(\theta_a)$ } } & \multicolumn{8}{c}{ \thead{ Rewards $R_{a,n}$} } \\
\cline{3-10}
 &  & $n=1$ & $2$ & $3$ & $4$ & $5$ & $6$ & $7$ & $8$ \\
\hline
\hline
Arm 1 ($a=1$) & $0.235$ & $0$ & $1$ & $1$ & $1$ & $0$ & $0$ & $0$ & $0$ \\
Arm 2 ($a=2$) & $0.443$ & $1$ & $0$ & $0$ & $1$ & $1$ & $1$ & $1$ & $0$ \\
Arm 3 ($a=3$) & $0.787$ & $1$ & $1$ & $1$ & $1$ & $0$ & $0$ & $1$ & $1$ \\
\bottomrule
\end{tabular}
\caption{An example of the outcome in a Bernoulli MAB with $K=3$ and $T=8$.} \label{tbl-outcome-example}
\end{table}

If we consider only the priors, arm~1 is best since $\bar{\mu}_1$ is largest among $(\bar{\mu}_1, \bar{\mu}_2, \bar{\mu}_3)$.
If, however, we have full information about the parameter values, arm~3 is best since $\mu_3$ is largest among $(\mu_1, \mu_2, \mu_3)$.

\subsection{Inner Problems Induced by Different Penalty Functions}

\noindent \textbf{No penalty.}
To clarify the role of penalties, we first consider the case of zero penalty, i.e., $z_t \equiv 0$, which was not discussed in \S \ref{s-framework}.
With zero penalty, the DM at any time earns the current realized reward without adjustment.
The clairvoyant DM, who is informed of the outcome $\omega$, can find the best action sequence for this particular outcome $\omega$.
Recall that $R_{a,n}$ is defined to be the reward from the $n^\text{th}$ pull of arm $a$, not the reward from arm $a$ at time $n$, and so the DM is not allowed to skip any of the reward realizations and the total reward does not depend on the order of pulls.
As depicted in the table below, the optimal solution is to pull arm 1 four times, arm 2 once, and arm 3 three times, which yields a total reward of $7$.

\begin{table}[htbp]
\centering
\begin{tabular}{c |  c  c  c  c  c  c  c  c  | c}
\toprule
 & \multicolumn{8}{c|}{ \thead{Payoffs under zero penalty} } & \multirow{2}{*}{ \thead{Maximal payoff} } \\
 \cline{2-9}
 & $n=1$ & 2 & 3 & 4 & 5 & 6 & 7 & 8 \\
\hline
\hline
Arm 1 & \cellcolor{black!25}$0$ & \cellcolor{black!25}$1$ & \cellcolor{black!25}$1$ & \cellcolor{black!25}$1$ & $0$ & $0$ & $0$ & $0$ & \multirow{3}{*}{$7$} \\
Arm 2 & \cellcolor{black!25}$1$ & $0$ & $0$ & $1$ & $1$ & $1$ & $1$ & $0$ \\
Arm 3 & \cellcolor{black!25}$1$ & \cellcolor{black!25}$1$ & \cellcolor{black!25}$1$ & $1$ & $0$ & $0$ & $1$ & $1$ \\
\bottomrule
\end{tabular}
\end{table}

\noindent \textbf{\textsc{TS} penalty.}
Next, let us examine the penalty $z_t^\textsc{TS}(\mathbf{a}_{1:t},\omega) \defeq r_t(\mathbf{a}_{1:t},\omega) - \mu_{a_t}(\theta_{a_t})$ under which the DM earns $\mu_a$ whenever playing arm $a$.
The hindsight optimal action sequence is to pull arm 3 (the arm with the largest mean reward $\mu_a$) eight times in a row and the DM can earn a total reward of $T \times \mu_3 = 6.296$ at most.

\begin{table}[htbp]
\centering
\begin{tabular}{c |  c  c  c  c  c  c  c  c  | c}
\toprule
 & \multicolumn{8}{c|}{ \thead{Payoffs under \textnormal{$z_t^\textsc{TS}$}} } & \multirow{2}{*}{ \thead{Maximal payoff} } \\
\cline{2-9} 
 & $n=1$ & 2 & 3 & 4 & 5 & 6 & 7 & 8 \\
\hline
\hline
Arm 1 & $.235$ & $.235$ & $.235$ & $.235$ & $.235$ & $.235$ & $.235$ & $.235$ & \multirow{3}{*}{$6.296$} \\
Arm 2 & $.443$ & $.443$ & $.443$ & $.443$ & $.443$ & $.443$ & $.443$ & $.443$ \\
Arm 3 & \cellcolor{black!25}$.787$ & \cellcolor{black!25}$.787$ & \cellcolor{black!25}$.787$ & \cellcolor{black!25}$.787$ & \cellcolor{black!25}$.787$ & \cellcolor{black!25}$.787$ & \cellcolor{black!25}$.787$ & \cellcolor{black!25}$.787$ \\
\bottomrule
\end{tabular}
\end{table}

\noindent \textbf{\textsc{IRS.FH} penalty.}
When the penalties are given by $z_t^\textsc{Irs.FH}(\mathbf{a}_{1:t},\omega) \defeq r_t(\mathbf{a}_{1:t},\omega) - \hat{\mu}_{a_t,T-1}(\omega)$, the DM earns $\hat{\mu}_{a,T-1}(\omega)$ whenever playing arm $a$.
Recall that $\hat{\mu}_{a,T-1}(\omega)$ is the Bayesian estimate on mean reward of arm $a$ after observing reward realizations $R_{a,1},\cdots,R_{a,T-1}$.
In this particular example, we have $\left(\hat{\mu}_{1,T-1}, \hat{\mu}_{2,T-1}, \hat{\mu}_{3,T-1}\right) = \left( \frac{6}{11}, \frac{6}{9}, \frac{6}{11} \right)$ and the maximal payoff is $T \times \hat{\mu}_{2,T-1} = 5.333$, which can be obtained by playing arm 2 throughout the entire time horizon.

\begin{table}[htbp]
\centering
\begin{tabular}{c |  c  c  c  c  c  c  c  c  | c}
\toprule
 & \multicolumn{8}{c|}{ \thead{Payoffs under \textnormal{$z_t^\textsc{Irs.FH}$}} } & \multirow{2}{*}{ \thead{Maximal payoff} } \\
\cline{2-9} 
 & $n=1$ & 2 & 3 & 4 & 5 & 6 & 7 & 8 \\
\hline
\hline
Arm 1 & $6/11$ & $6/11$ & $6/11$ & $6/11$ & $6/11$ & $6/11$ & $6/11$ & $6/11$ & \multirow{3}{*}{$5.333$} \\
Arm 2 & \cellcolor{black!25}$6/9$ & \cellcolor{black!25}$6/9$ & \cellcolor{black!25}$6/9$ & \cellcolor{black!25}$6/9$ & \cellcolor{black!25}$6/9$ & \cellcolor{black!25}$6/9$ & \cellcolor{black!25}$6/9$ & \cellcolor{black!25}$6/9$ \\
Arm 3 & $6/11$ & $6/11$ & $6/11$ & $6/11$ & $6/11$ & $6/11$ & $6/11$ & $6/11$ \\
\bottomrule
\end{tabular}
\end{table}

\noindent \textbf{\textsc{IRS.V-Zero} penalty.}
Finally, let us focus on $z_t^\textsc{Irs.V-Zero}(\mathbf{a}_{1:t}, \omega) \defeq r_t(\mathbf{a}_{1:t}, \omega) - \hat{\mu}_{a_t,n_{t-1}(\mathbf{a}_{1:t-1}, a_t)}$ under which the DM earns $\hat{\mu}_{a,n-1}(\omega)$ from the $n^\text{th}$ pull of arm $a$.
Since the payoff from an arm changes over time as the Bayesian estimate evolves, playing only one arm is no longer optimal, unlike in the previous two cases.
It can be easily verified that the optimal allocation is to play arm 1 six times and arm 2 two times, as visualized in the table below.

\begin{table}[htbp]
\centering
\begin{tabular}{c |  c  c  c  c  c  c  c  c  | c}
\toprule
 & \multicolumn{8}{c|}{ \thead{Payoffs under \textnormal{$z_t^\textsc{Irs.V-Zero}$}} } & \multirow{2}{*}{ \thead{Maximal payoff} }  \\
\cline{2-9} 
 & $n=1$ & 2 & 3 & 4 & 5 & 6 & 7 & 8 \\
\hline
\hline
Arm 1 & \cellcolor{black!25}$3/4$ & \cellcolor{black!25}$3/5$ & \cellcolor{black!25}$4/6$ & \cellcolor{black!25}$5/7$ & \cellcolor{black!25}$6/8$ & \cellcolor{black!25}$6/9$ & $6/10$ & $6/11$ & \multirow{3}{*}{$5.314$} \\
Arm 2 & \cellcolor{black!25}$1/2$ & \cellcolor{black!25}$2/3$ & $2/4$ & $2/5$ & $3/6$ & $4/7$ & $5/8$ & $6/9$ \\
Arm 3 & $1/4$ & $2/5$ & $3/6$ & $4/7$ & $5/8$ & $5/9$ & $5/10$ & $6/11$ \\
\bottomrule
\end{tabular}
\end{table}

\noindent \textbf{\textsc{IRS.V-EMax} and the ideal penalty.}
Regarding the penalty functions $z_t^\textsc{Irs.V-EMax}$ and $z_t^\textup{ideal}$, we cannot visualize the optimal solution with a table since the total payoff depends on the detailed sequence of pulls and not only the number of pulls.
While omitting the visual proof of optimality, we have that the action sequence $\mathbf{a}_{1:8}^* = (1, 2, 2, 1, 1, 1, 1, 1)$ achieves the maximal payoff of $5.806$ under $z_t^\textsc{Irs.V-EMax}$, and $\mathbf{a}_{1:8}^* = ( 1, 1, 1, 1, 1, 1, 1, 1 )$ achieves the maximal payoff of $6.063$ under $z_t^\textup{ideal}$.
In particular for $z_t^\textup{ideal}$, the maximal payoff depends only on the prior belief $\mathbf{y}$ and the time horizon $T$, irrespective of the outcome\footnote{
	For details, see the proof of the strong duality theorem in \S \ref{prf-weak-daulity}.
	While the maximal value does not depend on $\omega$, the optimal action sequence still depends on $\omega$.
	More specifically, it is the sequence of actions that the (non-anticipating) Bayesian optimal policy will take when $\omega$ is sequentially revealed.
} $\omega$.

We have so far illustrated how the different penalty functions induce the different inner problems and the different best actions given the same outcome $\omega$.
The readers may notice from the above examples that, as the penalty function becomes more complicated, the hindsight best action sequence becomes less dependent on a particular realization of
$\omega$. Instead, it becomes more dependent on the prior belief.

\subsection{IRS Performance Bounds}

The maximal payoffs above are calculated for a particular outcome given by Table \ref{tbl-outcome-example}.
Recall that the IRS performance bound $W^z$ is defined as the expected value of the maximal payoff where the expectation is taken with respect to the randomness of outcome $\omega$ over its prior distribution $\Iscr(T, \mathbf{y})$.
We can obtain this value by simulation, i.e., by solving a bunch of inner problems with respect to the randomly generated outcomes $\omega^{(1)}, \omega^{(2)}, \cdots, \omega^{(S)}$ and taking the average of the maximal values.
For this particular Bernoulli MAB setting ($T=8$ with given priors), we obtain the following performance bounds:

\begin{table}[htbp]
\centering
\begin{tabular}{c |  c | c  | c | c || c }
\toprule
$W^{0}$ & $W^\textsc{TS}$ & $W^\textsc{Irs.FH}$ & $W^\textsc{Irs.V-Zero}$ & $W^\textsc{Irs.V-EMax}$ & $W^\textup{ideal} = V^*$ \\
\hline
\hline
6.805 & 6.429 & 6.279 & 6.111 & 6.075 & 6.063 \\
\bottomrule
\end{tabular}
\end{table}

We observe that the performance bounds are monotone, i.e., $W^{0} > W^\textsc{TS} > W^\textsc{Irs.FH} > W^\textsc{Irs.V-Zero} > W^\textsc{Irs.V-EMax} > W^\textup{ideal} = V^*$, which is consistent with Theorem \ref{thm-monotonicity}.

\subsection{Illustration of the IRS Policy (IRS.V-Zero)}

We illustrate how the policy $\pi^\textsc{Irs.V-Zero}$ makes decisions sequentially when the true outcome $\omega$ is the one specified in Table \ref{tbl-outcome-example}.
At $t=1$, it first synthesizes a future scenario based on the prior belief (i.e., sampling $\tilde{\omega}_1 \sim \Iscr(\mathbf{y}_0)$) and finds the best action sequence in the presence of penalties $z_t^\textsc{Irs.V-Zero}$ in the belief that the sampled outcome $\tilde{\omega}_1$ is the ground truth.
The following table shows an example in which $\pi^\textsc{Irs.V-Zero}$ plays arm 1.

\begin{table}[htbp]
\centering
\begin{tabular}{c | c |  c  c  c  c  c  c  c  c  | c}
\toprule
\multirow{2}{*}{$t=1$} & \multirow{2}{*}{ \thead{Priors $\mathbf{y}_{0}$} } & \multicolumn{8}{c|}{ \thead{Payoffs with respect to $\tilde{\omega}_1 \sim \Iscr(\mathbf{y}_{0})$} } & \multirow{2}{*}{ \thead{Action} } \\
\cline{3-10} 
  &  & $n=1$ & 2 & 3 & 4 & 5 & 6 & 7 & 8 \\
\hline
\hline
Arm 1 & Beta($3,1)$ & \cellcolor{black!25}$3/4$ & \cellcolor{black!25}$4/5$ & \cellcolor{black!25}$5/6$ & \cellcolor{black!25}$6/7$ & \cellcolor{black!25}$7/8$ & \cellcolor{black!25}$7/9$ & \cellcolor{black!25}$8/10$ & \cellcolor{black!25}$9/11$ & \multirow{3}{*}{$a_{1}=1$} \\
Arm 2 & Beta($1,1)$ & $1/2$ & $1/3$ & $1/4$ & $1/5$ & $1/6$ & $1/7$ & $2/8$ & $3/9$ \\
Arm 3 & Beta($1,3)$ & $1/4$ & $1/5$ & $1/6$ & $1/7$ & $1/8$ & $1/9$ & $1/10$ & $2/11$ \\
\bottomrule
\end{tabular}
\end{table}

As a result of the first action ($a_1=1$), we observe that $R_{1,1}=0$ (encoded in the true outcome $\omega$) and the associated belief is updated from $\text{Beta}(3,1)$ to $\text{Beta}(3,2)$ according to Bayes' rule.
In order to make the next decision $a_2$ at time $t=2$, $\pi^\textsc{Irs.V-Zero}$ simulates an outcome for the remaining time horizon, i.e., $\tilde{\omega}_2 \sim \Iscr(\mathbf{y}_1)$, independently of the outcome $\tilde{\omega}_1$ used at $t=1$.
Again, $\pi^\textsc{Irs.V-Zero}$ finds the best action sequence for this new scenario and takes its first action.\footnote{In case of \textsc{Irs.V-Zero}, we select the arm with the largest pull allocation as a first action.}
The table below shows an instance of $\tilde{\omega}_2$ in which the policy will pull arm 2.

\begin{table}[htbp]
\centering
\begin{tabular}{c | c |  c  c  c  c  c  c  c  | c}
\toprule
\multirow{2}{*}{$t=2$} & \multirow{2}{*}{ \thead{Priors $\mathbf{y}_{1}$} } & \multicolumn{7}{c|}{ \thead{Payoffs with respect to $\tilde{\omega}_2 \sim \Iscr(\mathbf{y}_{1})$} } & \multirow{2}{*}{ \thead{Action} } \\
\cline{3-9} 
  &  & $n=1$ & 2 & 3 & 4 & 5 & 6 & 7 \\
\hline
\hline
Arm 1 & Beta($3,2)$ & \cellcolor{black!25}$3/5$ & \cellcolor{black!25}$4/6$ & $4/7$ & $4/8$ & $4/9$ & $5/10$ & $5/11$ & \multirow{3}{*}{$a_{2}=2$} \\
Arm 2 & Beta($1,1)$ & \cellcolor{black!25}$1/2$ & \cellcolor{black!25}$2/3$ & \cellcolor{black!25}$3/4$ & \cellcolor{black!25}$3/5$ & \cellcolor{black!25}$4/6$ & $4/7$ & $5/8$ \\
Arm 3 & Beta($1,3)$ & $1/4$ & $1/5$ & $1/6$ & $1/7$ & $1/8$ & $1/9$ & $1/10$ \\
\bottomrule
\end{tabular}
\end{table}

We can update the prior of arm 2 as a new reward realization $R_{2,1}=1$ is revealed.
In the following decision epochs $t=3,4,\cdots$, the policy repeats the same decision-making procedure -- (i) samples $\tilde{\omega}_t \sim \Iscr(\mathbf{y}_{t-1})$, (ii) solves the inner problem, and (iii) plays the best arm that the optimal solution suggests -- while updating the priors as the true reward realizations are revealed sequentially.

The following table illustrates the last decision epoch.
As there remains one time period only, the policy $\pi^\textsc{Irs.V-Zero}$ tries to maximize $\hat{\mu}_{a,0}(\tilde{\omega}_7) = \bar{\mu}_a(\mathbf{y}_7)$, which is the expected mean reward given the prior at that moment.
Such a decision is totally myopic, but it is Bayesian optimal.

\begin{table}[htbp]
\centering
\begin{tabular}{c | c |  c  | c}
\toprule
\multirow{2}{*}{$t=8$} & \multirow{2}{*}{ \thead{Priors $\mathbf{y}_{7}$} } & \multicolumn{1}{c|}{ \thead{Payoffs with respect to $\tilde{\omega}_7 \sim \Iscr(\mathbf{y}_{7})$} } & \multirow{2}{*}{ \thead{Action} } \\
\cline{3-3} 
  &  & $n=1$ \\
\hline
\hline
Arm 1 & Beta($6,3)$ & \cellcolor{black!25}$6/9$ & \multirow{3}{*}{$a_{8}=1$} \\
Arm 2 & Beta($2,2)$ & $2/4$ \\
Arm 3 & Beta($1,3)$ & $1/4$ \\
\bottomrule
\end{tabular}
\end{table}

\newpage

\section{Algorithms in Detail}

\subsection{Implementation of \textsc{IRS.V-Zero}} \label{ss-irs-vzero-impl}
We provide a pseudo-code of the policy $\pi^\textsc{Irs.V-Zero}$ introduced in \S \ref{ss-irs-vzero}.
The same logic can be directly used to compute the performance bound $W^\textsc{Irs.V-Zero}$ if the sampled outcome $\tilde{\omega}$ is replaced with the true outcome $\omega$.

\begin{algorithm2e}[H]
  \SetAlgoLined\DontPrintSemicolon
  \SetKwFunction{algo}{IRS.V-Zero}
  \SetKwProg{myalg}{Function}{}{}
  \myalg{\algo{$T, \mathbf{y}$}}{
\nl	$\tilde{\theta}_a \sim \Pscr_a(y_a), \tilde{R}_{a,n} \sim \Rscr_a(\tilde{\theta}), \quad \forall n \in \{1, \ldots, T\}, \forall a \in \{1, \ldots, K\}$ \;
\nl	\For{$a=1,\cdots,K$} {
\nl		$\tilde{y}_{a,0} \gets y_a, \tilde{S}_{a,0} \gets 0$ \;
\nl		\For{$n=1,\cdots,T$} {
\nl			$\tilde{S}_{a,n} \gets \tilde{S}_{a,n-1} + \bar{\mu}_a(\tilde{y}_{a,n-1})$ \;
\nl			 $\tilde{y}_{a,n} \gets \Uscr_a( \tilde{y}_{a,n-1}, \tilde{R}_{a,n} )$ \;
		}
	}
\nl	$\tilde{M}_{0,0} \gets 0, \tilde{M}_{0,n} \gets -\infty, \forall n \in \{1, \ldots, T\}$ \;
\nl	\For{$a=1,\cdots,K$} {
\nl		\For{$n=0,\cdots,T$} {
\nl			$\tilde{M}_{a,n} \gets \max_{0 \leq m \leq n}\{ \tilde{M}_{a-1,n-m} + \tilde{S}_{a,m} \} $ \;
\nl			$\tilde{L}_{a,n} \gets \argmax_{0 \leq m \leq n}\{ \tilde{M}_{a-1,n-m} + \tilde{S}_{a,m} \} $ \;
		}
	}
\nl	$\tau \gets T$ \;
\nl	\For{$a=K,\cdots,1$} {
\nl		$\tilde{n}_a^* \gets \tilde{L}_{a,\tau}$ \;
\nl		$\tau \gets \tau - \tilde{n}_a^*$ \;
	}
\nl	\KwRet $\argmax_a \tilde{n}_a^*$ \;
  }{}
  \caption{ \newedit{Arm selection rule of $\pi^\textsc{Irs.V-Zero}$ when remaining time is $T$ and current belief is $\mathbf{y}$} }
\end{algorithm2e}

\subsection{Implementation of \textsc{IRS.V-EMax}} \label{ss-irs-vemax-impl}
We use the notation $\mathbf{y}_t( \mathbf{n}_{1:K}, \omega )$ to denote the belief as a function of pull counts $\mathbf{n}_{1:K}\defeq ( n_1,\cdots, n_K ) \in \N_0^K$, based on the observation that the belief is completely \newedit{determined} by how many times each of the arms \newedit{has been} pulled, $\mathbf{n}_{1:K}$, irrespective of the specific sequence in which the arms \newedit{have been} pulled.
Given the pull counts $\mathbf{n}_{1:K}$, we define the payoff of pulling arm $a$ one more time after pulling \newedit{the individual arms} $n_1,\cdots,n_K$ times \newedit{respectively}: with $t=\sum_{a=1}^K n_a$, \newedit{the effective payoff associated with arm $a$ at time $t$ is}
\begin{equation}
	r^z( \mathbf{n}_{1:K}, a, \omega ) \defeq \hat{\mu}_{a,n_a}(\omega) - W^\textsc{TS}\left( T-t-1, \mathbf{y}_{t+1}( \mathbf{n}_{1:K} + \mathbf{e}_a, \omega ) \right)
		 + W^\textsc{TS}\left( T-t-1, \mathbf{y}_t( \mathbf{n}_{1:K}, \omega ) \right),
\end{equation}
where $\mathbf{e}_a \in \N_0^K$ is a basis vector such that the $a^\text{th}$ component is one and the others are zero.
Note that we used the fact that $\E\left[ \left. W^\textsc{TS}\left( T-t, \mathbf{y}_t \right) \right| \newedit{H_{t-1}} \right] = W^\textsc{TS}\left( T-t, \mathbf{y}_{t-1} \right)$.

Consider a subproblem of \eqref{e-inner-problem} that maximizes the total payoff given the number of pulls $\mathbf{n}_{1:K}$ across all the arms: with $t = \sum_{a=1}^K n_a$, we get
\begin{equation}
	M(\mathbf{n}_{1:K}, \omega) \defeq \max_{\mathbf{a}_{1:t} \in \Ascr^t}\left\{ \sum_{s=1}^t r_s(\mathbf{a}_{1:s},\omega) - z_s^\textsc{Irs.V-EMax}(\mathbf{a}_{1:s}, \omega); ~ \sum_{s=1}^t \mathbf{1}\{ a_s=a \} = n_a, \forall a \right\}.
\end{equation}
Consequently, the maximal value $M(\mathbf{n}_{1:K}, \omega)$ should satisfy the following Bellman equation:
\begin{equation} \label{e-irs-vemax-subproblem}
	M(\mathbf{n}_{1:K}, \omega) = \max_{a \in \Ascr : n_a \geq 1}\left\{ M(\mathbf{n}_{1:K} - \mathbf{e}_a, \omega ) + r^z( \mathbf{n}_{1:K} - \mathbf{e}_a, a, \omega ) \right\},
\end{equation}
\newedit{i.e., when letting $a^*$ be the maximizer of \eqref{e-irs-vemax-subproblem}, it is optimal to play arm $a^*$ after making the best effort within the allocation $\mathbf{n}_{1:K} - \mathbf{e}_a$.}
For all feasible counts $\mathbf{n}_{1:K}$'s such that $\sum_{a=1}^K n_a \leq T$, we can compute $M(\mathbf{n}_{1:K}, \omega)$'s by sequentially solving \eqref{e-irs-vemax-subproblem} in an appropriate order.
By doing so, we can obtain the maximal value of the original inner problem \eqref{e-inner-problem} by evaluating
\begin{equation} \label{e-irs-vemax-finalize}
	\max_{\mathbf{n}_{1:K} \in N_T }\left\{ M(\mathbf{n}_{1:K}, \omega) \right\},
\end{equation}
where $N_T \defeq \{ (n_1,\cdots,n_K) \in \N_0^K: \sum_{a=1}^K n_a = T \}$\newedit{, and the performance bound $W^\textsc{Irs.V-EMax}$ is the expected value of \eqref{e-irs-vemax-finalize} with respect to the random realization of $\omega$.}
The optimal action sequence $\mathbf{a}_{1:T}^*$ can be obtained by tracking $M(\mathbf{n}_{1:K}, \omega)$'s backward.
\\

\begin{algorithm2e}[H]
  \SetAlgoLined\DontPrintSemicolon
  \SetKwFunction{algo}{IRS.V-EMax}
  \SetKwProg{myalg}{Function}{}{}
  \myalg{\algo{$T, \mathbf{y}$}}{
\nl	$\tilde{\theta}_a \sim \Pscr_a(y_a), \tilde{R}_{a,n} \sim \Rscr_a(\tilde{\theta}), \quad \forall n \in \{1, \ldots, T\}, \forall a \in \{1, \ldots, K\}$ \;
\nl	$ \tilde{y}_{a,0} \gets y_a, \tilde{y}_{a,n} \gets \Uscr_a( \tilde{y}_{a,n-1}, \tilde{R}_{a,n} ), \quad \forall n \in \{1, \ldots, T\}, \forall a \in \{1, \ldots, K\}$ \;
\nl	\For{each $\mathbf{n}_{1:K} \in N_{\leq T}$} {
\nl		$\tilde{\Gamma}[ \mathbf{n}_{1:K} ] \gets \E_{\tilde{\mathbf{y}}(\mathbf{n}_{1:K})}\left[ \max_a \mu_a(\theta_a) \right]$ \;
	}
\nl	\For{each $\mathbf{n}_{1:K} \in N_{< T}$} {
\nl		$ \tilde{r}^z[ \mathbf{n}_{1:K}, a ] \gets \bar{\mu}_a( \tilde{y}_{a,n_a-1} ) + \left( T - \sum_{a=1}^K n_a - 1 \right) \times \left( \tilde{\Gamma}[ \mathbf{n}_{1:K} ] - \tilde{\Gamma}[ \mathbf{n}_{1:K} + \mathbf{e}_a] \right), ~ \forall a \in \{1, \ldots, K\}$ \;
	}
\nl	$\tilde{M}[ \mathbf{0} ] \gets 0$ \;
\nl	\For{each $\mathbf{n}_{1:K} \in N_{\leq T} \setminus \{ \mathbf{0} \}$ in order \newedit{with increasing $\sum_{a=1}^K n_a$}} {
\nl		$\tilde{M}[ \mathbf{n}_{1:K} ] \gets \max_{a : n_a > 0} \left\{ \tilde{M}[ \mathbf{n}_{1:K} - \mathbf{e}_a ] + \tilde{r}^z[ \mathbf{n}_{1:K} - \mathbf{e}_a, a ] \right\}$ \;
\nl		$\tilde{A}[ \mathbf{n}_{1:K} ] \gets \argmax_{a : n_a > 0} \left\{ \tilde{M}[ \mathbf{n}_{1:K} - \mathbf{e}_a ] + \tilde{r}^z[ \mathbf{n}_{1:K} - \mathbf{e}_a, a ] \right\}$ \;
	}
\nl	$\mathbf{m}_{1:K} \gets \argmax_{\mathbf{n}_{1:K} \in N_T}\left\{ \tilde{M}[ \mathbf{n}_{1:K} ] \right\}$ \;
\nl	\For{$t=T,\cdots,1$} {
\nl		$\tilde{a}_t^* \gets \tilde{A}[ \mathbf{m}_{1:K} ]$ \;
\nl		$m_{\tilde{a}_t^*} \gets m_{\tilde{a}_t^*} - 1$ \;
	}
\nl	\KwRet $\tilde{a}_1^*$ \;
  }{}
  \caption{ \newedit{Arm selection rule of $\pi^\textsc{Irs.V-Zero}$ when remaining time is $T$ and current belief is $\mathbf{y}$} }
\end{algorithm2e}
Here, $\tilde{\mathbf{y}}(\mathbf{n}_{1:K}) \defeq ( \tilde{y}_{1,n_1}, \cdots, \tilde{y}_{K,n_K})$, $N_{\leq T} \defeq \{ \mathbf{n}_{1:K}; \sum_a n_a \leq T \}$, $N_{< T} \defeq \{ \mathbf{n}_{1:K}; \sum_a n_a < T \}$, and in line 8, $\mathbf{n}_{1:K}$ iterates over $N_{\leq T} \setminus \{ \mathbf{0} \}$ in an order in which $\sum_{a=1}^K n_a$ is non-decreasing.

Since $| N_{\leq T} | = O(T^K)$, it requires $O(KT^K)$ operations to compute all $M(\mathbf{n}_{1:K}, \omega)$'s.
However, another practical issue is the cost of computing $W^\textsc{TS}(T, \mathbf{y}) = T \times \E_{\newedit{\mathbf{y}}}\left[ \max_a \mu_a(\theta_a) \right]$ which has to be evaluated $O(T^K)$ times in total.
There is no simple closed-form expression in general, and it should be evaluated with numerical integration or Monte Carlo sampling.

\subsection{Implementation of IRS.Index} \label{ss-irs-index-impl}

\newedit{We first prove the identity that was utilized in \S \ref{ss-irs-index}, and then provide the pseudo code for \textsc{Irs.Index} policy.}

\begin{proposition} \label{prop-irs-index-reformulation}
	The optimization problem \eqref{e-irs-index-1} can be reformulated as
	\begin{equation} 
		\max_{0 \leq n \leq T}\left\{ T \times \Gamma^\lambda_0 + (T-n) \times \left( \lambda - \min_{0 \leq i \leq n} \Gamma^\lambda_i \right)  + \sum_{i=1}^n \left( \hat{\mu}_{a,i-1} - \Gamma^\lambda_{i-1} \right) \right\}.
	\end{equation}
	Here, the decision variable $n$ is the total number of pulls of a stochastic arm.
\end{proposition}
\begin{proof}
Fix $m \defeq n_T$\newedit{, i.e., the total number of pulls on the stochastic arm}.
Note that if $a_t=0$, then $(T-t) \times (\Gamma^\lambda_{n_t} - \Gamma^\lambda_{n_{t-1}}) = 0$ since $n_t = n_{t-1}$.
The objective function can be represented as
\begin{equation} \label{e-irs-index-3}
	\sum_{n=1}^m \hat{\mu}_{a,n-1}  + (T - m) \times \lambda - \sum_{n=1}^m (T-t_n) \times \left(\Gamma^\lambda_n - \Gamma^\lambda_{n-1} \right),
\end{equation}
where $t_n \defeq \inf\{ t; n_t \geq n \}$ \newedit{represents the time at which the $n^\text{th}$ pull on the stochastic arm is made}.
It suffices to find \newedit{the optimal pulling times} $(t_1,\cdots,t_m)$ with $1 \leq t_1 < t_2 < \cdots < t_m \leq T$ \newedit{by which} $\sum_{n=1}^m (T-t_n) \times \left(\Gamma^\lambda_n - \Gamma^\lambda_{n-1} \right)$ \newedit{is minimized}.
With $t_0 \defeq 0$ and $t_{m+1} \defeq T+1$, we have
\begin{align}
	&\sum_{n=1}^m (T-t_n) \times \left(\Gamma^\lambda_n - \Gamma^\lambda_{n-1} \right)
		\\&= \sum_{n=1}^m (T-t_n) \times \Gamma^\lambda_n - \sum_{n=1}^m (T-t_n) \times \Gamma^\lambda_{n-1} 
		\\&= \sum_{n=1}^m (T-t_n) \times \Gamma^\lambda_n - \sum_{n=0}^{m-1} (T-t_{n+1}) \times \Gamma^\lambda_n 
		\\&= \sum_{n=0}^m (T-t_n) \times \Gamma^\lambda_n  - (T-t_0) \times \Gamma_0^\lambda  - \sum_{n=0}^m (T-t_{n+1}) \times \Gamma^\lambda_n + (T-t_{m+1}) \times \Gamma_m^\lambda
		\\&= -\Gamma^\lambda_m - T \times \Gamma_0^\lambda + \sum_{n=0}^m (t_{n+1}-t_n) \times \Gamma_n^\lambda.
		\label{e-irs-index-4}
\end{align}
\newedit{Consider the minimum value among $\Gamma_0^\lambda, \ldots, \Gamma_m^\lambda$ and let $n^* \defeq \argmin_{0 \leq n \leq m} \Gamma_n^\lambda$.
In order to minimize \eqref{e-irs-index-4}, it should satisfy that $t_{n+1}-t_n = T-m+1$ for $n = n^*$ and $t_{n+1}-t_n=1$ for $n \ne n^*$}.
For such $t_n$'s, \eqref{e-irs-index-3} reduces to
\begin{align}
	& \sum_{n=1}^m \hat{\mu}_{a,n-1}  + (T - m) \times \lambda - \left( - \Gamma^\lambda_m - T \times \Gamma_0^\lambda + \sum_{n=0}^m \Gamma_n^\lambda + (T-m) \times \min_{0 \leq n \leq m} \Gamma_m^\lambda \right)
	\\&= \sum_{n=1}^m \hat{\mu}_{a,n-1}  + (T - m) \times \left( \lambda - \min_{0 \leq n \leq m} \Gamma_m^\lambda \right) + T \times \Gamma_0^\lambda - \sum_{n=0}^{m-1} \Gamma_n^\lambda.
\end{align}
By taking its maximum value over $m=0,\cdots,T$, we obtain \eqref{e-irs-index-2}.
\end{proof}

\newedit{The following pseudo code implements the arm selection rule of the \textsc{Irs.Index} policy when remaining time is $T$ and current belief is $\mathbf{y}$.
In line 14, the infimum can be found via the bisection method, and $\tilde{\mathbf{y}}_{a,0:T} \defeq ( \tilde{y}_{a,0}, \ldots, \tilde{y}_{a,T} )$ represents the sequence of beliefs under the sampled outcome.
}

\begin{algorithm2e}[H]
  \SetAlgoLined\DontPrintSemicolon
  \SetKwFunction{algosingle}{IRS.Single.Worth-Trying}
  \SetKwProg{myalg}{Function}{}{}
  \myalg{\algosingle{$a, T, \lambda, \newedit{\tilde{\mathbf{y}}_{a,0:T}} $}}{
\nl    $\tilde{\Gamma}_n^\lambda \gets \E_{\newedit{\tilde{y}_{a,n}}}\left[ \max( \mu_a(\theta_a), \lambda ) \right], \forall n \in \{0,\ldots,T\}$ \;
\nl	$\tilde{S}^\mu_{a,0} \gets 0$, $\tilde{S}^\Gamma_0 \gets 0$, $\tilde{m}^\Gamma_0 \gets \tilde{\Gamma}_0^\lambda$ \;
\nl	\For{$n=1,\cdots,T$} {
\nl		$\tilde{S}^\mu_{a,n} \gets \tilde{S}^\mu_{a,n-1} + \bar{\mu}_a(\tilde{y}_{a,n-1})$ \;
\nl		 $\tilde{S}^\Gamma_n \gets \tilde{S}^\Gamma_{a,n-1} + \tilde{\Gamma}_n^\lambda$ \;
\nl		 $\tilde{m}^\Gamma_n \gets \min\left( \tilde{m}^\Gamma_{n-1}, \tilde{\Gamma}_{n-1}^\lambda \right)$ \;
	}
\nl	$\tilde{\varphi}_a \gets \max_{1 \leq n \leq T} \left\{ \tilde{S}^\mu_{a,n} + T \times \tilde{\Gamma}^\lambda_0 + (T-n) \times \left( \lambda - \tilde{m}^\Gamma_n \right)  - \tilde{S}^\Gamma_n \right\} - T \times \lambda$ \;
\nl	\uIf{$\tilde{\varphi}_a \geq 0$} {
\nl		\KwRet \text{true} \;
	} \Else {
\nl		\KwRet \text{false} \;
	}
  }{}
\vspace{0.3cm}
  \SetKwFunction{algo}{IRS.Index}
  \SetKwProg{myalg}{Function}{}{}
  \myalg{\algo{$T, \mathbf{y}$}}{
\nl	$\tilde{\theta}_a \sim \Pscr_a(y_a), \tilde{R}_{a,n} \sim \Rscr_a(\tilde{\theta}), \quad \forall n \in \{1, \ldots, T\}, \forall a \in \{1, \ldots, K\}$ \;
\nl	$ \tilde{y}_{a,0} \gets y_a, \quad \tilde{y}_{a,n} \gets \Uscr_a( \tilde{y}_{a,n-1}, \tilde{R}_{a,n} ), \quad \forall n \in \{1, \ldots, T\}, \quad \forall a \in \{1, \ldots, K\}$ \;
\nl	\For{$a=1,\cdots,K$} {
\nl		$\tilde{\lambda}^*_a \gets \inf \left\{ \lambda; \algosingle{$a, T, \lambda, \newedit{\tilde{\mathbf{y}}_{a,0:T}}$} = \text{true} \right\} $ \;
	}
\nl	\KwRet $\argmax_a \tilde{\lambda}^*_a $ \;
  }{}
  \caption{ \newedit{Arm selection rule of \textsc{Irs.Index} policy when remaining time is $T$ and current belief is $\mathbf{y}$} }
\end{algorithm2e}
\vspace{0.5cm}

\section{Proofs for \S \ref{s-framework}} \label{ss-proofs-duality}

\begin{proposition}[Mean equivalence] \label{prop-mean-equivalence}
If the penalty function $z_t$ is dual feasible, \newedit{the presence of penalties does not affect the performance of a non-anticipating policy $\pi$: i.e.,}
\begin{equation}	
	\E^\pi_{\newedit{\mathbf{y}}}\left[ \sum_{t=1}^T r_t( \newedit{\mathbf{A}_{1:t}^\pi}, \omega ) - z_t(\newedit{\mathbf{A}_{1:t}^\pi}, \omega) \right]
	= \E^\pi_{\newedit{\mathbf{y}}}\left[ \sum_{t=1}^T r_t( \newedit{\mathbf{A}_{1:t}^\pi}, \omega ) \right]
	=: V(\pi, T, \mathbf{y}).
\end{equation}
\end{proposition}
\begin{proof}
	\newedit{The claim immediately follows from the definition of dual feasibility and the linearity of the expectation operator.}
\end{proof}

\subsection{Proof of Theorem \ref{thm-weak-duality}} \label{prf-weak-daulity}

\newedit{Despite that the results of Theorem \ref{thm-weak-duality} were already well established in \citet{Brown10}, we provide the detailed proof as our context is slightly different from that of \citet{Brown10} regarding the measurability of $r_t$.
We define an appending operator $\oplus$ that concatenates an element into a vector so that $\mathbf{a}_{1:t} = \mathbf{a}_{1:t-1} \oplus a_t$.}

\noindent \textbf{Weak duality.}
Define \newedit{the filtration for the perfect information relaxation} $\Gscr_t \defeq \Fscr_t \cup \sigma(\omega)$ and consider a relaxed policy space $\Pi_\mathbb{G} \defeq \left\{ \pi: \newedit{A_t^\pi} \text{ is $\Gscr_{t-1}$-measurable, } \forall t \right\} $.
Then, we have
\begin{align}
	V^*(T, \mathbf{y}) &\defeq \sup_{\pi \in \Pi_\mathbb{F}} \E\left[ \sum_{t=1}^T r_t( \newedit{\mathbf{A}_{1:t}^\pi} ) \right]
		\stackrel{ \text{Prop \ref{prop-mean-equivalence}} }{=} \sup_{\pi \in \Pi_\mathbb{F}} \E\left[ \sum_{t=1}^T r_t( \newedit{\mathbf{A}_{1:t}^\pi} ) - z_t( \newedit{\mathbf{A}_{1:t}^\pi} ) \right]
		\\&\leq \sup_{\pi \in \Pi_\mathbb{G}} \E\left[ \sum_{t=1}^T r_t( \newedit{\mathbf{A}_{1:t}^\pi} ) - z_t( \newedit{\mathbf{A}_{1:t}^\pi} ) \right]
		= \E\left[ \max_{\mathbf{a}_{1:T} \in \Ascr^T} \sum_{t=1}^T r_t( \mathbf{a}_{1:t} ) - z_t( \mathbf{a}_{1:t} ) \right]
		\\&= W^z(T, \mathbf{y}),
\end{align}
where the inequality holds since $\Pi_\mathbb{F} \subseteq \Pi_\mathbb{G}$.
\qed

\noindent \textbf{Strong duality.}	\label{prf-strong-duality}
Fix $T$ and $\mathbf{y}$.
Let $V^\textup{in}_t(\mathbf{a}_{1:t-1}, \omega)$ and $Q^\textup{in}_t(\mathbf{a}_{1:t-1}, a, \omega)$ be, respectively, the value function and the state-action value (Q-value) function that are associated with the inner problem \eqref{e-inner-problem} given a particular outcome $\omega$ under the ideal penalty \eqref{e-ideal-penalty}.
With $V^\textup{in}_{T+1} \equiv 0$, we have the \newedit{following} Bellman equation for the inner problem:
\begin{align}
	Q^\textup{in}_t(\mathbf{a}_{1:t-1}, a, \omega) 
		&\defeq r_t(\mathbf{a}_{1:t-1} \oplus a, \omega) - z_t^\text{ideal}(\mathbf{a}_{1:t-1} \oplus a, \omega) + V^\textup{in}_{t+1}(\mathbf{a}_{1:t-1} \oplus a, \omega),
	\\
	V^\textup{in}_t(\mathbf{a}_{1:t-1}, \omega)
		&= \max_{a \in \Ascr} \left\{ Q^\textup{in}_t(\mathbf{a}_{1:t-1}, a, \omega) \right\}.
\end{align}
We argue by induction to show that
\begin{align}
	V^\textup{in}_t(\mathbf{a}_{1:t-1}, \omega) &= V^*(T-t+1, \mathbf{y}_{t-1}(\mathbf{a}_{1:t-1},\omega)),
	\\
	Q^\textup{in}_t(\mathbf{a}_{1:t-1}, a, \omega) &= Q^*(T-t+1, \mathbf{y}_{t-1}(\mathbf{a}_{1:t-1}, \omega), a),
\end{align}
for all $\mathbf{a}_{1:t-1} \in \Ascr^{t-1}$, $a \in \Ascr$ and $t \in \{1, \ldots, T+1\}$.

As a terminal case, when $t=T+1$, the claim holds trivially, since $V^\textup{in}_{T+1}( \mathbf{a}_{1:T}, \omega) = 0 = V^*(0, \mathbf{y}_{T}(\mathbf{a}_{1:T},\omega))$.
Now assume that the claim holds for $t+1$: \newedit{i.e.,} $V^\textup{in}_{t+1}(\mathbf{a}_{1:t}, \omega) = V^*(T-t, \mathbf{y}_t(\mathbf{a}_{1:t},\omega))$ for all $\mathbf{a}_{1:t} \in \Ascr^t$.
For any $\mathbf{a}_{1:t-1} \in \Ascr^{t-1}$ and $a \in \Ascr$, then,
\begin{align}
	Q^\textup{in}_t(\mathbf{a}_{1:t-1}, a, \omega)
		&= r_t(\mathbf{a}_{1:t-1} \oplus a, \omega) - z_t^\text{ideal}(\mathbf{a}_{1:t-1} \oplus a, \omega) + V^\textup{in}_{t+1}(\mathbf{a}_{1:t-1} \oplus a, \omega) 
		\\&= \E\left[ \left. r_t( \mathbf{a}_{1:t-1} \oplus a, \omega ) + V^*\left( T-t, \mathbf{y}_t(\mathbf{a}_{1:t-1} \oplus a, \omega) \right) \right| \newedit{H_{t-1}}(\mathbf{a}_{1:t-1}, \omega) \right] 
			\\ & \quad \quad \underbrace{ - V^*\left( T-t, \mathbf{y}_t(\mathbf{a}_{1:t-1} \oplus a, \omega) \right) +  V^\textup{in}_{t+1}(\mathbf{a}_{1:t-1} \oplus a, \omega) }_{=0}
		\\&= \E\left[ \left. r_t( \mathbf{a}_{1:t-1} \oplus a, \omega ) + V^*\left( T-t, \mathbf{y}_t(\mathbf{a}_{1:t-1} \oplus a, \omega) \right) \right| \newedit{H_{t-1}}(\mathbf{a}_{1:t-1}, \omega) \right]
		\\&= \E_{\newedit{\mathbf{y}_{t-1}(\mathbf{a}_{1:t-1},\omega)}} \left[ R_a + V^*\left( T-t, \Uscr( \mathbf{y}_{t-1}(\mathbf{a}_{1:t-1},\omega), a, R_a) \right) \right] 
		\\&= Q^*( T-t, \mathbf{y}_{t-1}(\mathbf{a}_{1:t-1}, \omega), a ),
\end{align}
where the last equality follows from the original Bellman equation \eqref{e-bellman}.
Consequently, \newedit{we obtain}
\begin{align}
	V^\textup{in}_t(\mathbf{a}_{1:t-1}, \omega)
		&= \max_{a \in \Ascr} \left\{ Q^\textup{in}_t(\mathbf{a}_{1:t-1}, a, \omega) \right\}
		\\&= \max_{a \in \Ascr} \left\{ Q^*( T-t, \mathbf{y}_{t-1}(\mathbf{a}_{1:t-1}, \omega), a ) \right\}
		\\&= V^*( T-t, \mathbf{y}_{t-1}(\mathbf{a}_{1:t-1}, \omega) ).
\end{align}
Therefore the claim holds for all $t=1,\cdots,T$.
In particular for $t=1$, we have
\begin{equation} \label{e-strong-duality-1}
	V_1^\textup{in}(\emptyset, \omega ) = V^*(T, \mathbf{y})
	, \quad Q_1^\textup{in}(\emptyset, a, \omega ) = Q^*(T, \mathbf{y}, a)
	, \quad \forall \omega.
\end{equation}
Note that the maximal value of the inner problem does not depend on the outcome $\omega$, \newedit{i.e., it} is deterministic with respect to the randomness of $\omega$.
As its expected value, $W^\text{ideal}(T, \mathbf{y}) = V^*(T,\mathbf{y})$.
\qed

\subsection{Proof of Remark \ref{rem-ideal-penalty-optimality}}

We proceed on the proof of strong duality.
The policy $\pi^\text{ideal}$ solves the same inner problem with respect to a randomly sampled outcome $\tilde{\omega}$.
When the remaining time is $T$ and the current belief is $\mathbf{y}$, it takes an action with the largest Q-value: together with \eqref{e-strong-duality-1}, it yields
\begin{equation}
	a^{\pi^\text{ideal}} = \argmax_a Q_1^\text{in}( \emptyset, a, \tilde{\omega}) = \argmax_a Q^*(T, \mathbf{y}, a ).
\end{equation}
Therefore, at each moment, \newedit{irrespective of the sampled outcome $\tilde{\omega}$}, the policy $\pi^\text{ideal}$ always takes the same action that the Bayesian optimal policy would take.
Although there might be some ambiguity regarding tie breaking in $\argmax$, it does not affect the expected performance.
Therefore, $V(\pi^\text{ideal}, T, \mathbf{y}) = V^*(T, \mathbf{y})$.
\qed

\subsection{Proof of Remark \ref{rem-dual-feasibility}} \label{prf-dual-feasible}
\newedit{
First observe that for any non-anticipating policy $\pi \in \Pi_\mathbb{F}$, since $A_t^\pi$ is $\Fscr_{t-1}$-measurable, we have
\begin{equation}
	\E_\mathbf{y}\left[ \sum_{t=1}^T r_t(\mathbf{A}_{1:t}^\pi, \omega ) \right]
		= \E_\mathbf{y}\left[ \sum_{t=1}^T \E\left( \left. r_t(\mathbf{A}_{1:t}^\pi, \omega ) \right| \Fscr_{t-1}, \bm{\theta} \right) \right]
		= \E_\mathbf{y}\left[ \sum_{t=1}^T \mu_{A_t^\pi}(\theta_{A_t^\pi}) \right].
\end{equation}
Since $\E[r_t(\mathbf{a}_{1:t}, \omega) | \bm{\theta}] = \mu_{a_t}(\theta_{a_t})$ for any $\mathbf{a}_{1:t} \in \Ascr^t$, we further deduce that
\begin{equation}
	\E_\mathbf{y}\left[ \sum_{t=1}^T z_t^\textsc{TS}( \mathbf{A}_{1:t}^\pi, \omega ) \right]
		= \E_\mathbf{y}\left[ \sum_{t=1}^T r_t(\mathbf{A}_{1:t}^\pi, \omega ) \right] - \E_\mathbf{y}\left[ \sum_{t=1}^T \mu_{A_t^\pi}(\theta_{A_t^\pi}) \right]
		= 0,
\end{equation}
and thus $z_t^\textsc{TS}$ is dual feasible.

Also observe that $\E[r_t(\mathbf{a}_{1:t}) | \hat{\bm{\mu}}_{T-1} ] = \E[ \mu_{a_t} | \hat{\bm{\mu}}_{T-1} ] = \E[ \mu_{a_t} | \hat{\bm{\mu}}_{T-1}, H_{t-1} ]$ and $\E[ r_t(\mathbf{a}_{1:t}) | H_{t-1} ] = \E[ \mu_{a_t} | H_{t-1} ]$ for any $\mathbf{a}_{1:t} \in \Ascr^t$.
We can easily verify that each of penalty functions \eqref{e-ideal-penalty}--\eqref{e-penalty-irs-vemax} has a form of
\begin{equation}
	z_t(\mathbf{a}_{1:t}, \omega ) = z_t^\textsc{TS}(\mathbf{a}_{1:t}, \omega) + w_t(\mathbf{a}_{1:t}, \omega) - \E[ w_t(\mathbf{a}_{1:t}, \omega) | G_{t-1}(\mathbf{a}_{1:t-1}, \omega) ],
\end{equation}
for some deterministic function $w_t$ and some relaxed information set $G_{t-1} \supseteq H_{t-1}$.
By invoking Proposition 2.3 (iii) of \citet{Brown10}, we have that $z_t^\textsc{Irs.FH} - z_t^\textsc{TS}$, $z_t^\textsc{Irs.V-Zero} - z_t^\textsc{TS}$, $z_t^\textsc{Irs.V-EMax} - z_t^\textsc{TS}$, and $z_t^\text{ideal} - z_t^\textsc{TS}$ are dual feasible, and therefore so are $z_t^\textsc{Irs.FH}$, $z_t^\textsc{Irs.V-Zero}$, $z_t^\textsc{Irs.V-EMax}$, and $z_t^\text{ideal}$.
}
\qed

\section{Proofs for \S \ref{s-analysis}}

\subsection{Notes on Regularity}
\begin{proposition} \label{prop-integrability}
	If $\newedit{\E_\mathbf{y}}|R_{a,n}| < \infty$ for all $a$,
	\begin{equation}
		\newedit{\E_\mathbf{y}}| \mu_a(\theta_a) |  < \infty,
		\quad \text{and} \quad
		W^\textsc{TS}(T, \mathbf{y}) < \infty
		, \quad \forall T \in \N.
	\end{equation}
\end{proposition}

\begin{proof}
By Jensen's inequality,
\begin{equation}
	\newedit{\E_\mathbf{y}}| \mu_a(\theta_a) | 
		= \newedit{\E_\mathbf{y}}\left[ \left| \E\left( R_{a,n} | \theta_a \right) \right| \right]
		\leq \newedit{\E_\mathbf{y}}\left[ \E\left( \left. | R_{a,n} | \right| \theta_a \right) \right]
		= \newedit{\E_\mathbf{y}}| R_{a,n} |
		< \infty.
\end{equation}
Consequently,
\begin{equation}
	\newedit{\E_\mathbf{y}}\left[ \max_a \mu_a(\theta_a) \right] 
		\leq \newedit{\E_\mathbf{y}}\left[ \sum_{a=1}^K |\mu_a(\theta_a)| \right]
		= \sum_{a=1}^K \newedit{\E_\mathbf{y}}| \mu_a(\theta_a) |
		< \infty.
\end{equation}
\newedit{The claim holds since $W^\textsc{TS}(T, \mathbf{y}) = T \times \E_\mathbf{y}[ \max_a \mu_a(\theta_a) ]$.}
\end{proof}

\begin{proposition} \label{prop-convergence}
	If $\newedit{\E_\mathbf{y}}|R_{a,n}| < \infty$,
	\begin{equation}
		\lim_{n \rightarrow \infty} \hat{\mu}_{a,n}(\omega; y_a) = \lim_{n \rightarrow \infty} \frac{1}{n} \sum_{i=1}^n R_{a,i} = \mu_a(\theta_a)
			\quad \text{almost surely},
	\end{equation}
	where $\hat{\mu}_{a,n}(\omega; y_a) \defeq \newedit{\E_{y_a}}\left[ \mu_a(\theta_a) | R_{a,1}, \cdots, R_{a,n} \right]$.
\end{proposition}

\begin{proof}
Fix $a$ and let $\Hscr_n \defeq \sigma\left( R_{a,1}, \cdots, R_{a,n} \right)$.
First note that, by the strong law of large numbers, $\lim_{n \rightarrow \infty} \frac{1}{n} \sum_{i=1}^n R_{a,i} = \mu_a(\theta_a)$ almost surely.
Therefore, $\mu_a(\theta_a)$ is measurable with respect to $\Hscr_\infty \defeq \bigcup_n \Hscr_n$.
Also note that $\hat{\mu}_{a,n} = \E\left( \mu_a(\theta_a) | \Hscr_n \right)$ is a Doob martingale adapted to $\Hscr_n$.
By Levy's upward theorem, since $\mu_a(\theta_a) \in \Lscr^1$ by Proposition \ref{prop-integrability}, $\hat{\mu}_{a,n}$ converges to $\E\left( \mu_a(\theta_a) | \Hscr_\infty \right) = \mu_a(\theta_a)$ almost surely as $n \rightarrow \infty$.
\end{proof}

\subsection{Proof of Proposition \ref{prop-asymptotic-behavior}} \label{prf-asymptotic-behavior}

\noindent \textbf{Asymptotic behavior of \textnormal{$\pi^\textsc{Irs.FH}$}.}
Let $\tilde{\omega}$ be the sampled outcome used by $\pi^\textsc{Irs.FH}$.
By Proposition \ref{prop-convergence}, we have $\lim_{n \rightarrow \infty} \hat{\mu}_{a,n}(\tilde{\omega}) = \mu_a(\tilde{\theta}_a)$ for almost all $\tilde{\omega}$.
This, together with the assumption that $\mu_i(\theta_i) \ne \mu_j(\theta_j)$ for $i \ne j$, since $\argmax_a \mu_a(\tilde{\theta}_a)$ is uniquely defined for almost all $\tilde{\omega}$, yields
\begin{equation}
	\argmax_a \mu_a(\tilde{\theta}_a)
	= \argmax_a \lim_{n \rightarrow \infty} \hat{\mu}_{a,n}(\tilde{\omega})
	= \lim_{n \rightarrow \infty} \argmax_a \hat{\mu}_{a,n}(\tilde{\omega})
	\quad \text{a.s.}
\end{equation}
Since almost-sure convergence guarantees convergence in distribution, for any $a \in \Ascr$,
\begin{align}
	\lim_{T \rightarrow \infty} \PR\left[ \newedit{A^\textsc{Irs.FH}(T, \mathbf{y})} = a \right]
		&= \lim_{T \rightarrow \infty} \PR\left[ \argmax_{a'} \hat{\mu}_{a',T-1}(\tilde{\omega}) = a \right]
		\\&= \PR\left[ \argmax_{a'} \mu_{a'}(\tilde{\theta}_{a'}) = a \right]
		\\&= \PR\left[ \newedit{A^\textsc{TS}(\mathbf{y})} = a \right].
\end{align}
Note that we are not assuming that $\pi^\textsc{Irs.FH}$ and $\pi^\textsc{TS}$ share the randomness.
The sampled parameters used in $\pi^\textsc{TS}$ are not necessarily the ones used in $\pi^\textsc{Irs.FH}$, but their distributions are identical since they are drawn from the same prior.
\qed

\noindent \textbf{Asymptotic behavior of \textnormal{$\pi^\textsc{Irs.V-Zero}$}.}
\newedit{To simplify notation, let $A_T^\circ \defeq A^\textsc{Irs.V-Zero}(T, \mathbf{y})$.
As above, it suffices to show that $\lim_{T \rightarrow \infty} A_T^\circ = \argmax_{a \in \Ascr} \mu_a(\tilde{\theta}_a) := A^\textsc{TS}$ for almost all sampled outcome $\tilde{\omega}$.
We hide $\tilde{\omega}$ and $\tilde{\theta}_a$ from the notation for the further simplication.
}

Define 
\begin{equation}
	\Delta \defeq \min_{a \ne A^\textsc{TS}} \left| \mu_{A^\textsc{TS}} - \mu_a \right|
	\quad \text{and} \quad
	M \defeq \sup_{a \in \Ascr, n \geq 0} \left| \hat{\mu}_{a,n} \right|.
\end{equation}
We have $0 < \Delta < 2M < \infty$ almost surely since $\mu_i(\tilde{\theta}_i) \ne \mu_j(\tilde{\theta}_j)$ for $i \ne j$ and $\lim_{n \rightarrow \infty} \hat{\mu}_{a,n} = \mu_a < \infty$ almost surely for all $a$.
In addition, there exists $N \in \N$ such that
\begin{equation}
	\left| \hat{\mu}_{a,n} - \mu_a \right| < \frac{\Delta}{4}
	, \quad \forall n \geq N, \quad \forall a \in \Ascr.
\end{equation}
For such $N$, we have
\begin{equation} \label{e-asymptotic-behavior-gap}
	\inf_{n \geq N} \hat{\mu}_{a^\textsc{TS},n}  \geq \sup_{n \geq N} \hat{\mu}_{a,n} +  \frac{\Delta}{2}, \quad \forall a \ne A^\textsc{TS}.
\end{equation}
Note that $A^\textsc{TS}$, $\Delta$, $M$, and $N$ \newedit{do not have the dependency on} $T$.

To argue by contradiction, suppose that $A_T^\circ \ne A^\textsc{TS}$ for some large $T$ such that $T \geq 2 N + \frac{ 8 M N }{ \Delta } + 2$.
Define the optimal allocation to the inner problem of \textsc{Irs.V-Zero} for such $T$:
\begin{equation}
	\mathbf{n}_{1:K}^\circ \defeq \argmax_{\mathbf{n}_{1:K} \in N_T}\left\{ \sum_{a=1}^K \sum_{s=1}^{n_a} \hat{\mu}_{a,s-1} \right\},
\end{equation}
where the ties are broken arbitrarily in $\argmax\{ \}$.
\newedit{We let $n^\circ(a)$ be the $a^\text{th}$ component of $\mathbf{n}_{1:K}^\circ$.
According to the specified arm selection rule, we have} $A_T^\circ = \argmax_a n^\circ(a)$ \newedit{and hence} $n^\circ(A_T^\circ) \geq \lfloor \frac{T}{2} \rfloor$ ($ > N$).
\newedit{We prove the claim for the following two cases:}

\noindent \textbf{Case 1:}
If $n^\circ(a^\textsc{TS}) \geq N$, consider \newedit{an allocation $\mathbf{n}_{1:K}^\dagger$ that is a deviation from the given optimal allocation $\mathbf{n}_{1:K}^\circ$ such that arm $a^\textsc{TS}$ gets one pull whereas arm $A_T^\circ$ gets one less pull}: i.e., $n^\dagger(A^\textsc{TS}) = n^\circ(A^\textsc{TS})+1$, $n^\dagger(A_T^\circ) = n^\circ(A_T^\circ)-1$, and $n^\dagger(a) = n^\circ(a)$ for any $a \notin \{ A^\textsc{TS}, A_T^\circ \}$.
\newedit{The change in the total payoff from this deviation is}
\begin{equation}
	\sum_{a=1}^K \sum_{i=1}^{ n^\dagger(a) } \hat{\mu}_{a,i-1} - \sum_{a=1}^K \sum_{i=1}^{ n^\circ(a) } \hat{\mu}_{a,i-1}
		= \hat{\mu}_{A^\textsc{TS},n^\circ(A^\textsc{TS})} - \hat{\mu}_{A_T^\circ,n^\circ(A_T^\circ)-1}
		\geq \frac{\Delta}{2}
		> 0,
\end{equation}
where the inequality follows from \eqref{e-asymptotic-behavior-gap} \newedit{and that $n^\circ(A^\textsc{TS}) \geq N$ and $n^\circ(A_T^\circ) \geq N$}.
The allocation $\mathbf{n}_{1:K}^\dagger$ \newedit{is strictly better than} $\mathbf{n}_{1:K}^\circ$, which contradicts the assumption that $\mathbf{n}_{1:K}^\circ$ is an optimal allocation.

\noindent \textbf{Case 2:}
If $n^\circ(A^\textsc{TS}) < N$, consider \newedit{an allocation $\mathbf{n}_{1:K}^\dagger$ that is a deviation from the given optimal allocation $\mathbf{n}_{1:K}^\circ$ such that arm $A_T^\circ$ gets no more than $N$ pulls whereas arm $A^\textsc{TS}$ gets the remains}: i.e.,
\begin{equation}
	n^\dagger(a) \defeq \left\{ \begin{array}{ll}
			n^\circ(A^\textsc{TS}) + (n^\circ(A_T^\circ)- N) & \text{ if } a = A^\textsc{TS}, \\
			N & \text{ if } a = A_T^\circ, \\
			n^\circ(a) & \text{ if } a \notin \{ A^\textsc{TS}, A_T^\circ \}.
		\end{array} \right.
\end{equation}
By making this \newedit{the deviation, the total payoff should increase by}
\begin{align}
	& \sum_{a=1}^K \sum_{i=1}^{ n^\dagger(a) } \hat{\mu}_{a,i-1} - \sum_{a=1}^K \sum_{i=1}^{ n^\circ(a) } \hat{\mu}_{a,i-1}
		\\&= \sum_{i=n^\circ(A^\textsc{TS}) +1 }^{ n^\circ(A^\textsc{TS}) + (n^\circ(A_T^\circ)- N) } \hat{\mu}_{A^\textsc{TS},i-1} 
			- \sum_{i=N+1 }^{ n^\circ(A_T^\circ) } \hat{\mu}_{A_T^\circ,i-1}
		\\&\geq - ( N - n^\circ(A^\textsc{TS}) ) \cdot 2M
			+\sum_{i=N +1 }^{ n^\circ(A_T^\circ) } \hat{\mu}_{A^\textsc{TS},i-1} 
			- \sum_{i=N+1 }^{ n^\circ(A_T^\circ) } \hat{\mu}_{A_T^\circ,i-1}
		\\&\geq - ( N - n^\circ(A^\textsc{TS}) ) \cdot 2M
			+ ( n^\circ(A_T^\circ) - N) \cdot \frac{\Delta}{2}
		\\&\geq  ( n^\circ(A_T^\circ) - N) \cdot \frac{\Delta}{2} - 2 N M.
\end{align}
Since $T \geq 2N + \frac{ 8 M N }{ \Delta } + 2$ and $n^\circ(A_T^\circ) \geq \lfloor \frac{T}{2} \rfloor$, the last term is strictly positive, which \newedit{is} a contradiction.

We've shown that for almost all $\tilde{\omega}$, when $T$ is large enough, the optimal allocation $\mathbf{n}_{1:K}^\circ$ must allocate more than a half of the pulls on arm $A^\textsc{TS} = \argmax_a \mu_a(\tilde{\theta}_a)$.
\newedit{This concludes the proof.}

\subsection{Proof of Theorem \ref{thm-monotonicity}} \label{prf-monotonicity}

\subsubsection{Proof of ``\textnormal{$W^\textsc{TS}(T, \mathbf{y}) \geq W^\textsc{Irs.FH}(T, \mathbf{y})$}''}
\begin{proof}
It immediately follows from Jensen's inequality: since $\max(\cdots)$ is a convex function,
\begin{equation}
	W^\textsc{TS}(T, \mathbf{y}) 
		= T \times \newedit{\E_\mathbf{y}}\left[ \max_a \mu_a(\theta_a) \right] 
		\geq T \times  \newedit{\E_\mathbf{y}}\left[ \max_a \E\left( \left. \mu_a(\theta_a) \right| \hat{\bm{\mu}}_{T-1} \right) \right]
		= W^\textsc{Irs.FH}(T, \mathbf{y}).
\end{equation}
\end{proof}

\subsubsection{Proof of ``\textnormal{$W^\textsc{Irs.FH}(T, \mathbf{y}) \geq W^\textsc{Irs.V-Zero}(T, \mathbf{y})$}''} \label{prf-monotonicity2}
\begin{lemma}[Variant of Jensen's inequality] \label{lem-jensen-variation}
	Suppose that $\varphi : \R \rightarrow \R$ is an \textbf{increasing} (deterministic) function. Then, for any real-valued random variable $X$ such that $\E|X| < \infty$,
	\begin{equation}
		\E\left[ \max\left\{ X + \varphi(X), 0 \right\} \right] \geq \E\left[ \max\left\{ \E(X) + \varphi(X), 0 \right\} \right].
	\end{equation}
\end{lemma}
\begin{proof}
{\small
Define $\mu \defeq \E(X)$ and $f_x(t) \defeq \max\{ t + \varphi(x), 0 \}$. Since $f_x(\cdot)$ is a convex function for each $x \in \R$,
\begin{equation}
	f_x(t) \geq f_x(\mu) + (t-\mu) \cdot f_x'(\mu) = \max\{ \mu + \varphi(x), 0 \} + (t-\mu) \cdot \mathbf{1}\{ \mu + \varphi(x) \geq 0 \}
	, \quad \forall t, \quad \forall x.
\end{equation}
By setting $t = x$, we get
\begin{equation} \label{e-jensen-variation-eq1}
	\max\{ x + \varphi(x), 0 \}  = f_x(x) \geq \max\{ \mu + \varphi(x), 0 \} + (x-\mu) \cdot \mathbf{1}\{ \mu + \varphi(x) \geq 0 \}
	, \quad \forall x.
\end{equation}
Note that, since $\mathbf{1}\{\mu + \varphi(x) \geq 0\}$ is increasing in $x$, (i) for any $x \geq \mu$, $(x - \mu) \geq 0$ and $\mathbf{1}\{\mu + \varphi(x)\} \geq \mathbf{1}\{\mu + \varphi(\mu)\}$, and (ii) for any $x < \mu$, $(x - \mu) < 0$ and $\mathbf{1}\{\mu + \varphi(x)\} \leq \mathbf{1}\{\mu + \varphi(\mu)\}$.
Therefore,
\begin{equation}
	(x - \mu) \cdot \mathbf{1}\{ \mu + \varphi(x) \geq 0 \} \geq (x - \mu) \cdot \mathbf{1}\{ \mu + \varphi(\mu) \geq 0 \}, \quad \forall x \in \R.
\end{equation}
Combining this with \eqref{e-jensen-variation-eq1}, we get
\begin{equation} \label{e-jensen-variation-eq2}
	\max\{ x + \varphi(x), 0 \} \geq \max\{ \mu + \varphi(x), 0 \} + (x - \mu) \cdot \mathbf{1}\{ \mu + \varphi(\mu) \geq 0 \}, \quad
		\forall x \in \R.
\end{equation}
For random variable $X$, by taking expectation, we get
\begin{align}
	\E\left[ \max\{ X + \varphi(X), 0 \} \right]
		&\geq \E\left[ \max\{ \mu + \varphi(X), 0 \} + (X - \mu) \cdot \mathbf{1}\{ \mu + \varphi(\mu) \geq 0 \} \right]
		\\&\geq \E\left[ \max\{ \mu + \varphi(X), 0 \} \right]  + \E( X - \mu ) \cdot \mathbf{1}\{ \mu + \varphi(\mu) \geq 0 \} 
		\\&= \E\left[ \max\{ \mu + \varphi(X), 0 \} \right].
\end{align}
}
\end{proof}

\begin{corollary} \label{cor-jensen-variation-1}
	On a probability space $(\Omega, \Fscr, \PR)$, let $\varphi(x, \omega):\R \times \Omega \rightarrow \R$ be a function such that (i) the mapping $x \mapsto \varphi(x, \omega)$ is \textbf{increasing} for each $\omega \in \Omega$ and (ii) for some sub-$\sigma$-field $\Hscr \subseteq \Fscr$, the mapping $\omega \mapsto \varphi(x, \omega)$ is \textbf{$\Hscr$-measurable} for each $x \in \R$ (i.e., $\varphi(\cdot, \omega)$ is a deterministic function conditioned on $\Hscr$).
 Then
	\begin{equation}
		\E\left[ \max\left\{ X(\omega) + \varphi(X(\omega), \omega), 0 \right\} \right] \geq \E\left[ \max\left\{ \E(X|\Hscr)(\omega) + \varphi(X(\omega), \omega), 0 \right\} \right].
	\end{equation}
\end{corollary}

\begin{proof}
{\small
Define
\begin{equation}
	\mu(\omega) \defeq \E(X | \Hscr)(\omega), \quad
	I(\omega) \defeq \mathbf{1}\{ \mu(\omega) + \varphi(\mu(\omega), \omega) \geq 0 \}.
\end{equation}
By \eqref{e-jensen-variation-eq2}, we have
\begin{equation}
	\max\{ x + \varphi(x,\omega), 0 \} \geq \max\{ \mu(\omega) + \varphi(x,\omega), 0 \} + (x-\mu(\omega) ) \cdot I(\omega), \quad \forall x \in \R, \quad
		\text{for each } \omega \in \Omega.
\end{equation}
Since $\mu(\omega)$ and $I(\omega)$ are $\Hscr$-measurable,
\begin{align}
	\E\left[ \max\{ X(\omega) + \varphi(X(\omega), \omega), 0 \} \right]
		&\geq \E\left[  \max\{ \mu(\omega) + \varphi(X(\omega), \omega), 0 \} + (X(\omega)-\mu(\omega)) \cdot I(\omega)  \right]
		\\&= \E\left[ \E\left( \left. \max\{ \mu(\omega) + \varphi(X(\omega), \omega), 0 \} +   (X(\omega)-\mu(\omega)) \cdot I(\omega) \right| \Hscr \right) \right]
		\\&= \E\left[ \max\{ \mu(\omega) + \varphi(X(\omega), \omega), 0 \} \right] + \E\left[ \E\left( \left.  (X(\omega)-\mu(\omega)) \cdot I(\omega) \right| \Hscr \right) \right]
		\\&= \E\left[ \max\{ \E\left( X | \Hscr \right)(\omega) + \varphi(X(\omega), \omega), 0 \} \right] 
		\\	& \quad + \E\left[ \underbrace{ \left(\E(X|\Hscr) (\omega) -\mu(\omega) \right) }_{=0} \cdot I(\omega) \right]
		\\&= \E\left[ \max\{ \E\left( X | \Hscr \right)(\omega) + \varphi(X(\omega), \omega), 0 \} \right] .
\end{align}
}
\end{proof}

\begin{corollary} \label{cor-jensen-variation-2}
	On a probability space $(\Omega, \Fscr, \PR)$, let $(C_0,\cdots,C_T)$ be \textbf{$\Hscr$-measurable real-valued random variables} for some sub-$\sigma$-field $\Hscr \subseteq \Fscr$ (i.e., $C_i$'s are constants conditioned on $\Hscr$).
	Then
	\begin{equation}
		\E\left[ \max_{0 \leq i \leq T}\left\{ (i-n)^+ \times X + C_i \right\} \right] 
			\geq 
			\E\left[ \max_{0 \leq i \leq T}\left\{ \E\left( \left.X \right| \Hscr \right) \cdot \mathbf{1}\{ i \geq n+1 \} +  (i-n-1)^+ \times X + C_i \right\}   \right]
	\end{equation}
	for any $n = 0,1,\cdots,T$.
\end{corollary}

\begin{proof}
{\small
When $n=T$, both sides become $\E\left[ \max_{0 \leq i \leq T}\left\{ C_i \right\} \right]$, which makes the claim true.
Fix $n < T$ and define
\begin{equation}
	\varphi(x, \omega) \triangleq \max_{n+1 \leq i \leq T} \left\{ (i-n-1) \times x + C_i(\omega) \right\} - \max_{ 0 \leq i \leq n } \left\{ C_i(\omega) \right\}.
\end{equation}
Note that $\varphi(x, \omega)$ satisfies the conditions in Corollary \ref{cor-jensen-variation-1}. By Corollary \ref{cor-jensen-variation-1},
\begin{align}
	& \E\left[ \max_{0 \leq i \leq T}\left\{ (i-n)^+ \times X + C_i \right\} \right] 
		\\&= \E\left[ \max\left\{ \max_{n+1 \leq i \leq T}\left\{ (i-n) \times X + C_i \right\} , \quad  \max_{0 \leq i \leq n} C_i \right\} \right] 
		\\&= \E\left[ \max\left\{ X+ \max_{n+1 \leq i \leq T}\left\{ (i-n-1) \times X + C_i \right\} , \quad  \max_{0 \leq i \leq n} C_i \right\} \right] 
		\\&= \E\left[ \max\left\{ X(\omega) + \underbrace{ \max_{n+1 \leq i \leq T}\left\{ (i-n-1) \times X(\omega) + C_i(\omega) \right\} - \max_{0 \leq i \leq n} C_i(\omega) }_{= \varphi(X(\omega), \omega) } , \quad  0 \right\} + \max_{0 \leq i \leq n} C_i(\omega) \right] 
		\\&\geq \E\left[ \max\left\{ \E\left( X | \Hscr \right)(\omega) +  \max_{n+1 \leq i \leq T}\left\{ (i-n-1) \times X(\omega) + C_i(\omega) \right\} - \max_{0 \leq i \leq n} C_i(\omega)  , \quad  0 \right\} + \max_{0 \leq i \leq n} C_i(\omega) \right] 
		\\&= \E\left[ \max\left\{ \max_{n+1 \leq i \leq T}\left\{ \E\left( X | \Hscr \right)  + (i-n-1) \times X + C_i \right\}  , \quad  \max_{0 \leq i \leq n} C_i  \right\} \right] 
		\\&= \E\left[ \max_{0 \leq i \leq T}\left\{ \E\left( \left. X \right| \Hscr \right) \cdot \mathbf{1}\{ i \geq n+1 \} +  (i-n-1)^+ \times X + C_i \right\}   \right].
\end{align}
}
\end{proof}

\noindent \textbf{Proof of ``\textnormal{$W^\textsc{Irs.FH}(T, \mathbf{y}) \geq W^\textsc{Irs.V-Zero}(T, \mathbf{y})$}.''}
{\small 
Define
\begin{equation}
	N_T \defeq \left\{ \mathbf{n}_{1:K} \in \N_0^K : \sum_{a=1}^K n_a = T \right\}
	\quad \text{and} \quad
	S_a(n_a) \defeq \sum_{i=1}^{n_a} \hat{\mu}_{a,i-1}.
\end{equation}
What we want to show is
\begin{align}
	W^\textsc{Irs.FH} \equiv \E\left[ T \times \max_a\{ \hat{\mu}_{a,T-1}  \} \right] 
		&= \E\left[ \max_{\mathbf{n}_{1:K} \in N_T}\left\{ \sum_{a=1}^K n_a \times \hat{\mu}_{a,T-1} \right\} \right] 
		\\&\geq \E\left[ \max_{\mathbf{n}_{1:K} \in N_T}\left\{ \sum_{a=1}^K S_a(n_a) \right\} \right] \equiv W^\textsc{Irs.V-Zero}.
\end{align}
Further define
\begin{equation}
	U_{k,n} \defeq \E\left[ \max_{\mathbf{n}_{1:K} \in N_T}\left\{ \left( \sum_{a=1}^{k-1} S_a(n_a) \right) + \left( S_k(n_k \wedge n) + (n_k-n)^+ \times \hat{\mu}_{a,T-1} \right) + \left( \sum_{a=k+1}^K n_a \times \hat{\mu}_{a,T-1} \right) \right\} \right],
\end{equation}
where $a \wedge b \defeq \min(a,b)$.
Observe that $W^\textsc{Irs.FH} = U_{1,0}$, $W^\textsc{Irs.V-Zero} = U_{K,T}$, and $U_{k+1,0} = U_{k,T}$.
Therefore, it suffices to show that
\begin{equation}
	U_{k,n} \geq U_{k,n+1}, \quad \forall k=1,\cdots,K, \quad \forall n=0,\cdots,T-1.
\end{equation}
Fix $k$ and $n$. Define a sub-$\sigma$-field
\begin{equation}
	\Hscr \defeq \sigma\left( \{ R_{a,s} \}_{a=k, 1 \leq s \leq n} \cup \{ R_{a,s} \}_{a \ne k, 1 \leq s \leq T-1} \right).
\end{equation}
For each $i=0,\cdots,T$, define
\begin{equation}
	C_i \defeq \max\left\{ \left( \sum_{a=1}^{k-1} S_a(n_a) \right) + S_k(i \wedge n) + \left( \sum_{a=k+1}^K n_a \times \hat{\mu}_{a,T-1} \right) ~:~ \mathbf{n}_{1:K} \in N_T, n_k = i \right\}.
\end{equation}
Note that $C_i$'s are $\Hscr$-measurable and
\begin{equation}
	U_{k,n} = \E\left[ \max_{0 \leq i \leq T}\left\{ (i-n)^+ \times \hat{\mu}_{k,T-1} + C_i \right\} \right].
\end{equation}
With $X \defeq \hat{\mu}_{a,T-1}$,
\begin{align}
	U_{k,n} &= \E\left[ \max_{0 \leq i \leq T}\left\{ (i-n)^+ \times X + C_i \right\} \right]
		\\ &\stackrel{\text{Corollary \ref{cor-jensen-variation-2}}}{\geq} \E\left[ \max_{0 \leq i \leq T}\left\{ \E\left( \left.X \right| \Hscr \right) \cdot \mathbf{1}\{ i \geq n+1 \} +  (i-n-1)^+ \times X + C_i \right\}   \right]
		\\ &\stackrel{\text{(a)}}{=} \E\left[ \max_{0 \leq i \leq T}\left\{ \hat{\mu}_{k,n} \cdot \mathbf{1}\{ i \geq n+1 \} +  (i-n-1)^+ \times \hat{\mu}_{a,T-1} + C_i \right\}   \right]
		\\ &\stackrel{\text{(b)}}{=} U_{k,n+1}.
\end{align}
Equation (a) holds since $\E\left( X | \Hscr \right) = \mathbb{E}\left( \hat{\mu}_{k,T-1} | \Hscr \right) = \mathbb{E}\left( \hat{\mu}_{k,T-1} | R_{k,1}, \cdots, R_{k,n} \right) = \hat{\mu}_{a,n} $,
and equation (b) holds since $S_k(i \wedge n) + \hat{\mu}_{k,n} \cdot \mathbf{1}\{ i \geq n+1 \} = \sum_{s=1}^n \hat{\mu}_{k,s-1} \cdot \mathbf{1}\{ i \geq s \} + \hat{\mu}_{k,n} \cdot \mathbf{1}\{ i \geq n+1 \} = \sum_{s=1}^{n+1} \hat{\mu}_{k,s-1} \cdot \mathbf{1}\{ i \geq s \} = S_k(i \wedge (n+1) )$.
\qed
\\
}


\noindent \textbf{A note on the proof.}
One may wonder if the above result can be derived in a simpler way by exploiting the properties of
nested filtration \citep[e.g., Proposition 2.3 of][]{Brown10}.
Unlike the proof of $W^\textsc{TS} \geq W^\textsc{Irs.FH}$, however, the proof of $W^\textsc{Irs.FH} \geq W^\textsc{Irs.V-Zero}$ does not simply follow from the fact that $\sigma(\hat{\bm{\mu}}_{T-1})$ is \newedit{larger} than $\sigma(H_{t-1})$.

\newedit{Consider a Bernoulli MAB with $K=2$, $T=2$, and a prior distribution $\text{Beta}(1,1)$, and let us introduce} its variation whose reward function is given by $r_t'(\cdot)$ as follows:
\begin{equation}
	r_1'(a_1) = r_1(a_1), \quad r_2'(\mathbf{a}_{1:2}) = - \kappa r_2(\mathbf{a}_{1:2}),
\end{equation}
where $r_t(\cdot)$ is the reward function of the original Bernoulli MAB.
When $\kappa > 0$, \newedit{one can show that}
\begin{align}
	W^\textsc{Irs.FH} &= \mathbb{E}\left[ \max_{\mathbf{a}_{1:T}} \left\{ \sum_{t=1}^T \mathbb{E}( r_t'(\mathbf{a}_{1:t}) | \hat{\bm{\mu}}_{T-1} ) \right\} \right] = \frac{7}{12} - \frac{5}{12}\kappa,
	\\
	W^\textsc{Irs.V-Zero} &= \mathbb{E}\left[ \max_{\mathbf{a}_{1:T}} \left\{ \sum_{t=1}^T \mathbb{E}( r_t'(\mathbf{a}_{1:t}) | H_{t-1} ) \right\} \right] = \frac{1}{2} - \frac{3}{8} \kappa.
\end{align}
If $\kappa$ is large enough, we obtain $W^\textsc{Irs.FH} < W^\textsc{Irs.V-Zero}$, which is opposite to the above result.

Recall that the additional gain from knowing the future information can be decomposed into two components; the gain from knowing the immediate reward and the gain from knowing the next belief state, where $\textsc{Irs.V-Zero}$ considers the former component only.
When those two \newedit{components are not aligned} as in this example (i.e., a higher $r_1'$ leads to a worse next belief state), \newedit{the DM can exploit the penalties if they penalize only for the first component (e.g., when $r_1'$ is smaller than expected, the DM will get compensated for this difference but she can still earn the larger reward in the next period)}.

\newedit{This is also related to the fact that $z_t^\textsc{Irs.V-Zero}$ does not correspond to zero penalty under the some (partial) information relaxation, but should be understood as an approximation of $z_t^\text{ideal}$ under the perfect information relaxation.
As opposed to \textsc{TS} and \textsc{Irs.FH}, the optimal solution to the \textsc{Irs.V-Zero}'s inner problem may depend on the entire outcome $\omega$.
With the terminology of \citet{Brown10}, there is a mismatch between the filtration that generates the penalties and the filtration that characterizes the relaxed policy space.}

\subsubsection{Proof of ``\textnormal{$W^\textsc{TS}(T,\mathbf{y}) \geq W^\textsc{Irs.V-EMax}(T, \mathbf{y})$}''}

To show that $W^\textsc{TS} \geq W^\textsc{Irs.V-EMax}$, we take a completely different approach that \newedit{utilizes} Theorem 4 in \cite{Desai12}.
\newedit{We here rephrase the definition and the theorem therein using our notation.}

\begin{definition}[Supersolution] \label{def-supersolution}
	An approximate value function $\widehat{V}: \N_0 \times \Yscr \rightarrow \R$ is a \textbf{supersolution} to the Bellman equation \eqref{e-bellman} if
	\begin{equation}
		\widehat{V}(T,\mathbf{y}) \geq \max_{a \in \Ascr}\left\{ \E_{y_a}\left[ R_{a,1} + \widehat{V}( T-1, \Uscr( \mathbf{y}, R_{a,1}, r ) ) \right] \right\}, \quad \forall \mathbf{y} \in \Yscr, \quad \forall T \geq 1,
	\end{equation}
	with $\widehat{V}(0, \mathbf{y}) =0$ for all $\mathbf{y} \in \Yscr$.
\end{definition}

\begin{remark} \label{rem-supersolution}
	If $\widehat{V}(\cdot, \cdot)$ is a supersolution, then for any given $\omega$, $T$, and $\mathbf{y}$,
	\begin{equation}
		\widehat{V}(T-t+1, \mathbf{y}_{t-1}(\mathbf{a}_{1:t-1}, \omega; \mathbf{y}) ) 
			\geq 	\newedit{\E_\mathbf{y}}\left[ \left. r_t(\mathbf{a}_{1:t-1} \oplus a, \omega; \mathbf{y}) + \widehat{V}(T-t, \mathbf{y}_t(\mathbf{a}_{1:t-1} \oplus a, \omega; \mathbf{y}))  \right| H_{t-1}(\mathbf{a}_{1:t-1},\omega) \right],
	\end{equation}
	for all $a \in \Ascr$, $\mathbf{a}_{1:t-1} \in \Ascr^{t-1}$ and $t \in \{1, \ldots, T\}$.
\end{remark}

\begin{lemma}[Theorem 4 of \citet{Desai12}, rephrased] \label{lem-bound-inequality-2}
Consider a penalty function $\hat{z}_t$ generated by $\widehat{V}(\cdot, \cdot)$:
\begin{align}
	\hat{z}_t( \mathbf{a}_{1:t}, \omega; T, \mathbf{y} ) &\defeq r_t( \mathbf{a}_{1:t}, \omega )  -  \newedit{\E_\mathbf{y}}\left[ r_t( \mathbf{a}_{1:t}, \omega ) \left| H_{t-1}( \mathbf{a}_{1:t-1}, \omega ) \right. \right]
		\\ &  \quad + \widehat{V}\left( T-t, \mathbf{y}_t( \mathbf{a}_{1:t}, \omega; \mathbf{y} ) \right) - \newedit{\E_\mathbf{y}}\left[ \left. \widehat{V}\left( T-t, \mathbf{y}_t( \mathbf{a}_{1:t}, \omega; \mathbf{y} ) \right)   \right| H_{t-1}( \mathbf{a}_{1:t-1}, \omega ) \right]. \nonumber
\end{align}
If $\widehat{V}(\cdot, \cdot)$ is a supersolution, \newedit{then the performance bound induced by penalty function $\hat{z}_t$ is tighter than $\widehat{V}$: i.e.,}
\begin{equation}
	W^{\hat{z}}(T, \mathbf{y}) \leq \widehat{V}(T, \mathbf{y}).
\end{equation}
\newedit{And this holds in a stronger sense: for each outcome $\omega$, the maximal value of the inner problem with respect to $\omega$ (denoted by $V^{\hat{z},\textup{in}}_1( \emptyset, \omega; T, \mathbf{y} )$ in the proof) is smaller than or equal to $\widehat{V}(T, \mathbf{y})$.}
\end{lemma}

\begin{proof}
Let $V^{\hat{z},\textup{in}}_t(\cdot)$ be the DP solution of inner problem \eqref{e-inner-problem} for a given penalty $\hat{z}_t$ with respect to a particular outcome $\omega$:
\begin{equation}
	V^{\hat{z},\textup{in}}_t(\mathbf{a}_{1:t-1}, \omega; T, \mathbf{y})
		= \max_{a \in \Ascr} \left\{ r_t(\mathbf{a}_{1:t-1} \oplus a, \omega) - \hat{z}_t(\mathbf{a}_{1:t-1} \oplus a, \omega; T, \mathbf{y}) + V^{\hat{z},\textup{in}}_{t+1}(\mathbf{a}_{1:t-1} \oplus a, \omega; T, \mathbf{y})  \right\},
\end{equation}
with $V^{\hat{z},\textup{in}}_{T+1}(\cdot, \omega; T, \mathbf{y}) = 0$.
Then, we have $W^{\hat{z}}(T, \mathbf{y}) = \E\left[ V^{\hat{z},\textup{in}}_1( \emptyset, \omega; T, \mathbf{y} ) \right]$.
To prove the claim, it suffices to show that, for any given $\omega$,
\begin{equation} \label{e-bound-inequality-claim}
	V^{\hat{z},\textup{in}}_t\left( \mathbf{a}_{1:t-1}, \omega; T, \mathbf{y}_{t-1}( \mathbf{a}_{1:t-1}, \omega; \mathbf{y} ) \right) 
		\leq 
		\widehat{V}\left( T-t+1, \mathbf{y}_{t-1}( \mathbf{a}_{1:t-1}, \omega; \mathbf{y}) \right),
\end{equation}
for all $\mathbf{a}_{1:t-1} \in \Ascr^{t-1}$ and for all $t=1,\cdots,T+1$.

We argue by induction.
As a terminal case, when $t=T+1$, the inequality \eqref{e-bound-inequality-claim} holds trivially since both sides are zero.
Fix $t$ and suppose that the inequality \eqref{e-bound-inequality-claim} holds for $t+1$.
Omitting $\omega$ for brevity, we get
\begin{align}
	& \widehat{V}\left( T-t+1, \mathbf{y}_{t-1}( \mathbf{a}_{1:t-1} ) \right) - V^{\hat{z},\textup{in}}_t\left( \mathbf{a}_{1:t-1}; T, \mathbf{y}_{t-1}( \mathbf{a}_{1:t-1} ) \right) 
	\\&= \widehat{V}\left( T-t+1, \mathbf{y}_{t-1}( \mathbf{a}_{1:t-1}) \right)
		- \max_{a \in \Ascr}\left\{ 
			r_t(\mathbf{a}_{1:t-1} \oplus a) - \hat{z}_t(\mathbf{a}_{1:t-1} \oplus a; T, \mathbf{y}) + V^{\hat{z},\textup{in}}_{t+1}(\mathbf{a}_{1:t-1} \oplus a; T, \mathbf{y})
		 \right\}
	\\&= \min_{a \in \Ascr}\left\{ \begin{array}{l}
		\underbrace{ \widehat{V}\left( T-t, \mathbf{y}_t( \mathbf{a}_{1:t} ) \right) - V^{\hat{z},\textup{in}}_{t+1}(\mathbf{a}_{1:t-1} \oplus a; T, \mathbf{y}) }_{\geq 0 ~~ (\because ~\text{induction hypothesis})}
		\\
		+ \underbrace{ \widehat{V}\left( T-t+1, \mathbf{y}_{t-1}( \mathbf{a}_{1:t-1} ) \right)
		 - \E\left[ \left. r_t( \mathbf{a}_{1:t-1} \oplus a ) + \widehat{V}\left( T-t, \mathbf{y}_t( \mathbf{a}_{1:t-1} \oplus a ) \right)   \right| H_{t-1} \right] }_{\geq 0 ~~ (\because ~\text{Remark \ref{rem-supersolution}})}
		\end{array} \right\} 
	\\&\geq 0.
\end{align}
\end{proof}

\noindent \textbf{Proof of ``\textnormal{$W^\textsc{TS}(T,\mathbf{y}) \geq W^\textsc{Irs.V-EMax}(T, \mathbf{y})$}.''}
Recall that $z_t^\textsc{Irs.V-EMax}$ is a penalty function generated by $W^\textsc{TS}$.
We observe that $W^\textsc{TS}(\cdot, \cdot)$ is a supersolution: for any $T$ and $\mathbf{y}$,
\begin{align}
	W^\textsc{TS}(T, \mathbf{y})
		&= \newedit{\E_\mathbf{y}}\left[ T \times \max_{a \in \Ascr} \mu_a(\theta_a) \right]
		\\&= \newedit{\E_\mathbf{y}}\left[  \max_{a \in \Ascr} \mu_a(\theta_a) \right] + W^\textsc{TS}(T-1,\mathbf{y})
		\\&\geq \max_{a \in \Ascr}\left\{ \newedit{\E_{y_a}}\left[ \mu_a(\theta_a) \right] + W^\textsc{TS}(T-1,\mathbf{y}) \right\}
		\\&= \max_{a \in \Ascr}\left\{ \newedit{\E_\mathbf{y}}\left[ R_{a,1} + W^\textsc{TS}(T-1,\mathbf{y})  \right] \right\}
		\\&= \max_{a \in \Ascr}\left\{ \newedit{\E_\mathbf{y}}\left[ R_{a,1} + W^\textsc{TS}(T-1,\Uscr( \mathbf{y}, a, R_{a,1}) )  \right] \right\}.
\end{align}
The last equality holds since $\E\left[ W^\textsc{TS}\left( T-1,\Uscr( \mathbf{y}, a_1, r_1(a_1,\omega) ) \right) \right] = W^\textsc{TS}(T-1,\mathbf{y})$, as argued in \eqref{e-expected-full-info}.
By Lemma \ref{lem-bound-inequality-2}, we have $W^\textsc{Irs.V-EMax}(T, \mathbf{y}) \leq W^\textsc{TS}(T, \mathbf{y})$ \newedit{which also holds in a stronger sense}.
\qed

\subsection{Proof of Theorem \ref{thm-suboptimality}} \label{prf-suboptimality}

\newedit{\subsubsection{Suboptimality Decomposition}} \label{ss-irs-properties-ext}

As in \S \ref{prf-strong-duality}, we define the Q-values of the inner problem given a particular outcome $\omega$, a penalty function $z_t(\cdot)$, a time horizon $T$, and a prior belief $\mathbf{y}$.
\begin{align}
	 Q^{z,\textup{in}}_t(\mathbf{a}_{1:t-1}, a, \omega; T, \mathbf{y}) 
		&= r_t(\mathbf{a}_{1:t-1} \oplus a, \omega) - z_t(\mathbf{a}_{1:t-1} \oplus a, \omega; T, \mathbf{y}) 
		\\ &  \quad + V^{z,\textup{in}}_{t+1}(\mathbf{a}_{1:t-1} \oplus a, \omega; T, \mathbf{y}), \nonumber
	\\ V^{z,\textup{in}}_t(\mathbf{a}_{1:t-1}, \omega; T, \mathbf{y})
		&= \max_{a \in \Ascr} \left\{ Q^{z,\textup{in}}_t(\mathbf{a}_{1:t-1}, a, \omega; T, \mathbf{y}) \right\},
\end{align}
with $V_{T+1}^{z,\textup{in}}( \cdot, \omega; T, \mathbf{y} ) \equiv 0$.
Additionally define the total payoff of an action sequence and the hindsight best action under penalties:
\begin{align}
	\Sscr^z( \mathbf{a}_{1:T}, \omega; T, \mathbf{y} ) 
		&\defeq \sum_{t=1}^T r_t(\mathbf{a}_{1:t}, \omega ) - z_t( \mathbf{a}_{1:t}, \omega; T, \mathbf{y} ),
	\\
	a_t^{z,*}( \mathbf{a}_{1:t-1}, \omega; T, \mathbf{y} ) 
		&\defeq \argmax_{a \in \Ascr} \left\{ Q^{z,\textup{in}}_t(\mathbf{a}_{1:t-1}, a, \omega; T, \mathbf{y}) \right\}.
	\label{e-best-action}
\end{align}
We have $V_1^{z,\textup{in}}(\emptyset, \omega; T, \mathbf{y}) = \max_{ \mathbf{a}_{1:T} \in \Ascr^T } \Sscr^z( \mathbf{a}_{1:T}, \omega; T, \mathbf{y} )$.

\begin{proposition}[Suboptimality decomposition]  \label{prop-suboptimality-decomposition}
Given a non-anticipating policy $\pi \in \Pi_\mathbb{F}$ and a dual-feasible penalty function $z_t$, \newedit{the suboptimality gap is the sum of the instantaneous suboptimalities of individual actions taken by $\pi$ along the sample path: i.e., }
\begin{align}
	W^z(T, \mathbf{y}) - V(\pi, T, \mathbf{y}) 
		&= \newedit{\E_\mathbf{y}}\left[ \max_{\mathbf{a}_{1:T}} \left\{ \Sscr^z(\mathbf{a}_{1:T}, \omega; T, \mathbf{y}) \right\} - \Sscr^z(\mathbf{A}_{1:T}^\pi,  \omega; T, \mathbf{y}) \right]
		\\&= \newedit{\E_\mathbf{y}}\left[  \sum_{t=1}^T \max_a \left\{ Q_t^{z,\textup{in}}(\mathbf{A}_{1:t-1}^\pi, a, \omega; T, \mathbf{y}) \right\} - Q_t^{z,\textup{in}}(\mathbf{A}_{1:t-1}^\pi, A_t^\pi, \omega; T, \mathbf{y}) \right],
		\label{e-suboptimality-decomposition}
\end{align}
where the expectation is taken with respect to the randomness of outcome $\omega$ and the randomness of policy $\pi$.

\end{proposition}

\begin{proof}
The first equality immediately follows from the definition of $W^z$ and mean equivalence (Proposition \ref{prop-mean-equivalence}).
Now fix $\omega$, $T$, and $\mathbf{y}$.
Consider the (pathwise) suboptimality of the action sequence $\mathbf{A}_{1:T}^\pi$ compared to the clairvoyant optimal solution.
It can be decomposed into the instantaneous suboptimalty incurred by the individual action at each time:
\begin{equation} \label{e-pathwise-suboptimality}
	\max_{\mathbf{a}_{1:T}} \left\{ \Sscr^z(\mathbf{a}_{1:T}) \right\} - \Sscr^z(\mathbf{A}_{1:T}^\pi)
		= \sum_{t=1}^T \max_a \left\{ Q_t^{z,\textup{in}}(\mathbf{A}_{1:t-1}^\pi, a) \right\} - Q_t^{z,\textup{in}}(\mathbf{A}_{1:t-1}^\pi, A_t^\pi).
\end{equation}
By taking expectation, we obtain the second equality.
\end{proof}

\newedit{
The next lemma shows that the instantaneous suboptimalty of the first action can be expressed in terms of mean reward metrics for each of the IRS penalty functions.

\begin{lemma} \label{lem-suboptimality-rephrase}
Fix time horizon $T$, prior belief $\mathbf{y}$, and the true outcome $\omega$, and hide the dependency on them in notation for $Q_1^{z,\textup{in}}(\cdot)$, $a_1^{z,*}(\cdot)$, $\mu_a(\cdot)$ and $\hat{\mu}_{a,n}(\cdot)$.
For each of the penalty functions $z^\textsc{TS}$, $z^\textsc{Irs.FH}$, and $z^\textsc{Irs.V-Zero}$, the instantaneous suboptimalty of action $a \in \Ascr$ satisfies the following: \\
(1) When $z \equiv z^\textsc{TS}$,
\begin{equation}
	Q_1^{z,\textup{in}}(a_1^{z,*}) - Q_1^{z,\textup{in}}(a) 
		= \mu_{a_1^{z,*}} - \mu_a.
\end{equation}
(2) When $z \equiv z^\textsc{Irs.FH}$,
\begin{equation}
	Q_1^{z,\textup{in}}(a_1^{z,*}) - Q_1^{z,\textup{in}}(a) 
		= \hat{\mu}_{a_1^{z,*}, T-1} - \hat{\mu}_{a, T-1}.
\end{equation}
(3) When $z \equiv z^\textsc{V-Zero}$,
\begin{equation}
	Q_1^{z,\textup{in}}(a_1^{z,*}) - Q_1^{z,\textup{in}}(a) 
		\leq \max_{0 \leq n \leq T-1} \left\{  \hat{\mu}_{a_1^{z,*},n}  \right\} - \hat{\mu}_{a,0}.
\end{equation}
\end{lemma}
}

\newedit{
\begin{proof}
\noindent (1) When $z \equiv z^\textsc{TS}$, we have
\begin{equation}
	Q_1^{z,\textup{in}}(a) = \mu_a + (T-1) \times \max_{a'} \mu_{a'}.
\end{equation}
Since the last term does not depend on action $a$, the claim follows.

\noindent (2) When $z \equiv z^\textsc{Irs.FH}$, we obtain the claim by replacing $\mu_a$ with $\hat{\mu}_{a,T-1}$ in the above proof.

\noindent (3) When $z \equiv z^\textsc{Irs.V-Zero}$, recall that the associated inner problem is to find an optimal allocation: i.e.,
\begin{equation}
	\max_{\mathbf{n}_{1:K} \in N_T}\left\{ \sum_{a=1}^K \sum_{i=0}^{n_a-1} \hat{\mu}_{a,i} \right\}.
\end{equation}
Let $\mathbf{n}_{1:K}^*$ be the optimal allocation.
Observe that the suboptimality is incurred only when $n_a^* = 0$, it is no worse than $\hat{\mu}_{a^*, n^*_{a^*}} - \hat{\mu}_{a,0}$ (the loss if the payoff when pulling $a$ one more time but pulling $a_1^{z,*}$ one less time).
Since $n^*_{a^*} \leq T-1$, the claim follows.
\end{proof}
}

\newedit{\subsubsection{Recursive Structure of IRS Penalty Functions}}

\newedit{To describe the recursive structure of Bayesian MAB problems explicitly, we} define a shift operator $\Mscr_t: \Ascr^t \times \Omega \rightarrow \Omega$,
\begin{equation}
	\Mscr_t( \mathbf{a}_{1:t}, \omega ) \defeq \left( R_{a,n_a}; \forall n_a > n_t(\mathbf{a}_{1:t},a), \forall a \in \Ascr  \right).
\end{equation}
The shifted outcome $\Mscr_{t-1}( \mathbf{a}_{1:t-1}, \omega )$ encodes the remaining reward realizations after taking $\mathbf{a}_{1:t-1}$.

\begin{remark}[Recursive structure of remaining uncertainties] \label{rem-recursive-uncertainty}
	Conditioned on $\Hscr_{t-1}( \mathbf{a}_{1:t-1}, \omega )$, the remaining uncertainties are sufficiently described by $\mathbf{y}_{t-1}( \mathbf{a}_{1:t-1}, \omega; \mathbf{y} )$, i.e.,
\begin{equation}
	\left. \Mscr_{t-1}( \mathbf{a}_{1:t-1}, \omega ) \right| H_{t-1}( \mathbf{a}_{1:t-1}, \omega ) 
		\quad \sim \quad \Iscr( \mathbf{y}_{t-1}(\mathbf{a}_{1:t-1}, \omega; \mathbf{y}) ).
\end{equation}
\end{remark}

\begin{remark}[Recursive structure of IRS penalties] \label{rem-recursive-penalty}
	Each of penalty functions \eqref{e-ideal-penalty}--\eqref{e-penalty-irs-vemax} has the following form:
	\begin{equation}
	z_t( \mathbf{a}_{1:t}, \omega; T, \mathbf{y} ) = \varphi^z( \Mscr_{t-1}( \mathbf{a}_{1:t-1}, \omega ), T-t+1, \mathbf{y}_{t-1}(\mathbf{a}_{1:t-1},\omega; \mathbf{y}) ),
\end{equation}
for some function $\varphi^z : \Omega \times \N \times \Yscr \rightarrow \R$, i.e., the penalty at each time is completely determined by the remaining rewards $\Mscr_{t-1}( \mathbf{a}_{1:t-1}, \omega )$, the remaining time horizon $T-t+1$, and the prior belief $\mathbf{y}_{t-1}(\mathbf{a}_{1:t-1}, \omega)$ at that moment.
\end{remark}

Remark \ref{rem-recursive-uncertainty} immediately follows from Bayes' rule, and Remark \ref{rem-recursive-penalty} can be easily verified.
We observe the recursive structure of the sequential inner problems that the DM solves throughout the decision-making process, which can be characterized by the following property.

\begin{proposition}[Generalized posterior sampling] \label{prop-generalized-posterior-sampling}
For each of penalty functions \eqref{e-ideal-penalty}--\eqref{e-penalty-irs-vemax}, the IRS policy $\pi$ is randomized in such a way that it takes an action $a$ with the probability that the action $a$ is indeed the best action $a_t^{z,*}$ at that moment, i.e.,
\begin{equation}
	\PR\left[ \left. A_t^\pi = a \right| \Fscr_{t-1} \right] = \PR\left[ \left. \newedit{ a_t^{z,*}( \mathbf{A}_{1:t-1}^\pi, \omega ) } = a \right| \Fscr_{t-1} \right]
		, \quad \forall a, \quad \forall t.
\end{equation}
The source of uncertainty in the LHS is the randomness of the policy (embedded in $\tilde{\omega}$) and that in the RHS is the randomness of nature (embedded in $\omega$).\deledit{
We let $a_t^{z,*}$ abbreviate $a_t^{z,*}( \mathbf{a}_{1:t-1}^{\pi^z}, \omega; T, \mathbf{y} )$ as defined in \eqref{e-best-action} and $\Fscr_{t-1}$ abbreviate $\Fscr_{t-1}(\mathbf{a}_{1:t-1}^{\pi^z}, \omega; T, \mathbf{y} )$.}
Here we assume that the tie-breaking rule in $\argmax$ of \eqref{e-best-action} is identical to the one used when $\pi^z$ solves the inner problem.
\end{proposition}

\begin{proof}
\newedit{
	Observe that the IRS's action $A_t^\pi$ can be represented as
	\begin{equation}
		A_t^\pi = a_1^{z,*}\left( \emptyset, \tilde{\omega}; T-t+1, \mathbf{y}_{t-1}( \mathbf{A}_{1:t-1}^\pi, \omega; \mathbf{y} ) \right),
	\end{equation}
	where $\tilde{\omega} \sim \Iscr( \mathbf{y}_{t-1}( \mathbf{A}_{1:t-1}^\pi, \omega; \mathbf{y} ) )$, i.e., the action that the clairvoyant DM will take in an MAB instance specified by horizon $T-t+1$, prior belief $\mathbf{y}_{t-1}( \mathbf{A}_{1:t-1}^\pi, \omega; \mathbf{y} )$, and the outcome $\tilde{\omega}$.
	Therefore, it suffices to verify that the inner problem that $\pi$ solves at time $t$ is identically distributed with the sub-inner problem with respect to ground-truth $\omega$ (i.e., the subproblem given the past action sequence $\mathbf{A}_{1:t-1}^\pi$).
	
Fix time $t$, past actions $\mathbf{a}_{1:t-1} = \mathbf{A}_{1:t-1}^\pi$, and the true outcome $\omega$.
The sub-inner problem determining $a_t^{z,*}( \mathbf{a}_{1:t-1}, \omega )$ is
\begin{equation} \label{e-generalized-posterior-sampling-subinner}
	\max_{\mathbf{a}'_{t:T}} \left\{ \sum_{s=t}^T r_s( \mathbf{a}_{1:t-1} \oplus \mathbf{a}'_{t:s}, \omega ) - z_s( \mathbf{a}_{1:t-1} \oplus \mathbf{a}'_{t:s}, \omega; T, \mathbf{y} ) \right\}.
\end{equation}
}
By Remark \ref{rem-recursive-penalty}, for any $s \in \{t, \ldots, T\}$, the penalty at (inner) time $s$ is given by
\begin{align}
	& z_s( \mathbf{a}_{1:t-1} \oplus \mathbf{a}'_{t:s}, \omega; T, \mathbf{y} )
		\\&= \varphi^z( \Mscr_{s-1}( \mathbf{a}_{1:t-1} \oplus \mathbf{a}'_{t:s-1}, \omega ), T-s+1, \mathbf{y}_{s-1}( \mathbf{a}_{1:t-1} \oplus \mathbf{a}'_{t:s-1}, \omega; \mathbf{y} ) )
		\\&= \varphi^z\left( \begin{array}{l}
			\Mscr_{s-t}( \mathbf{a}'_{t:s-1}, \Mscr_{t-1}( \mathbf{a}_{1:t-1}, \omega ) ), \\
			(T-t+1) - (s-t), \\
			\mathbf{y}_{s-t}( \mathbf{a}'_{t:s-1}, \Mscr_{t-1}( \mathbf{a}_{1:t-1}, \omega ); \mathbf{y}_{t-1}( \mathbf{a}_{1:t-1}, \omega; \mathbf{y} )
			\end{array} \right)
		\\&= z_{s-t+1}( \mathbf{a}'_{t:s}, \Mscr_{t-1}( \mathbf{a}_{1:t-1}, \omega ); T-t+1, \mathbf{y}_{t-1}( \mathbf{a}_{1:t-1}, \omega; \mathbf{y} ) ).
\end{align}
For rewards, similarly, we have $r_s( \mathbf{a}_{1:t-1} \oplus \mathbf{a}'_{t:s}, \omega ) = r_{s-t+1}( \mathbf{a}'_{t:s}, \Mscr_{t-1}( \mathbf{a}_{1:t-1}, \omega ) )$.
Therefore, the sub-inner problem \eqref{e-generalized-posterior-sampling-subinner} is reformulated as
\newedit{
\begin{equation}
	\max_{\mathbf{a}'_{t:T}} \left\{ \sum_{s=t}^T r_{s-t+1}( \mathbf{a}'_{t:s}, \Mscr_{t-1}( \mathbf{a}_{1:t-1}, \omega ) ) - z_{s-t+1}( \mathbf{a}'_{t:s}, \Mscr_{t-1}( \mathbf{a}_{1:t-1}, \omega ); T-t+1, \mathbf{y}_{t-1}( \mathbf{a}_{1:t-1}, \omega; \mathbf{y} ) ) \right\}.
\end{equation}
Given the fact that the shifted outcome $\Mscr_{t-1}(\mathbf{a}_{1:t-1}, \omega)$ and the sampled outcome $\tilde{\omega}$ are identically distributed with $\Iscr(\mathbf{y}_{t-1}(\mathbf{a}_{1:t-1}, \omega; \mathbf{y}))$ conditionally on $H_{t-1}(\mathbf{a}_{1:t-1}, \omega)$ (Remark \ref{rem-recursive-uncertainty}), this sub-inner problem follows the same distribution with
\begin{equation}
	\max_{\mathbf{a}'_{1:T-t+1}} \left\{ \sum_{s=1}^{T-t+1} r_s( \mathbf{a}'_{1:s}, \tilde{\omega} ) - z_s( \mathbf{a}'_{1:s}, \tilde{\omega}, T-t+1, \mathbf{y}_{t-1}( \mathbf{a}_{1:t-1}, \omega; \mathbf{y} ) ) \right\},
\end{equation}
which characterizes the IRS's action $A_t^\pi$.
Therefore, $a_t^{z,*}(\mathbf{A}_{1:t-1}^\pi, \omega)$ is identically distributed with $A_t^\pi$ conditioned on $\Fscr_{t-1}$.
}
\end{proof}

\newedit{
\begin{remark} \label{rem-suboptimality-rephrase}
Utilizing the recursive structure of IRS penalty functions, Lemma \ref{lem-suboptimality-rephrase} can be extended to describe the instantaneous suboptimality of the $t^\text{th}$ action.
Fix true outcome $\omega$ and past actions $\mathbf{a}_{1:t-1}$, and hide the dependency on them in notation for $Q_t^{z,\textup{in}}(\cdot)$, $a_t^{z,*}(\cdot)$, $n_t(\cdot)$, $\mu_a(\cdot)$ and $\hat{\mu}_{a,n}(\cdot)$.
\\
(1) When $z \equiv z^\textsc{TS}$,
\begin{equation}
	Q_t^{z,\textup{in}}(a_t^{z,*}) - Q_t^{z,\textup{in}}(a) 
		= \mu_{a_t^{z,*}} - \mu_a.
\end{equation}
(2) When $z \equiv z^\textsc{Irs.FH}$,
\begin{equation}
	Q_t^{z,\textup{in}}(a_t^{z,*}) - Q_t^{z,\textup{in}}(a) 
		= \hat{\mu}_{a_t^{z,*}, n_{t-1}(a_t^{z,*})+ T-t} - \hat{\mu}_{a, n_{t-1}(a)+T-t}.
\end{equation}
(3) When $z \equiv z^\textsc{V-Zero}$,
\begin{equation}
	Q_t^{z,\textup{in}}(a_t^{z,*}) - Q_t^{z,\textup{in}}(a) 
		\leq \max_{0 \leq n \leq T-t} \left\{  \hat{\mu}_{a_1^{z,*},n_{t-1}(a_1^{z,*}) + n}  \right\} - \hat{\mu}_{a, n_{t-1}(a)}.
\end{equation}
\end{remark}
}

\newedit{\subsubsection{Preliminary Lemmas on MAB with Natural Exponential Family Distributions}}

\newedit{
We first describe the notion of sub-Gaussian random variable as an effective tool for bounding its tail behavior.

\begin{definition}[Sub-Gaussian random variable]
	A random variable $X$ is $\sigma$-sub-Gaussian if
	\begin{equation}
		\E\left[ \exp\left( \lambda (X - \E X) \right) \right] \leq \exp\left( \frac{\sigma \lambda^2}{2} \right), \quad \forall \lambda \in \R,
	\end{equation}
	for some $\sigma > 0$.
\end{definition}
}

\newedit{
\begin{lemma} \label{lem-subgaussian-tail-bound}
	Given a random variable $X$, suppose that there exists $\sigma > 0$ such that
	\begin{equation}
		\PR\left[ X \geq \E X + z \sigma \right] \leq e^{-z^2/2}, \quad \forall z \geq 0.
	\end{equation}
	Then, the following holds:
	\begin{equation} \label{e-subgaussian-tail-bound}
		\E\left[ \left( X - (\E X + z \sigma) \right)^+ \right] \leq \frac{\sigma}{z} e^{-z^2/2}, \quad \forall z > 0.
	\end{equation}
\end{lemma}
\begin{corollary}
	If a random variable $X$ is $\sigma$-sub-Gaussian, it satisfies the condition of Lemma \ref{lem-subgaussian-tail-bound} and hence the inequality \eqref{e-subgaussian-tail-bound} holds.
\end{corollary}
}

\newedit{
\begin{proof}
	With $\mu \defeq \E X$, we have
	\begin{equation}
		\E\left[ \left( X - (\mu + z \sigma) \right)^+ \right] 
			= \int_{x = \mu + z \sigma}^\infty \PR\left[ X \geq x \right] dx
			= \int_{t=z}^\infty \PR\left[ X \geq \mu + t\sigma \right] \sigma dt
			\leq \sigma \int_{t=z}^\infty e^{-t^2/2} dt.
	\end{equation}
	Utilizing the tail bound established for the standard normal distribution, we can show that
	\begin{equation}
		\int_{t=z}^\infty \frac{1}{\sqrt{2\pi}} e^{-t^2/2} dt
			\leq \frac{1}{z} \frac{e^{-z^2/2}}{\sqrt{2\pi}}.
	\end{equation}
	By combining these two inequalities, we obtain the desired result.
	
	The corollary simply follows from Markov inequality: for any $z \geq 0$ and $\lambda \geq 0$, we have
	\begin{equation}
		\PR[ X \geq \mu + z \sigma ] 
			= \PR\left[ e^{\lambda(X-\mu)} \geq e^{\lambda z \sigma} \right]
			\leq \frac{ \E[ e^{\lambda(X-\mu)} ] }{ e^{\lambda x \sigma} }
			\leq \exp\left( \frac{\sigma^2 \lambda^2}{2} - \lambda z \sigma \right).
	\end{equation}
	By taking $\lambda = \frac{z}{\sigma}$, it follows that $\PR[ X \geq \mu + z \sigma ] \leq e^{-z^2/2}$.
\end{proof}
}

\newedit{
We now return to the context of MAB problems and show that the mean reward metrics are sub-Gaussian.
}

\newedit{
\begin{lemma}[Sub-Gaussianity of mean reward metrics] \label{lem-mean-reward-subgaussian}
	Consider the setting of Theorem \ref{thm-suboptimality}, i.e., the reward distribution of arm $a$ is described by an $L$-smooth log-partition function $A_a(\theta_a)$ and hyper-parameters $(\xi_a, \nu)$. Then, the conditional mean reward $\mu_a$ is $\sqrt{L/\nu}$-sub-Gaussian: i.e.,
	\begin{equation}
		\E_{(\xi_a, \nu)}\left[ \exp\left( \lambda ( \mu_a - \bar{\mu}_a ) \right) \right] \leq \exp\left( \frac{L \lambda^2}{2 \nu} \right), \quad \forall \lambda \in \R,
	\end{equation}
	where $\bar{\mu}_a = \E_{(\xi_a, \nu)}[ \mu_a ] = \frac{\xi_a}{\nu}$ is the prior predictive mean reward (i.e., the unconditional mean reward).
	Furthermore, the posterior predictive mean reward $\hat{\mu}_{a,n}$ is $\sqrt{ \frac{Ln}{\nu(\nu+n)} }$-sub-Gaussian: i.e.,
	\begin{equation} \label{e-predictivie-mean-subgaussian}
		\E_{(\xi_a, \nu)}\left[ \exp\left( \lambda ( \hat{\mu}_{a,n} - \bar{\mu}_a ) \right) \right] \leq \exp\left( \frac{\lambda^2}{2} \times \frac{Ln}{\nu(\nu+n)} \right), \quad \forall \lambda \in \R.
	\end{equation}
\end{lemma}
}

\newedit{
\begin{proof}
	We first prove that $\mu_a$ is $\sqrt{L/\nu}$-sub-Gaussian.
	Due to $L$-smoothness condition, $A_a(\theta_a)$ is finite valued for all $\theta_a \in \R$.
	For any $\lambda \in \R$, we have
	\begin{align}
		\E_{(\xi_a, \nu)}\left[ \exp\left( \lambda \mu_a \right) \right]
			&\stackrel{(i)}{=} \E_{(\xi_a, \nu)}\left[ \exp\left( \lambda A_a'(\theta_a) \right) \right]
			\\&= \int_{-\infty}^\infty \exp\left( \lambda A_a'(\theta_a) \right) \times f_a( \xi_a, \nu ) \exp\left( \xi_a \theta_a - \nu A_a(\theta_a) \right) d\theta_a
			\\&= \int_{-\infty}^\infty f_a( \xi_a, \nu ) \exp\left\{ \xi_a \theta_a - \nu A_a(\theta_a) + \lambda A_a'(\theta_a) \right\} d\theta_a 
			\\&= \int_{-\infty}^\infty f_a( \xi_a, \nu ) \exp\left\{ \xi_a \theta_a - \nu \left( A_a(\theta_a) - \lambda/\nu \cdot A_a'(\theta_a) \right) \right\} d\theta_a 
			\\&\stackrel{(ii)}{\leq} \int_{-\infty}^\infty f_a( \xi_a, \nu ) \exp\left\{ \xi_a \theta_a - \nu \left( A_a(\theta_a - \lambda/\nu) - \frac{L\lambda^2}{2 \nu^2} \right) \right\} d\theta_a 
			\\&= \exp\left( \frac{L \lambda^2}{2\nu} \right) \times \int_{-\infty}^\infty f_a( \xi_a, \nu ) \exp\left\{ \xi_a \theta_a - \nu A_a(\theta_a - \lambda/\nu) \right\} d\theta_a
			\\&= \exp\left( \frac{\xi_a \lambda}{\nu} + \frac{L \lambda^2}{2\nu} \right) \times \int_{-\infty}^\infty f_a( \xi_a, \nu ) \exp\left\{ \xi_a (\theta_a - \lambda/\nu) - \nu A_a(\theta_a - \lambda/\nu) \right\} d\theta_a
			\\&= \exp\left( \frac{\xi_a \lambda}{\nu} + \frac{L \lambda^2}{2\nu} \right) \times \int_{-\infty}^\infty f_a( \xi_a, \nu ) \exp\left\{ \xi_a \theta_a - \nu A_a(\theta_a) \right\} d\theta_a
			\\&= \exp\left( \frac{\xi_a \lambda}{\nu} + \frac{L \lambda^2}{2\nu} \right),
	\end{align}
	where we have utilized that (i) $\mu_a(\theta_a) = A_a'(\theta_a)$ and (ii) $A_a(\theta_a + \delta) \leq A_a(\theta_a) + \delta A_a'(\theta_a) + \frac{L}{2} \delta^2$.
	Since $\bar{\mu}_a = \xi_a/\nu$, we obtained the desired result.

	Next we focus on the posterior predictive mean reward $\hat{\mu}_{a,n}$.
	Recall that we have
	\begin{equation}
		\hat{\mu}_{a,n} = \frac{\xi_a + \sum_{i=1}^n R_{a,i} }{ \nu + n }.
	\end{equation}
	For any $\lambda \in \R$, we have
	\begin{align}
		\E_{(\xi_a,\nu)}\left[ \exp\left( \lambda \sum_{i=1}^n R_{a,i} \right) \right]
			&= \E_{(\xi_a,\nu)}\left[ \E\left\{ \left. \exp\left( \lambda \sum_{i=1}^n R_{a,i} \right) \right| \theta_a \right\} \right]
			\\&\stackrel{(i)}{=} \E_{(\xi_a,\nu)}\left[ \E\left\{ \left. \exp\left( \lambda R_{a,1} \right) \right| \theta_a \right\}^n \right] 
			\\&\stackrel{(ii)}{=} \E_{(\xi_a,\nu)}\left[ \exp\left\{ A_a\left(\theta_a + \lambda \right) - A_a(\theta_a)  \right\}^n \right] 
			\\&\stackrel{(iii)}{\leq} \E_{(\xi_a,\nu)}\left[ \exp\left\{ \lambda \cdot A_a'(\theta_a) + \frac{L \lambda^2}{2}  \right\}^n \right] 
			\\&\stackrel{(iv)}{=} \E_{(\xi_a,\nu)}\left[ \exp\left\{ n \lambda \cdot \mu_a + \frac{L n \lambda^2}{2}  \right\} \right] 
			\\&= \exp\left( n \lambda \bar{\mu}_a  + \frac{L n \lambda^2}{2} \right) \times \E_{(\xi_a,\nu)}\left[ \exp\left( n \lambda (\mu_a - \bar{\mu}_a ) \right) \right]
			\\&\stackrel{(v)}{\leq} \exp\left( n \lambda \bar{\mu}_a  + \frac{L n \lambda^2}{2} \right) \times \exp\left( \frac{L n^2 \lambda^2}{2 \nu} \right)
			\\&= \exp\left( n \lambda \bar{\mu}_a \right) \times \exp\left( \frac{\lambda^2}{2} \times \frac{ Ln( \nu + n) }{\nu} \right),
	\end{align}
	where we have utilized that (i) $R_{a,i}$'s are conditionally independent given $\theta_a$, (ii) the moment-generating function of $R_{a,1}$ is given by $\E[ \lambda R_a | \theta_a ] = \exp\left( A_a(\theta_a+\lambda) - A_a(\theta_a) \right)$, (iii) $A_a(\cdot)$ is $L$-smooth, (iv) $A_a'(\theta_a) = \mu_a(\theta_a)$, and (v) $\mu_a$ is $\sqrt{L/\nu}$-sub-Gaussian.
	Given that $\E\left[ \sum_{i=1}^n R_{a,i} \right] = n \bar{\mu}_a$, we just have shown that the sum $\sum_{i=1}^n R_{a,i}$ is $\sqrt{\frac{Ln (\nu+n)}{\nu}}$-sub-Gaussian.
	Therefore, its scaled version $\frac{\sum_{i=1}^n R_{a,i}}{\nu+n}$ is $\sqrt{ \frac{Ln}{\nu(\nu+n)} }$-sub-Gaussian, and so is $\hat{\mu}_{a,n}$.
\end{proof}
}

\newedit{
\begin{lemma} \label{lem-mean-reward-subgaussian-maximal}
	Consider the setting of Theorem \ref{thm-suboptimality}.
	With $\sigma_n \defeq \sqrt{ \frac{Ln}{\nu(\nu+n)} }$, the following holds:
	\begin{equation}
		\E\left[ \left( \max_{0 \leq i \leq n} \hat{\mu}_{a,i} - (\bar{\mu}_a + z \sigma_n) \right)^+ \right]
			\leq \frac{\sigma_n}{z} e^{-z^2/2}, \quad \forall z > 0.
	\end{equation}
\end{lemma}
}

\newedit{
\begin{proof}
	Recall that the posterior predictive mean reward process $\{ \hat{\mu}_{a,n} \}_{n \geq 0}$ is the martingale with respect to the filtration generated by reward realizations $R_{a,1}, R_{a,2}, \ldots$ and whose mean is $\bar{\mu}_a$.
	Therefore, $\{ \exp( \lambda \hat{\mu}_{a,n} ) \}_{n \geq 0}$ is a positive submartingale for any given $\lambda \geq 0$.
	By Doob's maximal inequality, we deduce that
	\begin{align}
		\PR\left[ \max_{0 \leq i \leq n} \hat{\mu}_{a,i} \geq \bar{\mu}_a + z \sigma_n \right]
			= \PR\left[ \max_{0 \leq i \leq n} \exp\left( \lambda (\hat{\mu}_{a,i} - \bar{\mu}_a) \right) \geq \exp\left( \lambda z \sigma_n \right) \right]
			\leq \frac{ \E\left[ \exp\left( \lambda (\hat{\mu}_{a,n} - \bar{\mu}_a) \right) \right]  }{ \exp\left( \lambda z \sigma_n \right)  }.
	\end{align}
	By Lemma \ref{lem-mean-reward-subgaussian}, since $\hat{\mu}_{a,n}$ is $\sigma_n$-sub-Gaussian, we further have
	\begin{equation}
		\frac{ \E\left[ \exp\left( \lambda (\hat{\mu}_{a,n} - \bar{\mu}_a) \right) \right]  }{ \exp\left( \lambda z \sigma_n \right)  }
			\leq \frac{ \exp\left( \frac{\lambda^2 \sigma_n^2}{2} \right) }{ \exp\left( \lambda z \sigma_n \right)  }
			= \exp\left( \frac{\lambda^2 \sigma_n^2}{2} - \lambda z \sigma_n \right).
	\end{equation}
	Therefore, by taking $\lambda \defeq \frac{z}{\sigma_n}$, we have $\PR\left[ \max_{0 \leq i \leq n} \hat{\mu}_{a,i} \geq \bar{\mu}_a + z \sigma_n \right] \leq e^{-z^2/2}$, and by invoking Lemma \ref{lem-subgaussian-tail-bound}, we obtain the claim.
\end{proof}
}

\newedit{\subsubsection{Proof of Theorem \ref{thm-suboptimality}}}

\newedit{
\begin{lemma} \label{lem-suboptimality-proof-sketch}
	Consider one of the IRS penalty functions $z^\textsc{TS}$, $z^\textsc{Irs.FH}$, and $z^\textsc{Irs.V-Zero}$.
	As discussed in Remark \ref{rem-suboptimality-rephrase}, we have
	\begin{equation}
		Q_t^{z,\textup{in}}(\mathbf{a}_{1:t-1}, a_t^{z,*}, \omega) - Q_t^{z,\textup{in}}(\mathbf{a}_{1:t-1}, a, \omega) \leq \mu_t^U( \mathbf{a}_{1:t-1}, a_t^{z,*}, \omega) - \mu_t^L( \mathbf{a}_{1:t-1}, a, \omega),
	\end{equation}
	for some $\mu_t^U( \mathbf{a}_{1:t-1}, a_1^{z,*}, \omega)$ and $\mu_t^L( \mathbf{a}_{1:t-1}, a, \omega)$, where $a_t^{z,*}$ abbreviates $a_t^{z,*}( \mathbf{a}_{1:t-1}, \omega )$.
	Suppose that there exists a sequence of confidence intervals $\left\{ (L_t(a), U_t(a)) \right\}_{a \in \Ascr, t \in \N}$ such that $(L_t(\cdot), U_t(\cdot))$ is $\sigma(H_{t-1})$-measurable, and
	\begin{align}
		\E_\mathbf{y}\left[ \left. \left( \mu_t^U( \mathbf{a}_{1:t-1}, a, \omega) - U_t(a) \right)^+ \right| H_{t-1}(\mathbf{a}_{1:t-1}, \omega) \right] &\leq \frac{C_U}{T},
			\quad \forall a, \forall t
		 \\
		\E_\mathbf{y}\left[ \left. \left( L_t(a) - \mu_t^L( \mathbf{a}_{1:t-1}, a, \omega) \right)^+ \right| H_{t-1}(\mathbf{a}_{1:t-1}, \omega) \right] &\leq \frac{C_L}{T},
			\quad \forall a, \forall t
	\end{align}
	for some constants $C_U > 0$ and $C_L > 0$.
	Then, for IRS policy $\pi$ induced by the chosen penalty function, we have
	\begin{equation}
		W^z( T, \mathbf{y} ) - V( \pi, T, \mathbf{y} ) \leq C_U + C_L + \sum_{t=1}^T \E\left[ U_t(A_t^\pi) - L_t(A_t^\pi) \right].
	\end{equation}
\end{lemma}
}

\newedit{
\begin{proof}
	Let $A_t^* \defeq a_t^{z,*}( \mathbf{A}_{1:t-1}^\pi, \omega )$, and let $\E_t[ \cdot ]$ denote $\E[ \cdot | \Fscr_{t-1} ]$.
	By Proposition \ref{prop-generalized-posterior-sampling} we have
	\begin{equation}
		\E_t[ U_t(A_t^\pi) ] 
			= \sum_{a \in \Ascr} U_t(a) \cdot \PR_t[ A_t^\pi = a ] 
			= \sum_{a \in \Ascr} L_t(a) \cdot \PR_t[ A_t^* = a ] 
			= \E_t[ U_t(A_t^*) ].
	\end{equation}
	Therefore, we have
	\begin{align}
		&\E_t\left[ \mu_t^U( A_t^* ) - \mu_t^L(A_t^\pi)  \right]
			\\&= \E_t\left[ \mu_t^U( A_t^* ) - \mu_t^L(A_t^\pi) \right] + \E_t\left[ U_t(A_t^\pi) - U_t(A_t^*) \right] + \E_t\left[ L_t(A_t^\pi) - L_t( A_t^\pi ) \right]
			\\&= \E_t\left[ \mu_t^U( A_t^* ) - U_t(A_t^*)  \right] + \E_t\left[  L_t(A_t^\pi) - \mu_t^L(A_t^\pi) \right] + \E_t\left[ U_t(A_t^\pi) - L_t( A_t^\pi ) \right] 
			\\&\leq \E_t\left[ \left( \mu_t^U( A_t^* ) - U_t(A_t^*) \right)^+  \right] + \E_t\left[  \left( L_t(A_t^\pi) - \mu_t^L(A_t^\pi) \right)^+ \right] + \E_t\left[ U_t(A_t^\pi) - L_t( A_t^\pi ) \right].
	\end{align}
	We further observe that
	\begin{equation}
		\E_t\left[ \left( \mu_t^U( A_t^* ) - U_t(A_t^*) \right)^+  \right] 
			= \sum_{a \in \Ascr} \E_t\left[ \left( \mu_t^U( a ) - U_t(a) \right)^+ \right] \PR_t[ A_t^* = a ]
			\leq \frac{C_U}{T} \sum_{a \in \Ascr} \PR_t[ A_t^* = a ]
			= \frac{C_U}{T}.
	\end{equation}
	Similarly, we have $\E_t\left[  \left( L_t(A_t^\pi) - \mu_t^L(A_t^\pi) \right)^+ \right] \leq \frac{C_L}{T}$.
	Combining all these results, we have
	\begin{align}
		W(T, \mathbf{y}) - V(\pi, T, \mathbf{y})
			&\stackrel{\text{Prop \ref{prop-suboptimality-decomposition}}}{=} \E\left[ \sum_{t=1}^T Q_t^{z,\textup{in}}( A_t^*) - Q_t^{z,\textup{in}}( A_t^\pi)  \right]
			\\&\leq \E\left[ \sum_{t=1}^T \mu_t^U( A_t^* ) - \mu_t^L(A_t^\pi) \right]
			\\&= \E\left[ \sum_{t=1}^T \E_t\left[ \mu_t^U( A_t^* ) - \mu_t^L(A_t^\pi) \right] \right]
			\\&\leq \E\left[ \sum_{t=1}^T \left( \frac{C_U}{T} + \frac{C_L}{T} + \E_t\left[ U_t(A_t^\pi) - L_t(A_t^\pi) \right] \right) \right]
			\\&\leq C_U + C_L + \sum_{t=1}^T \E\left[ U_t(A_t^\pi) - L_t(A_t^\pi) \right].
	\end{align}
\end{proof}
}

\newedit{
We are now ready to prove Theorem \ref{thm-suboptimality}.
To facilitate simpler notation, we define
\begin{equation}
	N_{t-1}^\pi(a) \defeq n_{t-1}( \mathbf{A}_{1:t-1}^\pi, a )
	, \quad
	\hat{\mu}_t^\pi( a, n ) \defeq \hat{\mu}_{a, N_{t-1}^\pi(a)+n},
\end{equation}
which represent, respectively, the number of pulls on arm $a$ prior to time $t$ under policy $\pi$, and the posterior predictive mean reward process given the past actions $\mathbf{A}_{1:t-1}^\pi$.
Observe that for each $a \in \Ascr$, the process $\{ \hat{\mu}_t^\pi( a , n ) \}_{n \geq 0}$ is a martingale, as discussed Remark \ref{rem:mean-reward-martingale}.

Further define
\begin{equation}
	\Delta_t^\pi(a, n) \defeq \sqrt{ \frac{L}{\nu+N_{t-1}^\pi(a)} \times \frac{n}{ \nu + N_{t-1}^\pi(a)+n} },
\end{equation}
which is measurable with respect to $\Fscr_{t-1}$.
In the context of Theorem \ref{thm-suboptimality}, the prior/posterior of arm $a$ at time $t$ is described by the hyperparameters $\left( \xi_a + \sum_{i=1}^{N_{t-1}^\pi(a)} R_{a,i}, \nu + N_{t-1}^\pi(a) \right)$ that converges to $\mu_a$, and therefore Lemma \ref{lem-mean-reward-subgaussian} implies that $\hat{\mu}_t^\pi(a, n)$ is $\Delta_t(a,n)$-sub-Gaussian \emph{conditioned} on $\Fscr_{t-1}$.
}

\newedit{
\noindent \textbf{(1) Suboptimality analysis for \textnormal{\textsc{TS}} \eqref{e-suboptimality-ts}.} 
As discussed in Remark \ref{rem-suboptimality-rephrase}, for \textsc{TS}, we have
\begin{equation}
	Q_t^{z,\textup{in}}(a_t^{z,*}) - Q_t^{z,\textup{in}}(a) 
		= \mu_{a_t^{z,*}} - \mu_a
		= \hat{\mu}_t^\pi( a_t^{z,*}, \infty ) - \hat{\mu}_t^\pi( a, \infty ).
\end{equation}
We construct the confidence intervals as follows:
\begin{equation}
	U_t(a) \defeq \hat{\mu}_t^\pi(a, 0) + \sqrt{ 2 \log T } \times \Delta_t^\pi(a, \infty)
	, \quad
	L_t(a) \defeq \hat{\mu}_t^\pi(a, 0) + \sqrt{ 2 \log T } \times \Delta_t^\pi(a, \infty),
\end{equation}
where $\Delta_t^\pi(a, \infty) = \lim_{n \rightarrow \infty} \Delta_t^\pi(a, n) = \sqrt{ \frac{L}{\nu + N_{t-1}^\pi(a)} }$ so that $\mu_a$ is $\Delta_t^\pi(a, \infty)$-sub-Gaussian conditioned on $\Fscr_{t-1}$.
By Lemma \ref{lem-subgaussian-tail-bound}, we have
\begin{equation}
	\E\left[ \left. \left( \mu_a - U_t(a) \right)^+ \right| \Fscr_{t-1} \right] 
		\leq \frac{\Delta_t^\pi(a, \infty)}{\sqrt{2 \log T}} e^{-\frac{2 \log T}{2}}
		\leq \frac{ \sqrt{ L/\nu } }{ T },
\end{equation}
where we use the fact that $2 \log T \geq 1$ for any $T \geq 2$.
Symmetrically, we have $\E\left[ \left. \left( L_t(a) - \mu_a \right)^+ \right| \Fscr_{t-1} \right] \leq \frac{ \sqrt{ L/\nu } }{ T }$.
By Lemma \ref{lem-suboptimality-proof-sketch}, we have
	\begin{align}
		W^\textsc{TS}( T, \mathbf{y} ) - V( \pi^\textsc{TS}, T, \mathbf{y} ) 
			&\leq 2 \sqrt{ L/\nu } + \sum_{t=1}^T \E\left[ U_t(A_t^\pi) - L_t(A_t^\pi) \right]
			\\&= 2 \sqrt{ L/\nu } + 2 \sqrt{2 \log T} \sum_{t=1}^T \Delta_t^\pi(A_t^\pi, \infty).
	\end{align}
	
Further observe that
\begin{align}
	&\sum_{t=1}^T \Delta_t^\pi( A_t^\pi, \infty) 
		= \sum_{t=1}^T \sqrt{ \frac{L}{ \nu + N_{t-1}^\pi(A_t^\pi)} }
		= \sum_{a \in \Ascr} \sum_{n=0}^{N_T^\pi(a)-1} \frac{\sqrt{L}}{\sqrt{\nu+n}}
		= \sum_{a \in \Ascr} \left( \frac{\sqrt{L}}{\sqrt{\nu}} + \sum_{n=1}^{N_T^\pi(a)-1} \frac{\sqrt{L}}{\sqrt{\nu+n}} \right)
		\\&\leq \sum_{a \in \Ascr} \left( \frac{\sqrt{L}}{\sqrt{\nu}} + \sum_{n=1}^{N_T^\pi(a)-1} \frac{\sqrt{L}}{\sqrt{n}} \right)
		\leq \sum_{a \in \Ascr} \left( \frac{\sqrt{L}}{\sqrt{\nu}} + \int_{x=0}^{N_T^\pi(a)} \frac{\sqrt{L}}{\sqrt{x}}dx \right)
		= \frac{K \sqrt{L}}{\sqrt{\nu}} + 2 \sqrt{L} \sum_{a \in \Ascr} \sqrt{N_T^\pi(a)}.
\end{align}
By utilizing Cauchy--Schwartz inequality, we deduce that
\begin{equation}
	\sum_{a \in \Ascr} \sqrt{N_T^\pi(a)} \leq \sqrt{ K \sum_{a \in \Ascr} N_T(a) } = \sqrt{KT}.
\end{equation}
Combining all these results, we conclude that
\begin{equation}
	W^\textsc{TS}( T, \mathbf{y} ) - V( \pi^\textsc{TS}, T, \mathbf{y} )
		\leq 2 \sqrt{L} \left[  \frac{1}{ \sqrt{\nu} } + \sqrt{2 \log T} \left( \frac{K}{\sqrt{\nu}} + 2 \sqrt{KT} \right) \right].
\end{equation}
}

\newedit{
\noindent \textbf{(2) Suboptimality analysis for \textnormal{\textsc{Irs.FH}} \eqref{e-suboptimality-irs-fh}.} 
As discussed in Remark \ref{rem-suboptimality-rephrase}, for \textsc{Irs.FH}, we have
\begin{equation}
	Q_t^{z,\textup{in}}(a_t^{z,*}) - Q_t^{z,\textup{in}}(a) 
		= \hat{\mu}_t^\pi( a_t^{z,*}, T-t ) - \hat{\mu}_t^\pi( a, T-t ).
\end{equation}
We construct the confidence intervals as follows:
\begin{equation}
	U_t(a) \defeq \hat{\mu}_t^\pi(a, 0) + \sqrt{ 2 \log T } \times \Delta_t^\pi(a, T-t)
	, \quad
	L_t(a) \defeq \hat{\mu}_t^\pi(a, 0) + \sqrt{ 2 \log T } \times \Delta_t^\pi(a, T-t).
\end{equation}
Given that $\hat{\mu}_t^\pi(a, T-t)$ is $\Delta_t^\pi(a, T-t)$-sub-Gaussian conditioned on $\Fscr_{t-1}$, by Lemma \ref{lem-subgaussian-tail-bound}, we have
\begin{equation}
	\E\left[ \left. \left( \hat{\mu}_t^\pi( a, T-t ) - U_t(a) \right)^+ \right| \Fscr_{t-1} \right] 
		\leq \frac{\Delta_t^\pi(a, T-t)}{\sqrt{2 \log T}} e^{-\frac{2 \log T}{2}}
		\leq \frac{\Delta_t^\pi(a, \infty)}{\sqrt{2 \log T}} e^{-\frac{2 \log T}{2}}
		\leq \frac{ \sqrt{ L/\nu } }{ T }.
\end{equation}
Symmetrically, we have $\E\left[ \left. \left( L_t(a) - \hat{\mu}_t^\pi( a, T-t ) \right)^+ \right| \Fscr_{t-1} \right] \leq \frac{ \sqrt{ L/\nu } }{ T }$.

On the other hand, since $N_{t-1}(a) \leq t$ in any case, we have
\begin{align}
	\frac{1}{\nu + N_{t-1}^\pi(a)} \times \frac{T-t}{\nu + N_{t-1}^\pi(a) + T - t }
		&= \frac{1}{\nu + N_{t-1}^\pi(a)} \times \left( 1 - \frac{ \nu + N_{t-1}^\pi(a) }{\nu + N_{t-1}^\pi(a) + T - t } \right)
		\\&= \frac{1}{\nu + N_{t-1}^\pi(a)} - \frac{ 1 }{\nu + N_{t-1}^\pi(a) + T - t }
		\\&\leq \frac{1}{\nu + N_{t-1}^\pi(a)} - \frac{ 1 }{\nu + T }.
\end{align}
Consequently,
\begin{align}
	\sum_{t=1}^T \sqrt{ \frac{1}{\nu + N_{t-1}^\pi(a)} - \frac{ 1 }{\nu + T } }
		&= \sum_{a \in \Ascr} \sum_{n=0}^{N_T^\pi(a)-1} \sqrt{ \frac{1}{\nu + n} - \frac{ 1 }{\nu + T } }
		\\&= \sum_{a \in \Ascr} \left( \sqrt{ \frac{1}{\nu} - \frac{1}{\nu+T} } + \sum_{n=1}^{N_T^\pi(a)-1} \sqrt{ \frac{1}{\nu + n} - \frac{ 1 }{\nu + T } } \right)
		\\&\leq \frac{K}{\sqrt{\nu}} + \sum_{a \in \Ascr} \sum_{n=1}^{N_T^\pi(a)-1} \sqrt{ \frac{1}{n} - \frac{ 1 }{T } }
		\\&\stackrel{(i)}{\leq} \frac{K}{\sqrt{\nu}} + \sum_{a \in \Ascr} \sum_{n=1}^{N_T^\pi(a)-1} \left(  \frac{1}{\sqrt{n}} - \frac{\sqrt{n}}{2T} \right)
		\\&\leq \frac{K}{\sqrt{\nu}} + \sum_{a \in \Ascr} \int_0^{N_T^\pi(a)}\left(  \frac{1}{\sqrt{x}} - \frac{\sqrt{x}}{2T} \right)dx
		\\&= \frac{K}{\sqrt{\nu}} + \sum_{a \in \Ascr} \left(  2 \sqrt{N_T^\pi(a)} - \frac{ \left( N_T^\pi(a) \right)^{3/2} }{2T} \right)
		\\&\stackrel{(ii)}{\leq} \frac{K}{\sqrt{\nu}} + 2\sqrt{KT} - \frac{1}{3} \sqrt{T/K},
\end{align}
where we have utilized that (i) the concavity of $\sqrt{\cdot}$, and (ii) $\min\{ \sum_{a=1}^K n_a^{3/2}; \sum_{a=1}^K n_a = T \} = \sum_{a=1}^K ( T / K)^{3/2} = \sqrt{T^3/K}$.

Combining all these results, we conclude that
\begin{align}
	W^\textsc{Irs.FH}( T, \mathbf{y} ) - V( \pi^\textsc{Irs.FH}, T, \mathbf{y} )
		&\leq 2 \sqrt{ \frac{L}{\nu} } + 2 \sqrt{2 \log T} \sum_{t=1}^T \Delta_t^\pi(A_t^\pi, T-t) 
		\\&\leq 2 \sqrt{L} \left[  \frac{1}{ \sqrt{\nu} } + \sqrt{2 \log T} \left( \frac{K}{\sqrt{\nu}} + 2 \sqrt{KT} - \frac{1}{3} \sqrt{T/K} \right) \right].
\end{align}
}

\newedit{
\noindent \textbf{(3) Suboptimality analysis for \textnormal{\textsc{Irs.V-Zero}} \eqref{e-suboptimality-irs-vzero}.} 
As discussed in Remark \ref{rem-suboptimality-rephrase}, for \textsc{Irs.FH}, we have
\begin{equation}
	Q_t^{z,\textup{in}}(a_t^{z,*}) - Q_t^{z,\textup{in}}(a) 
		= \max_{0 \leq n \leq T-t} \left\{ \hat{\mu}_t^\pi( a_t^{z,*}, n ) \right\} - \hat{\mu}_t^\pi( a, 0 ).
\end{equation}
We construct the confidence intervals as follows:
\begin{equation}
	U_t(a) \defeq \hat{\mu}_t^\pi(a, 0) + \sqrt{ 2 \log T } \times \Delta_t^\pi(a, T-t)
	, \quad
	L_t(a) \defeq \hat{\mu}_t^\pi(a, 0).
\end{equation}
By Lemma \ref{lem-mean-reward-subgaussian-maximal}, we have
\begin{equation}
	\E\left[ \left. \left( \max_{0 \leq n \leq T-t} \hat{\mu}_t^\pi( a, n ) - U_t(a) \right)^+ \right| \Fscr_{t-1} \right]
		\leq \frac{\Delta_t^\pi(a, T-t)}{\sqrt{2 \log T}} e^{-\frac{2 \log T}{2}}
		\leq \frac{ \sqrt{ L/\nu } }{ T },
\end{equation}
where
\begin{equation}
	\E\left[ \left. \hat{\mu}_t^\pi( a, 0 ) - L_t(a) \right| \Fscr_{t-1} \right] = 0.
\end{equation}
The rest of the proof is almost identical to the case of \textsc{Irs.FH}:
\begin{align}
	W^\textsc{Irs.V-Zero}( T, \mathbf{y} ) - V( \pi^\textsc{Irs.V-Zero}, T, \mathbf{y} )
		&\leq \sqrt{ \frac{L}{ \nu } } + \sum_{t=1}^T \E\left[ U_t(A_t^\pi) - L_t(A_t^\pi) \right].
		\\&= \sqrt{ \frac{L}{ \nu } } + \sqrt{2 \log T} \sum_{t=1}^T \Delta_t^\pi(A_t^\pi, T-t) 
		\\&\leq \sqrt{L} \left[  \frac{1}{ \sqrt{ \nu } } + \sqrt{2 \log T} \left( \frac{K}{\sqrt{\nu}} + 2 \sqrt{KT} - \frac{1}{3} \sqrt{T/K} \right) \right].
\end{align}
}